\pdfoutput=1
\documentclass{article}

\usepackage{microtype}
\usepackage{graphicx}
\usepackage{subcaption}
\usepackage{booktabs}
\usepackage{dsfont}

\usepackage[colorlinks,linkcolor={red!80!black},citecolor={blue},allbordercolors={1 1 1},urlcolor={blue!80!black},hypertexnames=false]{hyperref}

\usepackage[main,final]{neurips_2026}

\usepackage{algorithm}
\usepackage{algpseudocode}
\usepackage{caption}
\usepackage{amsmath}
\usepackage{mathtools}

\usepackage{amssymb}
\usepackage{amsthm}
\usepackage{enumitem}

\usepackage[capitalize,noabbrev]{cleveref}
\usepackage[dvipsnames]{xcolor}
\usepackage{tikz}
\usetikzlibrary{positioning}
\usepackage{natbib}
\newcommand{\coderepo}{\url{https://github.com/Basile-Terv/PEIRA}}
\newcommand{\lossrepo}{\url{https://github.com/MichaelArbel/peira-loss}}
\newcommand{\projectpage}{\url{https://peira-ssl.github.io}}

\newlist{assumplist}{enumerate}{1}
\setlist[assumplist]{label=(\textbf{\Alph*})}
\Crefname{assumplisti}{Assumption}{Assumptions}

\newlist{assumplist2}{enumerate}{1}
\setlist[assumplist2]{label=(\textbf{\Roman*})}
\Crefname{assumplist2i}{Assumption}{Assumptions}
\crefname{appendixsection}{Appendix}{Appendices}
\Crefname{appendixsection}{Appendix}{Appendices}

\theoremstyle{plain}

\newtheorem{lemma}{Lemma}[section]

\newtheorem{definition}{Definition}[section]

\newtheorem{proposition}{Proposition}[section]

\newtheorem{remark}{Remark}[section]

\makeatother

\usepackage[english]{babel}

\crefname{lemma}{Lemma}{Lemmas}
\Crefname{lemma}{Lemma}{Lemmas}
\crefname{theorem}{Theorem}{Theorems}
\Crefname{theorem}{Theorem}{Theorems}
\crefname{proposition}{Proposition}{Propositions}
\Crefname{proposition}{Proposition}{Propositions}
\crefname{corollary}{Corollary}{Corollaries}
\Crefname{corollary}{Corollary}{Corollaries}
\crefname{definition}{Definition}{Definitions}
\Crefname{definition}{Definition}{Definitions}
\Crefname{assumption}{Assumption}{Assumptions}

\newcommand*\diff{\mathop{}\!\mathrm{d}}

\newcommand{\Tr}{Tr}

\newcommand{\numcca}{{\textcolor{black}{R}}}
\newcommand{\rank}{ {\textcolor{black}{r}}}
\DeclareMathOperator{\erank}{erank}

\newcommand{\ucoor}{\mathbf{w}}

\newcommand{\parens}[1]{\left( #1 \right)}
\newcommand{\brackets}[1]{\left[ #1 \right]}
\newcommand{\verts}[1]{\left\lvert #1 \right\rvert}
\newcommand{\Verts}[1]{\left\lVert #1 \right\rVert}

\usepackage{xspace}

\newcommand{\method}{\text{PEIRA}\xspace}
\newcommand{\jepa}{\text{res}\xspace}
\newcommand{\algocomp}{SC-\method}

\title{\method: Learning Predictive Encoders \\ through Inter-View Regressor Alignment
}

\makeatletter
\newcommand{\blfootnote}[1]{%
  \begingroup
  \renewcommand\thefootnote{}\footnote{#1}%
  \addtocounter{footnote}{-1}%
  \endgroup
}
\makeatother

\author{%
  Michael Arbel$^{\ast}$ \\
  Univ. Grenoble Alpes, Inria \\
  CNRS, Grenoble INP, LJK \\
  \texttt{michael.arbel@inria.fr} \\
  \And
  Basile Terver$^{\ast}$ \\
  Ecole Normale Supérieure / PSL \\
  Inria Paris \\
  \texttt{basile.terver@ens.psl.edu} \\
  \AND
  Jean Ponce \\
  Ecole Normale Supérieure / PSL \\
  New York University \\
  \texttt{jean.ponce@ens.fr} \\
}

\begin{document}

\maketitle
\blfootnote{$^{\ast}$Equal contribution.}

\begin{abstract}
Non-contrastive self-supervised learning (SSL) is an effective framework for predictive representation learning, but popular (and in practice effective) methods such as SimSiam, BYOL, I-JEPA or DINO, which rely on a form of self-distillation to train a teacher-student network, remain poorly understood as they typically do not minimize a well-defined objective. We analyze the dynamics of a variant of the Joint Embedding Predictive Architecture (JEPA) using a regularized linear regressor to predict the learned representations of two views of the data from one another, and fully characterize its stability: non-collapsed stable equilibria align with leading nonlinear canonical correlation subspaces, while collapsed equilibria may also be stable attractors. Motivated by this result, we introduce PEIRA, a non-contrastive SSL method with an explicit objective defined through the trace of the optimal linear regressor. We show that its only stable equilibria are nontrivial global minimizers and recover the same canonical correlation subspaces, with regularization selecting the effective dimension. Experiments on ImageNet-1K and CIFAR-10 show PEIRA is competitive with VICReg and LeJEPA baselines, and qualitative empirical results support the theory.
Project page: \projectpage{}. Code: \lossrepo{} (loss implementation), \coderepo{} (experiment reproduction).
\end{abstract}

\section{Introduction}
\emph{Non-contrastive} self-supervised learning has emerged as a successful approach to representation learning, achieving strong empirical performance across vision and video understanding~\citep{Grill:2020a,Caron:2021,Chen:2021b,Bardes:2022,Bardes:2024,assran2025v}. 
In this setting, one learns encoders of different views of the same data so the corresponding features predict each other using a learned inter-view regressor, thereby avoiding the need to mine negative examples as in \emph{contrastive} methods \citep{Chen:2020c,He:2020}. These approaches are, however, susceptible to \emph{representational collapse}, in which the learned encoders and predictor respectively degenerate to constants and the identity~\citep{LeCun:2022a}. A large body of work has studied and sought to prevent collapse. 
Among non-contrastive methods, one can distinguish \emph{regularization-based} approaches, such as Barlow Twins \citep{Zbontar:2021}, VICReg \citep{Bardes:2022}, and LeJEPA \citep{Balestriero:2025}, from \emph{self-distillation} methods, such as SimSiam \citep{Chen:2021b}, BYOL \citep{Grill:2020a}, DINO \citep{Caron:2021,Oquab:2023,Simeoni:2025}, I-JEPA \citep{Assran:2023a} and V-JEPA \citep{Bardes:2024}. 
The latter use mechanisms such as stop-gradient and exponential moving averages, leading to dynamics that differ from standard gradient descent and may avoid collapse \citep{Tian:2021a}. 
Regularization-based methods, like VICReg or Barlow Twins, have been connected to the recovery of canonical correlation subspaces \citep{Balestriero:2022,Chapman:2024}, providing a useful interpretation of the representations they learn. 
By contrast, \emph{self-distillation} methods remain poorly understood despite their empirical success: beyond restricted analyses \citep{Tian:2021a,Tang:2023,Littwin:2024,Ponce:2025}, the learned representations and their collapse behavior are still not fully characterized.

In this work, we close this gap for a self-distillation variant of JEPA dynamics that learns nonlinear encoders coupled with an optimal regularized linear predictor. We establish, to our knowledge, the first rigorous connection between these dynamics and nonlinear canonical correlation analysis (CCA) \citep{Michaeli:2015}. The dynamics converge to representations spanning a nonlinear CCA subspace, with effective dimension selected implicitly by regularization. We characterize stable equilibria, identifying those that encode predictive structure and those that collapse. This extends prior analyses of non-contrastive self-distillation, which are restricted to linear encoders and assumptions such as isotropic covariance or orthogonal initialization \citep{Tian:2021a,Tang:2023,Littwin:2024}. 
The characterization motivates a new  unconstrained objective, \emph{Predictive Encoders through Inter-View Regressor Alignment} (\textbf{\method}), for JEPA-style learning, with the same statistical target that rules out collapse as a stable solution, both theoretically and empirically. \method aligns encoders by maximizing the trace of the best regularized linear predictor between views while controlling feature scale. 
We analyze its dynamics and equilibria, showing that it recovers the same nonlinear CCA subspaces as the considered self-distillation dynamics, again with regularization-induced spectral selection, while excluding collapse through the objective rather than training heuristics. Finally, we give a scalable stochastic implementation for parametric encoders inspired by stochastic compositional optimization \citep{wang2017stochastic}, which avoids differentiating through the optimal predictor and achieves competitive performance on SSL image representation learning. Specifically, the main contributions of this paper are threefold:

1. We present a full theoretical characterization of the dynamics of
JEPA-style self-distillation learning in terms of canonical correlation
analysis (CCA) when the predictor is linear and the encoders are
arbitrary square-integrable functions. This includes identifying all
equilibria of the corresponding dynamics {\bf (\cref{prop:convergence_ssl})} and
characterizing the stable ones, which include the collapsed zero
solution {\bf (\cref{prop:stability_ssl})}.

2. We introduce PEIRA, a novel approach to self-supervised learning with a well-defined unconstrained objective. We characterize the equilibria of its dynamics and show that PEIRA recovers the same CCA subspaces as self-distillation {\bf (\cref{prop:stability-peira})}. Unlike self-distillation, however, all stable equilibria of PEIRA are nontrivial global minimizers of the objective {\bf (\cref{prop:peira-thresholded-subspace})}.

3. We propose a scalable stochastic training procedure for \method{} 
{\bf (\cref{alg:peira})}, for parametric encoders (in practice neural networks), and demonstrate in \cref{sec:experiments} that it delivers competitive performance with respect to baselines like VICReg and LeJEPA on ImageNet-1K and CIFAR-10 benchmarks. 
Experiments on these benchmarks also confirm the theoretical claims about spectral alignment, show strong correlation between the \method{} objective and downstream performance, and robustness of the method to variations in hyperparameter values.

\section{Related work}

\textbf{Canonical correlation analysis and self-supervised learning.}
CCA is the classical framework for extracting maximally correlated components across paired views under a whitening constraint \citep{Hardoon:2004,Andrew:2013a,Fukumizu:2007a}. Recently, several SSL methods have incorporated CCA-inspired structure to prevent collapse. Whitening-based losses \citep{Ermolov:2021} and their analysis \citep{Huang:2024b} do so through explicit decorrelation, while penalized relaxations \citep{Zhang:2021c,Chang:2018} improve practicality but lose exact CCA recovery guarantees. \citet{Chapman:2024} proposed an unconstrained objective for CCA with provable recovery in the linear setting, and proposed to  use it for SSL. However, its quartic (4th order polynomial) dependence on the representation can make optimization prone to instabilities \citep{Shi:2023,De-Sa:2015}.
A few works also established connections between some existing SSL methods and CCA. \citet{Balestriero:2022} show that VICReg and Barlow Twins  recover spectral embedding methods, while \citet{Chapman:2024} further rigorously establish that VICReg recovers CCA subspaces, in the linear setting, and show that it can still collapse. Similarly, \citet{Lyu:2021a} show that maximizing correlation under decorrelation constraints recovers shared components, and discuss similarities with methods such as BYOL and Barlow Twins. However, these works stop short of providing a formal CCA recovery theory for self-distillation methods.

\textbf{Self-distillation methods for non-contrastive SSL.} 
Theory for these methods remains confined to highly simplified regimes, typically linear models with isotropic data or orthogonal initialization \citep{Tian:2021a,Tang:2023,Littwin:2024}. \citet{Tang:2023} have further shown that, provided the training procedure converges, the corresponding equilibrium spans a linear eigenspace of the cross-covariance matrix. Related analyses with shallow neural networks still rely on simplified data models \citep{Wen:2022a}. Extending these results is difficult: even in the linear case, relaxing isotropy leads to hard-to-characterize equilibria and unclear convergence, since the dynamics are not the gradient flow of any global objective \citep{Ponce:2025}.

\section{
From Self-Distillation to \method}
\label{sec:method}
Following the JEPA approach to self-supervised representation learning, 
we are interested in learning encoders of different views of the data such that the encoded views predict each other well. 
In general, this problem can be formulated as follows: 
Given a distribution $\mathbb{P}$ of data pairs $(x,y)$ (the views) in $\mathbb{R}^{d_1}\times \mathbb{R}^{d_2}$, the goal is to learn square-integrable and in general \emph{nonlinear} encoders $U:\mathbb{R}^{d_1}\rightarrow \mathbb{R}^k$ and $V:\mathbb{R}^{d_2}\rightarrow \mathbb{R}^k$ for views $x$ and $y$ of the data using a predictive objective. We consider a  regularized
\emph{symmetric}\footnote{
Symmetrized objectives were also considered in \citet{Tang:2023,Grill:2020a}, and naturally accommodate views in different spaces. There is no loss of generality  due to symmetrization and all our results can be adapted to asymmetric objectives, provided the views are interchangeable.} objective which involves a \emph{linear} prediction operator $P:\mathbb{R}^k\rightarrow\mathbb{R}^k$ represented as a $k\times
k$ matrix
\begin{equation}
\begin{gathered}
\mathcal{P}(P,U,V)
=
\mathcal{P}_{U\triangleright V}(P,U,V)
+
\mathcal{P}_{V\triangleright U}(P,V,U)
+
\frac{\lambda}{2}
\Verts{P}_F^2,
\end{gathered}
\end{equation}
where $\Verts{P}_{F}$ is the Frobenius norm of $P$, whereas $\mathcal{P}_{U\triangleright V}$ and $\mathcal{P}_{V\triangleright U}$ are directed objectives encouraging $P$ to linearly predict one encoded view from the other.
Specifically, defining $\Verts{U}_2^2=\mathbb{E}[\Verts{U(x)}_2^2]$ and
$\Verts{V}_2^2=\mathbb{E}[\Verts{V(y)}_2^2]$, we define these objectives as follows
\begin{equation}
\begin{aligned}
	\mathcal{P}_{U\triangleright V}(P,U,V)
&=
\frac{1}{2}\mathbb{E}\brackets{
  \Verts{PU(x)-V(y)}^2
} + \frac{\lambda}{2} \|U\|_2^2, \\
\mathcal{P}_{V\triangleright U}(P,V,U)
&=
\frac{1}{2}\mathbb{E}\brackets{
  \Verts{PV(y)-U(x)}^2
} + \frac{\lambda}{2} \|V\|_2^2,
\end{aligned}
\end{equation}
where expectations are taken w.r.t. the data distribution $\mathbb{P}$. Here, the symbol $U\triangleright V$ means that $V$ is treated as a \emph{target/teacher} view for prediction, whereas $U$ is a \emph{source/student} view. Throughout, \emph{self-distillation} refers to updates in which the teacher view is detached by a stop-gradient or tracked by an EMA; \emph{collapse} means that the encoded features span a subspace of lower dimension than achievable, either totally ($U=V=0$) or partially; and the \emph{effective rank} of a feature covariance is the exponential of the entropy of its normalized spectrum~\citep{RankMe}. 
Ideally, we would like to learn $P$, $U$ and $V$ by minimizing $\mathcal{P}$ w.r.t. these variables. 
Following \citet{Tang:2023}, we consider the fast-predictor regime where the regressor is optimized much faster than the encoders. 
This regime will allow us to provide a precise characterization of equilibria of the dynamics under consideration, by exploiting a closed form expression of the regressor.
Namely, for each pair of encoders $(U,V)$, the regressor is set to its optimal value,
\begin{equation}
\label{eq:csr-optimal-predictor}
P_{U,V}
=
\arg\min_{P\in\mathbb R^{k\times k}}
\mathcal P(P,U,V),
\end{equation}
which can be computed in closed form as the solution of a ridge regression problem. 
Thus, minimizing $\mathcal{P}$ w.r.t. $P$, $U$ and $V$ is equivalent
to minimizing w.r.t. $U$ and $V$ the \emph{residual} objective function
\begin{equation}
\label{eq:ssl_objective}
  \mathcal{P}_{\jepa}(U,V) =
  \mathcal{P}_{U\triangleright V}(P_{U,V},U,V)
+
\mathcal{P}_{V\triangleright U}(P_{U,V},V,U)
+
\frac{\lambda}{2}
\Verts{P_{U,V}}_F^2. 
\end{equation}
Optimizing either $\mathcal{P}$ or $\mathcal{P}_{\jepa}$ is known to be prone to trivial solutions, e.g., $U$ and $V$ and $P$ are all equal to $0$, leading to representation
collapse. 
To address this, self-distillation methods such as SimSiam, BYOL, V-JEPA or DINO update the encoders using a so-called \emph{semi-gradient} \citep[Chapter 9.3]{sutton2018reinforcement}, that is, a proxy for the actual gradient obtained by applying either a stop-gradient operation or exponential moving averages to the \emph{teacher} view. In this way, the \emph{student} view is trained to predict the \emph{teacher} view, rather than both views being updated through the full gradient of the objective. 
Although this heuristic works well empirically, the learned representations remain poorly understood. In \cref{subsec:ssl-cca}, we show that the self-distillation dynamics derived from the objective $\mathcal{P}_{\jepa}$ recover nonlinear CCA subspaces, but are still prone to (partial) collapse.

We propose instead to minimize with respect to $U$ and $V$ a slightly different objective which also depends on the optimal regressor $P_{U,V}$ but which \emph{is guaranteed to} prevent collapse. Let us first note that $P_{U,V}$ can be written in closed form as a matrix-valued \emph{signal-to-noise ratio} (SNR): 
\begin{equation}
\label{eq:closed-form_expression_P}
  P_{U,V}=\Sigma_{U,V}(N_{U,V}+
  \lambda I)^{-1},\,\,\text{where}\,\,
  \left\{\begin{array}{l}
           \Sigma_{U,V}=\mathbb{E}[V(y)U(x)^{\top}+U(x)V(y)^{\top}],\\
           N_{U,V}=\mathbb{E}[U(x)U(x)^{\top}+V(y)V(y)^{\top}],\\
  \end{array}\right.
\end{equation}
where the matrix $\Sigma_{U,V}$ measures shared cross-view signal, while $N_{U,V}$ measures within-view variability. 
The trace of the predictor appears as a natural scalar summary of the SNR as it sums the contributions of different eigen-modes of $P_{U,V}$ and maximizing it should favor representations that share most information between views. Thus, instead of $\mathcal{P}_{\jepa}$, we propose to minimize the \textbf{\method (for Predictive Encoders through Inter-View Regressor Alignment)} objective:
\begin{align}
\label{eq:csr-objective}
\mathcal{E}_{\method}(U,V)
=
-\frac{1}{2}\Tr\!\bigl(P_{U,V}\bigr)
+
\frac{\lambda}{2}\parens{\Verts{U}_2^2+ \Verts{V}_2^2}.
\end{align}
When both encoders are linear, minimizing the negative trace of $P_{U,V}$ is closely related to ratio-trace optimization for dimensionality reduction \citep{Wang:2007,Jia:2009} discussed in \cref{sec:background}. 
The quadratic regularization on $U$ and $V$ controls feature scale, preventing the negative trace term from being decreased merely by enlarging their norms.  
 The trace of the regressor can be written as the Frobenius scalar product between the symmetric matrices $\Sigma_{U,V}$ and $(N_{U,V}+\lambda I)^{-1}$. Therefore, minimizing $\mathcal{E}_{\method}$ should also encourage an \emph{alignment} between the eigenbases of the signal and noise matrices $\Sigma_{U,V}$ and $N_{U,V}$. Such an alignment has been shown to arise implicitly for some self-distillation methods with linear encoders \citep{Tian:2021a}.

\section{Canonical structure and stability of the learning dynamics}
\label{sec:dynamics}

We characterize the equilibria of the
\emph{self-distillation dynamics} arising from the objective $\mathcal{P}_{\jepa}$ in \cref{eq:ssl_objective} and the critical points of \method introduced in \cref{eq:csr-objective} in terms of nonlinear CCA.  

Nonlinear CCA identifies pairs of functions of the two views that are maximally correlated after normalization by a
whitening constraint (see \cref{sec:background} for a formal definition). 
These directions capture the shared signal between the views, whereas directions with small canonical correlation correspond to signal components that are only weakly correlated across views. 
To make the characterizations in this section independent of any particular model architecture, we do not impose parametric restrictions on the encoders. 
Instead, we work directly in the natural infinite dimensional function spaces of square-integrable encoders, $\mathcal U=L^2(\mathbb{P}_X;\mathbb{R}^k)$, $\mathcal V=L^2(\mathbb{P}_Y;\mathbb{R}^k),$ where $\mathbb P_X$ and $\mathbb P_Y$ denote the marginals of the joint distribution $\mathbb P$.
Parametric models, such as deep networks, are introduced only later as function approximators in these spaces (see the \emph{capacity hypothesis} in \citet{Huh:2024}). 
We use the following spectral form of nonlinear CCA throughout the analysis. 
\begin{definition}[Nonlinear Canonical Correlations]\label{def:cca}
The distribution $\mathbb{P}$ is said to admit a nonlinear canonical correlation decomposition, in the sense of \citet{Michaeli:2015}, if there exist orthonormal families $(\psi_i)_{1\leq i\leq \numcca }$  and $(\xi_i)_{1\leq i\leq \numcca}$ of $ L^2(\mathbb P_X)$  and $L^2(\mathbb{P}_Y)$, and a non-increasing sequence $(c_i)_{1\leq i\leq \numcca} \subset (0,1]$, where $\numcca$ may be infinite, such that $
\mathbb E[\psi_i(x)\xi_i(y)] = c_i$ and 
$\mathbb E[\psi_i(x)\xi_j(y)] = 0 \quad \text{for } i\neq j$. 
Moreover, the families exhaust all cross-correlated directions: any square-integrable function of $x$ orthogonal to all the $\psi_i$'s is uncorrelated with every square-integrable function of $y$, and conversely for functions of $y$ orthogonal to all the $\xi_i$'s.
The functions $\psi_i$ and $\xi_i$ are called the nonlinear canonical directions, and the scalars $c_i$ are the canonical correlations, which quantify the strength of the shared signal along the $i$-th direction.
\end{definition}

\subsection{
Characterizing the self-distillation equilibria and their stability}
\label{subsec:ssl-cca}
We study the continuous-time self-distillation dynamics arising from the residual objective $\mathcal{P}_{\jepa}$ in \cref{eq:ssl_objective}, as it allows us to use tools from dynamical systems theory. 
Specifically, we consider the learning dynamics arising from applying a stop-gradient\footnote{While we focus on stop-gradient rather than exponential-moving-average (EMA) dynamics, our results also characterize the equilibria of the corresponding EMA dynamics as recovering CCA directions, since the two dynamics admit the same equilibria \citep{Ponce:2025}. We leave a convergence and stability analysis of EMA dynamics to future work, and conjecture that the proof techniques developed here may extend to that setting.} 
operation to the \emph{teacher} views in $\mathcal{P}_{\jepa}$ before differentiating it, as done in SimSiam \citep{Chen:2021b}.  
Such a procedure is mathematically equivalent to the following  continuous-time dynamics defined in the non-parametric space $\mathcal{U}\times\mathcal{V}$:  
\begin{gather}
\label{eq:ode_ssl_appendix}
\dot{U_t}=-\partial_{U}\mathcal{P}_{U\triangleright V}(P_{U_t,V_t},U_t,V_t),
\qquad
\dot{V_t}=-\partial_{V}\mathcal{P}_{V\triangleright U}(P_{U_t,V_t},V_t,U_t),
\end{gather}
where $\partial_{U}\mathcal{P}_{U\triangleright V}$ (resp. $\partial_{V}\mathcal{P}_{V\triangleright U}$) denotes the partial derivative of $\mathcal{P}_{U\triangleright V}$ (resp. $\mathcal{P}_{V\triangleright U}$) w.r.t. $U$ (resp. $V$). 
The resulting flow extends the one studied by \citet{Tang:2023} for linear encoders to nonparametric ones, better capturing the rich representation classes used in practice.

The vector field defined in \cref{eq:ode_ssl_appendix} is generally not the gradient of any objective \citep{Ponce:2025}, making convergence harder to characterize without additional structure. 
Our first result establishes the existence of a Lyapunov function for the flow, i.e., a non-increasing function along such a flow, that confers a gradient-like structure to the system \citep{Chill:2009}. 
This, in turn, allows us to leverage the abstract Lojasiewicz–Simon framework for establishing convergence to equilibria \citep{chill2003lojasiewicz,Bolte:2014}. We identify such a function as a polynomial of the encoders expressed in terms of the signal and noise matrices $\Sigma_{U,V}$ and $N_{U,V}$ defined in \cref{eq:closed-form_expression_P}:
\begin{equation}\label{eq:lyapunov_ssl}
\begin{aligned}
\mathcal{L}(U,V)=&\frac{1}{3}\Tr\!\left((N_{U,V}+\lambda I)^3\right)-\frac{1}{2}\Tr\!\left(\Sigma_{U,V}^2\right).
\end{aligned}
\end{equation}
We further show that these equilibria admit an explicit characterization in terms of CCA directions. The following proposition formalizes this result under a standard  regularity assumption on the data for nonlinear CCA \citep{Michaeli:2015}, stated in \cref{appendix:cca} and ensuring the data distribution admits a nonlinear canonical correlation decomposition.
\begin{proposition}[Identification of the equilibria]\label{prop:convergence_ssl}
Under \cref{assump:density_ratio} of \cref{appendix:cca} and assuming that  $c_1,\dots,c_{\numcca}$ are all distinct, for any initial condition $(U_0, V_0) \in \mathcal{U}\times \mathcal{V}$, the dynamics in \cref{eq:ode_ssl_appendix} admit the  function $\mathcal{L}$ in \cref{eq:lyapunov_ssl} as a Lyapunov function, and necessarily converge to an equilibrium point of the form: 
\begin{align}\label{eq:equilibria_ssl}
	U^{\star}(x) = \frac{1}{\sqrt{2}}
	\sum_{\sigma=\pm}\sum_{i\in D^{\sigma}} f_{\epsilon_i^{\sigma}}(c_i,\lambda) \psi_i(x){\bf{q_i^{\sigma}}},\quad V^{\star}(y) = \frac{1}{\sqrt{2}}\sum_{\sigma=\pm}\sum_{i\in D^{\sigma}} \rho_{\sigma} f_{\epsilon_i^{\sigma}}(c_i,\lambda) \xi_i(y){\bf{q_i^{\sigma}}}
\end{align}
where $D^\pm$ are any subsets of $\{1,\dots,\numcca\}$ with $\verts{D^+} + \verts{D^-}\le k$, 
$(\mathbf q_i^\sigma)_{\sigma\in\{+,-\},\,i\in D^\sigma}$ are any orthonormal vectors in $\mathbb R^k$, 
$\epsilon_i^\sigma$ take any values in $\{-1,1\}$, 
$\rho_\pm=\pm1$, and $f_\epsilon$ is the spectral filter:
\begin{align}
	f_{\epsilon}(c,\lambda)=\frac{1}{2}\parens{|c|+\epsilon\sqrt{c^2-4\lambda}}\mathds{1}_{|c|^2-4\lambda\ge 0}.
\end{align}
\end{proposition}
The filter $f_\epsilon$ thresholds canonical directions according to the regularization level $\lambda$: only directions with sufficiently large canonical correlation can appear at an equilibrium. 
Thus, the equilibria span arbitrary subspaces generated by at most $k$ canonical directions. 
They may include positively or negatively correlated components, corresponding to the two sums in \cref{eq:equilibria_ssl}; these signs affect the eigenvalues of the 
signal matrix $\Sigma_{U^{\star},V^{\star}}$, but not the underlying canonical subspace. 
Since the index sets $D^+$ and $D^-$ are arbitrary, \cref{prop:convergence_ssl} alone does not imply recovery of the leading canonical directions. 
To determine which subspaces the dynamics select, the next result characterizes their stable equilibria, where an equilibrium $w^{\star}$ of a dynamical system $\dot{w}_t=-F(w_t)$ is \emph{stable} if the Jacobian $\nabla F(w^{\star})$ has no eigenvalue with strictly negative real part, and \emph{unstable} otherwise.
\begin{proposition}[Stability]\label{prop:stability_ssl}
	Under the assumptions of \cref{prop:convergence_ssl}, the stable equilibria are:
\begin{align}\label{eq:stable_equilibria_ssl}
	U^{\star}(x) = \frac{1}{\sqrt{2}}\sum_{\sigma=\pm}\sum_{i=1}^{\rank^{\sigma}} f_{1}(c_i,\lambda) \psi_i(x){\bf{q_i^{\sigma}}},\quad V^{\star}(y) = \frac{1}{\sqrt{2}}\sum_{\sigma=\pm}\sum_{i=1}^{\rank^{\sigma}} \rho_{\sigma} f_{1}(c_i,\lambda) \xi_i(y){\bf{q_i^{\sigma}}}
\end{align}

where $\rank^+$ and $\rank^{-}$ are any integers in $[0,R]$ and satisfying  $0\leq \rank^++\rank^{-}\leq k$, with the convention that  a sum of $\rank^{\pm}=0$ elements equals $0$,  $(\mathbf q_i^\sigma)_{\sigma\in\{+,-\},\,i\in \{1,...,\rank^{\sigma}\}}$ are any orthonormal vectors in $\mathbb R^k$, and $\rho_{\pm}=\pm 1$. In particular, the  totally collapsed equilibrium $U^{\star}=0, V^{\star}=0$ is stable.
\end{proposition}
\Cref{prop:stability_ssl} shows that collapse is not ruled out by the self-distillation dynamics: the trivial equilibrium is stable, and so are certain low-rank equilibria. 
Non-collapsed solutions are stable equilibria spanning the largest possible number of independent canonical directions. Specifically, let $\rank^*$ be the largest index with $f_1(c_{\rank^*},\lambda)>0$ and $\rank^{*}\leq \min(k,\numcca)$. Then non-collapsed solutions are equilibria of the form in \cref{eq:stable_equilibria_ssl} with $\rank^{+}=\rank^{*}$ or $\rank^{-}=\rank^{*}$. Stable equilibria with $\rank^{+},\rank^{-}<\rank^{*}$ are partially collapsed, as they do not span a maximal CCA subspace. 
In \cref{subsec:peira-recovery}, we show that \method yields equilibria with a similar form but without stable collapse.

\textbf{Sketch of proofs.} The general proof strategy and key results are provided in \cref{sec:hilbert_schmidt_analysis} with the full proof provided in \cref{proof:stability_ssl}.
We show that the gradient of the Lyapunov function is well aligned with the vector field in  \cref{eq:ode_ssl_appendix}, making it a \emph{gradient-like} system \citep{Chill:2009}. This property allows us to leverage the framework of \citet{Bolte:2014} based on the
Lojasiewicz--Simon gradient inequality \citep{Lojasiewicz:1982,Feehan:2020b} to establish convergence in \cref{prop:general_convergence_critical_point_0}. 
The characterization and stability of equilibria then follow from the spectral decomposition of the Jacobian of the vector field at equilibrium, derived in \cref{prop:diagonalization_jacobian}, together with the classification of equilibria established in \cref{prop:stability_kappa_1}.

\subsection{Characterizing the \method equilibria and their stability}
\label{subsec:peira-recovery}
We now study the gradient flow of the \method objective and characterize its critical points and their stability. 
The first result shows that the flow converges to critical points with the same canonical-subspace structure as the self-distillation  equilibria.

\begin{proposition}[Identification of equilibria]
\label{prop:stability-peira}
Under the same assumptions as \cref{prop:convergence_ssl}, the gradient flow dynamics $(\dot{U}_t,\dot{V}_t) = -\nabla_{U,V} \mathcal{E}_{\method}(U_t,V_t)$ with any initial condition $(U_0,V_0)\in \mathcal{U}\times \mathcal{V}$  converges to a critical point $(U^{\star},V^{\star})$ of $\mathcal{E}_{\method}$  
of the form:
\begin{align}
	U^{\star}(x) =  \frac{1}{\sqrt{2}}\sum_{i\in D} g(c_i,\lambda) \psi_i(x){\bf{q_i}},\qquad V^{\star}(y) =  \frac{1}{\sqrt{2}}\sum_{i\in D} g(c_i,\lambda) \xi_i(y){\bf{q_i}},
\end{align}
where $D$ is any subset of $\{1,\dots,\numcca\}$ with $|D|\le k$,  $({\bf{q_i}})_{i\in D}$ are any orthonormal vectors in $\mathbb{R}^k$, and $g$ is the spectral filter given by $
	g(c,\lambda) = \max\parens{\sqrt{c}-\lambda, 0}^{\frac{1}{2}}$.
\end{proposition}
The critical points in \cref{prop:stability-peira} share the canonical form of the self-distillation equilibria, but use a different spectral filter and contain only positively correlated components. Here, the filter acts as a soft threshold on the canonical correlations. Direct verification shows that $\Sigma_{U^{\star},V^{\star}}$ and $N_{U^{\star},V^{\star}}$ have the same eigenvectors, satisfying the alignment condition discussed in \cref{sec:method}. Since $D$ is arbitrary, this characterization alone does not ensure recovery of the leading canonical directions. The next proposition shows that the stable equilibria are exactly the ones spanning the largest maximally correlated subspaces, and that they are global minimizers of $\mathcal{E}_{\method}$.
\begin{proposition}[Stability]
\label{prop:peira-thresholded-subspace}
Under the same assumptions as \cref{prop:convergence_ssl}, let $0<\lambda<\sqrt{c_1}$ and $1\leq\rank_{\mathrm{max}}\leq \min(k,\numcca)$ be the largest integer such that $g(c_{\rank_{\mathrm{max}}},\lambda)>0$, with $(c_{i})_{1\leq i\leq \numcca}$ from \cref{def:cca}. Then the only stable equilibria are global minimizers of $\mathcal{E}_{\method}$ and take the form:
\begin{align}
\label{eq:csr-maximizer}
U^{\star}(x)
=
\frac{1}{\sqrt{2}}\sum_{i=1}^{\rank_{\mathrm{max}}}
g(c_i,\lambda) \psi_i(x) {\bf{q_i}}, \qquad V^{\star}(y)
=
\frac{1}{\sqrt{2}}\sum_{i=1}^{\rank_{\mathrm{max}}}
g(c_i,\lambda) \xi_i(y) {\bf{q_i}},
\end{align}
where $({\bf{q_i}})_{i=1}^{\rank_{\mathrm{max}}}$ are any orthonormal vectors in $\mathbb{R}^k$. 
Any other critical point of $\mathcal{E}_{\method}$ is a strict saddle point, and thus unstable. 
Moreover, the optimal value of the objective is: 
\begin{align}\label{eq:optim_val_peira}
\min_{U,V\in\mathcal U\times \mathcal{V}}\mathcal{E}_{\method}(U,V)=-\frac{1}{2}\sum_{i=1}^{\rank_{\mathrm{max}}}(\sqrt{c_i}-\lambda)^2.
\end{align}
\end{proposition}

\cref{prop:peira-thresholded-subspace}  implies that all critical points of \(\mathcal{E}_{\method}\) are either global minimizers or strict saddles, ruling out spurious local minima. 
Thus, the gradient flow generically\footnote{Dynamics do not converge \emph{generically} towards unstable equilibria or strict saddle points, that is, this only occurs for particular initializations  \citep{lee2016gradient,truong2020backtrackingbanach}.} converges to a global minimizer spanning the top-$\rank_{\mathrm{max}}$ canonical subspace.   
Consequently, stable collapse does not occur: any critical point spanning a lower-dimensional canonical subspace is unstable. In particular, in contrast to the self-distillation dynamics of \cref{subsec:ssl-cca}, the trivial collapsed representation $U=0$, $V=0$ is unstable for \method as it does not reach the optimal value in \cref{eq:optim_val_peira}.  
The effective dimension $\rank_{\mathrm{max}}$ is selected by $\lambda$ and is capped by $k$. 
Thus, when the target dimension is unknown, one can choose $k$ conservatively large and let $\lambda$ determine the number of active canonical directions. 
The condition $\lambda<\sqrt{c_1}$ is necessary: for $\lambda\geq\sqrt{c_1}$, the condition $\sqrt{c_i}>\lambda$ filters out all directions, so that $\rank_{\mathrm{max}}=0$ and the trivial solution is the unique stable equilibrium.

\textbf{Sketch of proofs.}
Full proofs are provided in \cref{proof:stability_method}. 
The convergence part follows from the general convergence result for \emph{gradient-like} systems, as in the case of \cref{prop:convergence_ssl}. The stability analysis then relies on the eigen-decomposition of the Hessian at critical points, derived in \cref{prop:diagonalization_jacobian}, together with the resulting
classification of all critical points in \cref{prop:stability_kappa_0}.

\section{Optimizing \method's objective}
\label{sec:algo}

\subsection{An auxiliary objective for gradient computation} 
Computing the gradient of $\mathcal{E}_{\method}$ requires differentiating through the dependence of the optimal regressor $P_{U,V}$ on $(U,V)$ in \cref{eq:closed-form_expression_P}. 
In practice, closed-form mini-batch estimates of the regressor, differentiated through by backpropagation, can lead to high-variance gradient estimates and unstable training \citep{Xu:2020a,Galashov:2025}.

Instead, we derive an auxiliary objective that has the same gradient as $\mathcal{E}_{\method}$ but which needs only estimating the regressor $P_{U,V}$ without differentiating it. 
This enables averaging techniques, such as exponential moving averages (EMA), for estimating \(P_{U,V}\), which  naturally reduce stochastic fluctuations and can improve training stability \citep{morales2024exponential}.
\begin{proposition}
\label{prop:gradient_expression_appendix}
The objective $\mathcal{E}_{\method}$ is continuously differentiable. Its gradient is
$L$-Lipschitz with $L=\frac{4}{\lambda}$ for any $0<\lambda<1$. 
Moreover, for any elements $U$ and $V$ in $\mathcal U$ and $\mathcal V$, the gradient satisfies $\nabla \mathcal{E}_{\method}(U,V)
=
\partial_{(U,V)}\mathcal{L}_{\mathrm{aux}}(U,V;P^\ast,Q^\ast)$, where $P^\ast=P_{U,V}$, $Q^\ast=\parens{N_{U,V}+\lambda I}^{-1}$, with $P_{U,V}$ and $N_{U,V}$ defined in \cref{eq:closed-form_expression_P}, whereas $\mathcal{L}_{\mathrm{aux}}$ is the following auxiliary objective:
\begin{equation}
	\label{eq:auxiliary-loss}
\begin{aligned}
\mathcal{L}_{\mathrm{aux}}(U,V;P,Q)
:=&
\frac{1}{2} \mathbb{E}\brackets{ U(x)^{\top}Q\parens{PU(x)-V(y)}  + V(y)^{\top}Q\parens{PV(y)-U(x)}} \\
&+\frac{\lambda}{2}\mathbb{E}\brackets{\Verts{U(x)}^2 + \Verts{V(y)}^2}.
\end{aligned}
\end{equation}

\end{proposition}
\cref{prop:gradient_expression_appendix} ensures the gradient is $\frac{4}{\lambda}$-Lipschitz whenever $0<\lambda<1$, thus supporting gradient-based optimization and informing the learning-rate choice \citep{Nesterov:2003}. 
The gradient identity follows by computing the expression of the gradient of $\mathcal{E}_{\method}$ and noticing that it matches the partial derivative w.r.t. $(U,V)$ of the function $\mathcal{L}_{\mathrm{aux}}$ of $4$ variables $U$, $V$, $P$ and $Q$, when the last two variables are set to $P^*$ and $Q^{*}$.
Next, we describe a stochastic algorithm leveraging such an identity.

\subsection{A stochastic compositional algorithm for \method } 

Optimization over \(\mathcal U\) and \(\mathcal V\) is infeasible and statistically ill-posed with finite data.  In practice, one restricts to parametric encoders \(U_\theta:\mathcal X\to\mathbb R^k\) and \(V_\psi:\mathcal Y\to\mathbb R^k\), typically neural networks. While the guarantees in \cref{prop:stability-peira,prop:peira-thresholded-subspace} do not directly apply, empirical scaling laws suggest that increased model size, data, and compute improve objective optimization \citep{Hestness:2017,Kaplan:2020}. We therefore use this parametrized setting as a scalable approximation to the population problem studied above, and discuss the resulting gap in \cref{sec:conclusion}.

Using the gradient expression in \cref{prop:gradient_expression_appendix}, we derive \cref{alg:peira} for learning encoders from paired data \(\mathcal D=\{(x,y)\}\). The algorithm exploits that \(P^\star\) and \(Q^\star\) are nonlinear functions of the population uncentered cross-covariance and auto-covariance of the features to maintain EMA estimates of these quantities that reduce gradient variance, in analogy with stochastic compositional gradient methods for objectives involving nonlinear functions of expected values \citep{wang2017stochastic}. \cref{fig:algo} illustrates the resulting gradient estimator.

\begin{figure}[H]

\begin{minipage}{0.45\linewidth}
At each iteration of \cref{alg:peira}, running estimates $\hat{P}^{\star}$ and $\hat{Q}^{\star}$ of $P^{\star}$ and $Q^{\star}$ are updated from the current features. Specifically, $\Sigma$ and $N$ maintain EMA estimates of $\Sigma_{U,V}$ and $N_{U,V}$ from mini-batch features $\Phi_X= (U_\theta(x_i))_{i=1}^B$ and $\Phi_Y= (V_{\psi}(y_i))_{i=1}^B$, stacked as matrices in $\mathbb R^{B\times k}$, yielding closed-form approximations of $P^{\star}$ and $Q^{\star}$. 
These quantities define a mini-batch estimate \(\widehat{\mathcal{L}}_{\mathrm{aux}}\), through which the encoder parameters are updated by backpropagation, with \(\hat P^\star\) and \(\hat Q^\star\) held fixed (lines 7--8).
This is efficient for moderate $k$, since the linear estimators are computed analytically from running feature statistics.
\end{minipage}\hfill
\begin{minipage}{0.53\linewidth}
\captionsetup{type=algorithm}
\rule{\linewidth}{0.8pt}\vspace{-0.9ex}
\caption{\algocomp{}}
\label{alg:peira}
\hrule
\begin{algorithmic}[1]
\Require Initial encoder parameters $\theta,\psi$, optimizer
\Require Initial $k\times k$ matrices $\Sigma,N$, dataset $\mathcal D$
\Require Regularization $\lambda$, EMA rate $\eta\in(0,1]$ (weight of the current batch)

\While{not converged}
    \State Sample a mini-batch $\mathcal B$ from $\mathcal D$
    \State $\Phi_X, \Phi_Y\leftarrow (U_{\theta}(x))_{(x,y)\in \mathcal{B}},(V_{\psi}(y))_{(x,y)\in \mathcal{B}}$
    \State $\Sigma \leftarrow (1-\eta)\Sigma +  \frac{\eta}{B}(\Phi_X^{\top}\Phi_Y + \Phi_Y^{\top}\Phi_X)$
    \State $N \leftarrow (1-\eta)N +  \frac{\eta}{B}(\Phi_X^{\top}\Phi_X + \Phi_Y^{\top}\Phi_Y)$
    \State $\hat{P}^{\star},\hat{Q}^{\star} \leftarrow \Sigma(N +\lambda I)^{-1}, (N +\lambda I)^{-1}$
    \State $g_\theta, g_\psi \leftarrow \nabla_{\theta,\psi} \widehat{\mathcal{L}}_{\mathrm{aux}}(U_\theta,V_\psi;\hat{P}^{\star},\hat{Q}^{\star})$
    \State Update $\theta, \psi$ using gradients $g_\theta,g_\psi$
\EndWhile
\end{algorithmic}
\vspace{0.85ex}
\hrule
\end{minipage}

\end{figure}

\begin{table}[t]
\centering
\caption{\small Linear-probe top-1 accuracy (\%) after pretraining a ResNet-50 for $100$ epochs on ImageNet-1K and a ResNet-18 for $1000$ epochs on CIFAR-10. VICReg, SIGReg, and \method{} numbers are mean $\pm$ one-sigma sample standard deviation over $3$ seeds, run with the same augmentations, optimizer and probe; SIGReg uses a narrower projector (hidden width $2048$ vs.\ $8192$). BYOL and SimSiam ImageNet-1K numbers are quoted from Table~4 of \citet{Bardes:2022} ($100$-epoch ResNet-50 reimplementation under each method's original recipe). The SimCLR ImageNet-1K number is quoted from Table~S3 of \citet{RankMe} ($100$-epoch ResNet-50, 8192-8192-2048 projector). Transfer: mean linear-probe accuracy over the $8$ LeJEPA datasets (VICReg, SIGReg: seed $1$; \method{}: mean $\pm$ sample standard deviation over $3$ seeds; \cref{app:transfer}). $k$: \method{} projector output dimension. Best controlled result per column in bold. ``--'': not run or no comparable published result.}
\label{tab:in1k}
\resizebox{\textwidth}{!}{%
\begin{tabular}{llccc}
\toprule
Method & Approach & ImageNet-1K~$\uparrow$ & CIFAR-10~$\uparrow$ & Transfer~$\uparrow$ \\
\midrule
BYOL~\citep{Grill:2020a}        & self-distillation (EMA)            & $69.3$  & -- & -- \\
SimSiam~\citep{Chen:2021b}      & self-distillation (stop-gradient) & $67.9$            & -- & -- \\
SimCLR~\citep{Chen:2020c}       & contrastive (negatives)     & $68.5$            & -- & -- \\
\midrule
VICReg~\citep{Bardes:2022}      & optimization                       & $68.81 \pm 0.09$  & $90.92 \pm 0.14$ & $\mathbf{72.07}$ \\
SIGReg~\citep{Balestriero:2025} & optimization                       & $66.26 \pm 0.28$  & $\mathbf{92.34 \pm 0.21}$ & $68.27$ \\
\method{} ($k{=}256$, ours)     & optimization                       & $66.45 \pm 0.11$  & $91.36 \pm 0.12$ & $70.89 \pm 0.18$ \\
\method{} ($k{=}1024$, ours)    & optimization                       & $\mathbf{69.17 \pm 0.06}$  & $91.68\pm 0.18$ & $71.09 \pm 0.24$ \\
\bottomrule
\end{tabular}%
}
\end{table}

\begin{figure}[t]
\centering
\includegraphics[width=.32\linewidth]{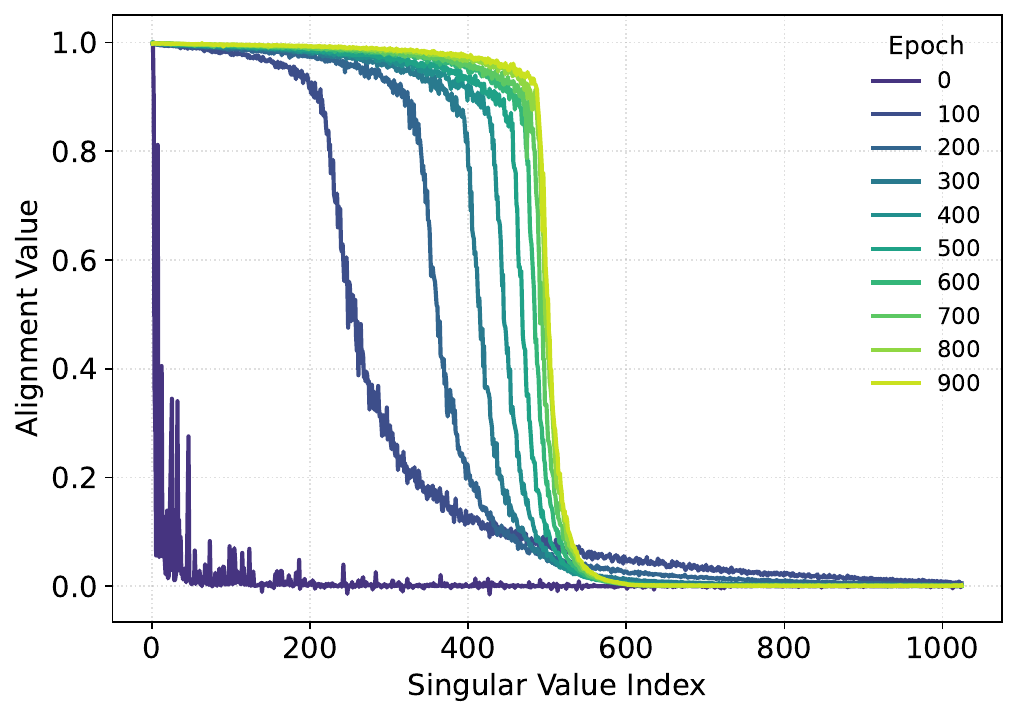}
\includegraphics[width=.32\linewidth]{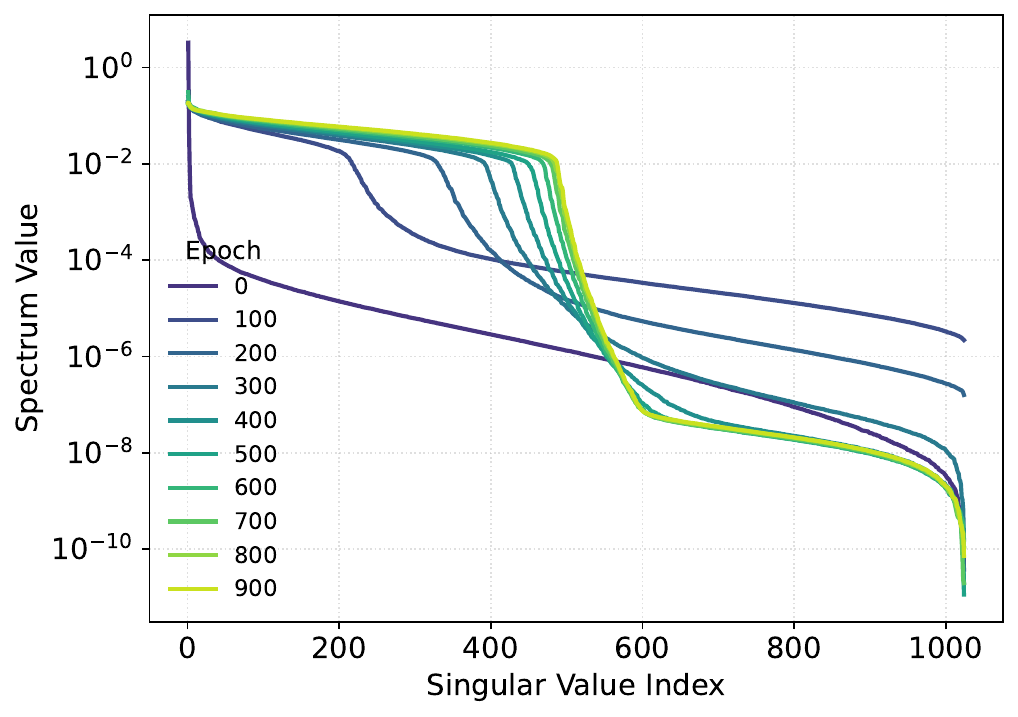}
\includegraphics[width=.32\linewidth]{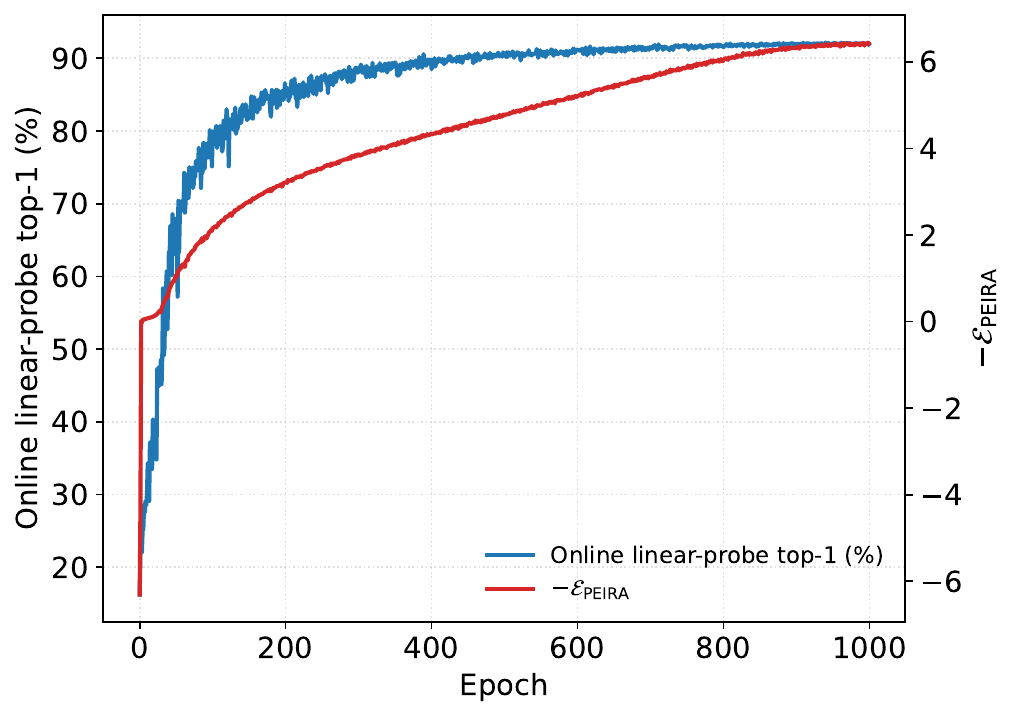}
\caption{\small \textbf{Empirical dynamics of PEIRA on CIFAR-10.} \method{} on a single training run. \emph{Left:} alignment $\alpha_i = e_i^{\top} \Sigma_{U,V} e_i / (\Verts{e_i}\Verts{\Sigma_{U,V}e_i})\in[0,1]$ of the top eigenvectors $e_i$ of the noise $N_{U,V}$ with respect to the signal $\Sigma_{U,V}$ during training. $\alpha_i=1$ implies $e_i$ is also an eigenvector of $\Sigma_{U,V}$. Both matrices gradually align their eigenvectors during training.
\emph{Middle:} spectrum of $N_{U,V}$ during training develops additional significant modes. 
\emph{Right:} validation top-1 and the negative \method{} objective $-\mathcal{E}_{\method}$ during training are strongly correlated.}  
\label{fig:cifar10-diagnostics}
\vspace{-1.em}
\end{figure}

\begin{remark} \textbf{Shared encoders.}
	When the two views are symmetric, a single shared encoder can be used for both, as is common in practice \citep{Chen:2021b}. The algorithm can then be adapted by setting $V_{\psi}=U_{\theta}$ and performing a single gradient update with $\nabla_{\theta}\widehat{\mathcal{L}}_{\mathrm{aux}}(U_{\theta}, U_{\theta};\hat{P}^{\star},\hat{Q}^{\star})$.
\end{remark}
\begin{remark}\label{remark:ema-schedule} \textbf{EMA-rate schedule.}
The rate $\eta$ weights the current mini-batch, so $1-\eta$ is the EMA momentum. In practice, $\eta$ is annealed from $\eta_{\mathrm{init}}$ to a smaller value $\eta_{\min}$ over training, favoring fresh statistics early and lower estimator variance near convergence.
\end{remark}
\begin{remark}\label{remark:bilevel}
\textbf{Joint updates.} 
\cref{eq:csr-objective,eq:csr-optimal-predictor} 
yield a bilevel problem, enabling joint updates of \(P^\star\), \(Q^\star\), and the encoders \citep{Dagreou:2022a,Arbel:2022a}, similarly to BYOL \citep{Grill:2020a}  and I-JEPA \citep{Assran:2023a}.
This would avoid explicit matrix inversion, but add hyperparameters and potentially degrade estimates of \(P^\star\) and \(Q^\star\).
\end{remark}

\section{Experiments}
\label{sec:experiments}

\begin{figure}[t]
\centering
\includegraphics[width=0.32\linewidth]{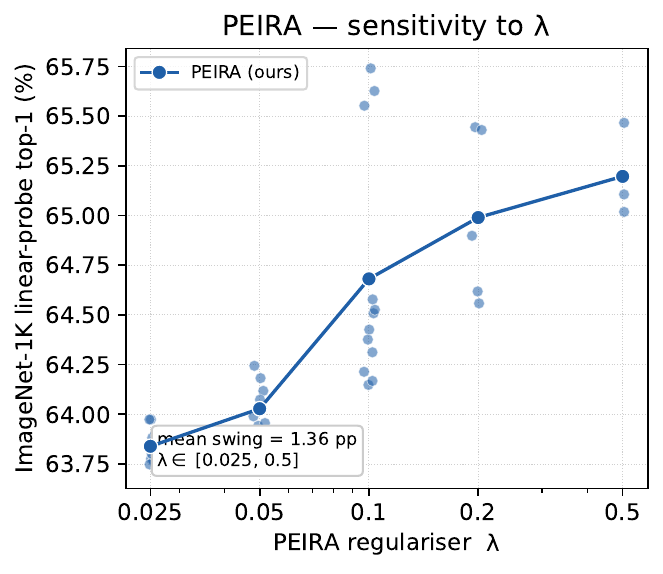}\hfill
\includegraphics[width=0.32\linewidth]{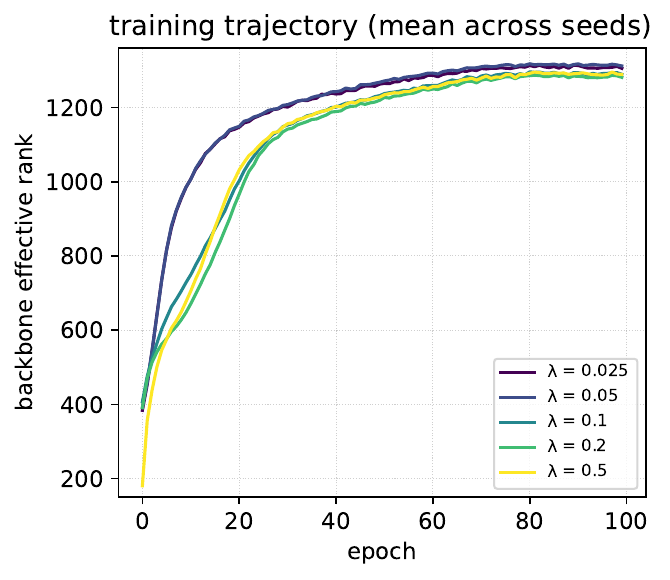}\hfill
\includegraphics[width=0.32\linewidth]{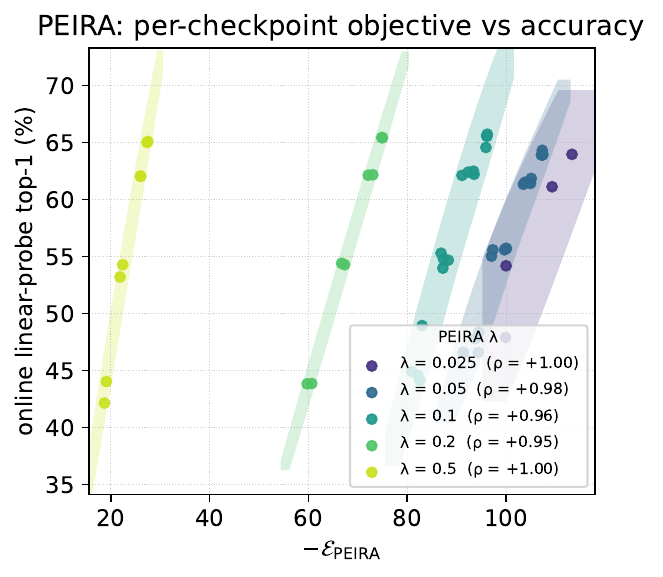}
\caption{\small \textbf{Study of \method{} regularization on ImageNet-1K.} ResNet-50, $100$ epochs, \method{} ablation setting (\cref{par:ablation-recipe}); $1$--$3$ random seeds per cell. \emph{Left:} sensitivity of the final online linear-probe top-1 accuracy to the regularizer $\lambda$; each point represents a hyperparameter configuration, with $\eta_{\min}$, optimizer weight decay, and gradient clipping varied across runs. \emph{Middle:} backbone entropy-based effective rank~\citep{RankMe} versus epoch, colored by~$\lambda$; smaller~$\lambda$ accumulates rank faster in the early phase. 
\emph{Right:} online linear-probe top-1 versus the negative objective $-\mathcal{E}_{\method}$ (mirroring the rightmost CIFAR-10 panel of \cref{fig:cifar10-diagnostics}), evaluated for checkpoints at epochs $25,\,50,\,75,\,100$ across runs spanning various $\lambda$ values; each dot is one (cell, checkpoint) pair, pooling over $\eta_{\min}$, optimizer weight decay, and gradient clipping. The per-$\lambda$ Spearman correlation is reported as $(\rho)$ in the legend.}
\label{fig:in1k-peira}
\vspace{-1.em}
\end{figure}

\textbf{Setup.} We evaluate \method{} through linear-probe classification on ImageNet-1K and CIFAR-10, and transfer learning on eight downstream classification datasets using ImageNet-1K-pretrained representations. We further investigate feature alignment, dimensional collapse, and sensitivity to training settings. We compare against VICReg~\citep{Bardes:2022} and SIGReg/LeJEPA~\citep{Balestriero:2025}, standardizing augmentations, pretraining budgets, and linear-probe evaluation while retaining method-specific projectors and optimization hyperparameters. On ImageNet-1K, VICReg follows the recipe of \citet{Bardes:2022,garrido2023on}, while SIGReg is configured based on the experiments in \cref{sec:appendix:sigreg_sensitivity}. On CIFAR-10, the baseline configurations selected at $300$ epochs by \citet{ebjepa} are evaluated for $1000$ epochs. For \method{}, the EMA schedule is fixed for all results in \cref{tab:in1k}, and only the regularization $\lambda$ is dataset dependent ($\lambda=0.7$ on CIFAR-10, $\lambda=0.1$ on ImageNet-1K). Full experimental details are provided in \cref{sec:training_details}.

\textbf{Representation quality.}
\cref{tab:in1k} summarizes results on ImageNet-1K, CIFAR-10 and the eight transfer tasks. 
On ImageNet-1K, \method{} with $k=1024$ achieves the strongest performance under our controlled evaluation, outperforming both VICReg and SIGReg. 
Its accuracy is also comparable to the published BYOL result and exceeds those of SimSiam and SimCLR, although these comparisons are indicative, as the reported results were obtained outside our controlled training protocol.
On CIFAR-10, \method{} outperforms VICReg and achieves performance close to SIGReg.  
Transfer results further demonstrate the generalization ability of \method{}: with $k=1024$, it outperforms SIGReg on average and approaches VICReg's mean performance (\cref{tab:in1k}). Detailed results in \cref{tab:app-transfer} of \cref{app:transfer} show that \method{} surpasses SIGReg on seven of eight datasets and achieves the best performance on three. 
Interestingly, the negative \method{} objective closely tracks linear-probe accuracy on CIFAR-10 (\cref{fig:cifar10-diagnostics}, right) and remains strongly correlated with it on ImageNet-1K (\cref{fig:in1k-peira}, right), suggesting its potential as a label-free indicator of representation quality. 
Additional experiments show that \method{} maintains competitive performance with longer pretraining (\cref{app:longtrain}), and achieves training times comparable to the baselines (\cref{app:wallclock}).

\textbf{Spectral alignment and resistance to collapse.} 
\Cref{fig:cifar10-diagnostics} (left and middle) shows that the leading eigendirections of the signal matrix $\Sigma_{U,V}$ progressively align with those of the noise matrix $N_{U,V}$, while additional significant spectral modes emerge, consistent with our population analysis. To assess resistance to collapse, we compare \method{} with its self-distillation counterpart across initialization scales on CIFAR-10 (\cref{app:cifar-initvar}). \method{} maintains high accuracy and substantial effective rank across scales $\{0.01,0.1,1\}$, whereas self-distillation exhibits greater sensitivity to initialization and pronounced dimensional collapse at smaller scales. This is consistent with \cref{prop:stability_ssl}: at zero initialization scale, the encoders produce the zero representation, which is a stable attractor of the self-distillation dynamics. Further experiments on synthetic data with known CCA structure confirm the predicted spectral selection and recovery of canonical directions (\cref{app:synthetic}).

\textbf{Sensitivity to hyperparameters.}
\Cref{tab:in1k} shows that projector dimension has a greater impact on ImageNet-1K than on CIFAR-10: reducing $k$ from $1024$ to $256$ decreases accuracy by $2.72$ and $0.32$ percentage points, respectively. A broader study confirms improvements up to $k=1024$ on ImageNet-1K (\cref{app:width}), suggesting that higher-dimensional representations are particularly beneficial on more challenging datasets. 
\Cref{fig:in1k-peira} (left) further demonstrates robustness to regularization, with mean accuracy varying by only $1.36$ percentage points over $\lambda\in[0.025,0.5]$. Smaller $\lambda$ accelerates early growth of backbone effective rank, while final ranks differ by only approximately $2\%$ (\cref{sec:appendix:eff_rank}),  consistent with the spectral role of $\lambda$ predicted by \cref{prop:peira-thresholded-subspace}. Finally, additional ablations show limited sensitivity to the EMA schedule and batch sizes (\cref{app:ema-bs}).

\section{Conclusions}
\label{sec:conclusion}
We have analyzed the JEPA dynamics in the  setting  where the predictor is linear although the encoders are arbitrary nonlinear square-integrable functions, identified the corresponding equilibria and their stability, and proposed an alternative objective that provably leads to nontrivial global optimizers. We have also demonstrated through preliminary experiments with standard benchmarks the promise of this approach. We now briefly discuss its limitations and future work.

\textbf{Limitations.} Our analysis is carried out at the population level and assumes access to the optimal linear predictor $P_{U,V}$. This predictor is a deliberate design choice, but in practice it is only estimated: \cref{alg:peira} replaces population moments with finite-sample EMA estimates, continuous-time flows with discrete stochastic updates, and nonparametric encoders with neural networks. Each approximation is principled in isolation: discrete noisy recursions track limiting continuous-time dynamics under suitable conditions~\citep{Benaim:1999}; empirical covariance operators concentrate around their population counterparts~\citep{Koltchinskii:2017}; averaging reduces estimation variance subject to a bias-variance trade-off~\citep{polyak1992acceleration}; and expressive networks approximate functions in $L^2$~\citep{Hornik:1991}, with convergence towards population optima proved in overparameterized and mean-field regimes~\citep{Chizat:2018,Mei:2018}. A rigorous end-to-end guarantee would require jointly controlling discretization, covariance estimation, EMA bias and variance, matrix-inverse perturbations, approximation, and nonconvex stochastic optimization, e.g., through the theory of singularly perturbed systems~\citep{Habets:2010,Arbel:2022a}, and is left for future work. 
Empirically, \method{} is robust to the quality of the predictor estimate across EMA schedules and batch sizes (\cref{app:ema-bs}). 
Our proofs further assume distinct canonical correlations, a simplifying assumption that can be relaxed at the cost of a more complicated characterization (\cref{app:degenerate}). 
Finally, although the experiments have demonstrated competitive performance on benchmarks like ImageNet-1K, more thorough experiments are needed to establish \method as a viable alternative to BYOL, SimSiam, VICReg, V-JEPA 2 and SIGReg for example, notably through experiments on more complex tasks such as video classification~\citep{ssv2, Kinetics400}.

\textbf{Future work.} Most (but not all) of our discussion and proofs are valid for encoders that belong to Hilbert spaces of infinite dimension, and predictors that are continuous linear operators over these. We plan to investigate next how much of the theory remains valid in this general setting and whether using reproducing kernel Hilbert spaces might lead to a practical computational solution. 
As shown by recent work on so-called world models~\citep{Sobal:2025,terver2026drivessuccessphysicalplanning}, learning action models in the form of control encoders and conditioning predictors on these models has led to new approaches to robotic planning. We also plan to investigate how \method{} could be used in this setting, the simplest but probably unsatisfactory instance being to concatenate the state and control features.

\begin{ack}
This work was supported in part by the French government under management of Agence Nationale de la Recherche as part of the ``France 2030'' program, PR[AI]RIE-PSAI project, reference ANR-23-IACL-0008. J.P. was supported in part by the Louis Vuitton/ENS chair in artificial intelligence, the Institute of Information \& Communications Technology Planning \& Evaluation (IITP) grant funded by the Korean Government (MSIT) (No.\ RS-2024-00457882, National AI Research Lab Project), and a Global Distinguished Professorship at the Courant Institute of Mathematical Sciences and the Center for Data Science at New York University. 
M.A. was supported by the ANR project BONSAI (grant ANR-23-CE23-0012-01) and was
granted access to the HPC resources of IDRIS under the allocation [AD011016476R1] made by GENCI.
\end{ack}

\bibliography{extracted_bib,cr_extra}

\newpage
\appendix
\numberwithin{figure}{section}
\numberwithin{table}{section}

\newpage
\section*{Appendix}
\addcontentsline{toc}{section}{Appendix}

\begingroup
\let\clearpage\relax
\hypersetup{linkcolor=black}
\tableofcontents
\endgroup

\newpage
\appendix

\section{Preliminaries and Background on Canonical Correlation Analysis}\label[appendixsection]{sec:background}
\subsection{Nonlinear canonical correlation analysis}\label[appendixsection]{appendix:cca} 
We start by briefly recalling population nonlinear CCA
\citep{Lancaster:1958,Michaeli:2015}, which seeks representations
$(U,V)\in\mathcal U\times\mathcal V$ maximizing inter-view correlation under orthonormality constraints. 
Specifically, these are  solutions to the following constrained optimization problem:
\begin{align}\label{eq:non_linear_CCA}
	\max_{(U,V)\in \mathcal{U}\times \mathcal{V}} & \mathbb{E}\brackets{U(x)^{\top}V(y)}  ,\qquad  \text{s.t.} \quad &\mathbb{E}\brackets{U(x)U(x)^{\top}} = I, \quad \mathbb{E}\brackets{V(y)V(y)^{\top}} = I.
\end{align}
Such representations can be thought of as \emph{predictive} in the sense that $V(y)$ is most predictable from $U(x)$ and vice versa.
We need the following mild regularity assumption on the data distribution to ensure nonlinear CCA admits a spectral characterization \citep{Michaeli:2015}:
\begin{assumplist}
\item\label{assump:density_ratio}
The density ratio $f(x,y)$, i.e., Radon-Nikodym derivative, of $\mathbb P$ with respect to
$\mathbb P_X\otimes\mathbb P_Y$ exists and is integrable under $\mathbb{P}$, i.e., 
$\mathbb E[f(x,y)]<\infty$.
\end{assumplist}
This is a mild assumption (see \cref{remark_assumption}) ruling out degenerate cases such as a deterministic dependence between $x$ and $y$.  Under \cref{assump:density_ratio}, CCA admits a spectral characterization in terms of the canonical directions $(\psi_{i},\xi_i)_{1\leq i\leq \numcca}$ introduced in \cref{def:cca}. 
Specifically, for $k\leq \numcca$, a solution to the CCA problem is obtained by stacking the first $k$
canonical directions:
\begin{align}
U(x) = \sum_{i=1}^k f_i \psi_i(x)
\qquad
V(y) = \sum_{i=1}^k f_i \xi_i(y),
\end{align}
where $(f_i)_{1\leq i\leq k}$ are the columns of the $k\times k$ identity matrix. Accordingly, the top-$k$ canonical subspaces are the spans of these directions.

\begin{remark}\label{remark_assumption}
\cref{assump:density_ratio} holds, for example, on compact domains when \(x\) and \(y\) are obtained by independently perturbing latent paired samples \((x_0,y_0)\) through kernels whose densities are uniformly bounded above and below. Indeed, if \(0<m_X\le k_X(x\mid x_0)\le M_X\) and \(0<m_Y\le k_Y(y\mid y_0)\le M_Y\), then \(p(x,y)\le M_XM_Y\), \(p(x)\ge m_X\), and \(p(y)\ge m_Y\), so \(f(x,y)=\frac{p(x,y)}{p(x)p(y)}\le \frac{M_XM_Y}{m_Xm_Y}\). For example, one may take \(k_X(x\mid x_0)=(1-\varepsilon_X)q_X(x\mid x_0)+\varepsilon_X/\operatorname{vol}(\mathcal X)\), where \(q_X\) is any uniformly bounded local perturbation density on \(\mathcal X\) and \(\varepsilon_X\in(0,1)\), and similarly for \(k_Y\).
\end{remark}

\subsection{Characterizing nonlinear CCA via generalized eigenvalue decomposition (GEVD)}\label{sec:appendix:cca-gevd}
A GEVD problem aims to find eigenpairs
$(s,w)$ such that $
Aw = s\,Bw$, where $A$ and $B$ are bounded symmetric operators in some Hilbert space $\mathcal{H}$, with $B$ positive definite. 
GEVD problems naturally arise in a range of machine learning problems \citep{Sun:2009} and, in particular, provide a tool for performing CCA.
Specifically, a solution of the CCA problem is characterized from a particular GEVD problem by taking its top-$k$ eigenvectors and applying an orthonormalization procedure, such as Gram-Schmidt, to enforce the constraints in \cref{eq:non_linear_CCA}. 
Specifically, $B$ is the identity operator defined on the product space $\mathcal{H} = L_2(\mathbb{P}_X)\times L_2(\mathbb{P}_Y)$ of real-valued square integrable functions w.r.t. $\mathbb{P}_X$ and $\mathbb{P}_Y$, whereas $A$ is the uncentered cross-covariance operator defined on the same space as:
\begin{align}\label{eq:operators}
	\langle w,Aw'\rangle_{\mathcal{H}}= \mathbb{E}\brackets{u(x)v'(y) + u'(x)v(y)},\qquad \forall 
	w=(u,v), w' = (u',v')\in \mathcal{H}. 
\end{align}
By Cauchy--Schwarz, $A$ is bounded with operator norm at most $1$:
\begin{align}\label{eq:bounded_op_norm_A}
	\verts{\langle w,Aw'\rangle_{\mathcal{H}}}\leq \Verts{w}_{\mathcal{H}}\Verts{w'}_{\mathcal{H}},\qquad \forall\, w,w' \in \mathcal{H}.
\end{align}
To ensure the problem\footnote{Since the operator $B$ is the identity, such problem reduces  to an EVD problem of the operator $A$, albeit, possibly, in infinite dimensions.} has a well-behaved leading spectrum, we need to assume that the operator $A$ is compact\footnote{A compact operator is a linear operator that maps bounded sequences to relatively compact sequences; equivalently, every image sequence has a convergent subsequence.}, which holds as soon as \cref{assump:density_ratio} holds \citep{Michaeli:2015}. In this case, the canonical directions $(\psi_i,\xi_i)_{1\leq i\leq \numcca}$ and canonical correlations $(c_i)_{1\leq i\leq \numcca}$ correspond to  eigenvectors and positive eigenvalues of the GEVD defined by the operators $A$ and $B$.  
Recovering representations $U,V$ that span the top-$k$ subspace then yields a CCA solution after orthogonalization.

\textbf{Eigen structure of the cross-covariance operator $A$.} 
The following proposition provides an eigen decomposition for operator $A$ under \cref{assump:density_ratio}, which allows expressing the solution of non-linear CCA and will be crucial in our analysis.  
\begin{proposition}\label{prop:compactness_operator_A}
	Under \cref{assump:density_ratio}, the operator $A$ is compact and admits a spectral decomposition. Specifically, letting $I=\mathbb{N}$ if $\mathcal{H}$ has infinite dimensions or $I= \{1,\dots,\dim(\mathcal{H})\}$ otherwise, there exists an orthonormal basis $(\zeta_i)_{i\in I}$ of $\mathcal{H}$ and eigenvalues $(s_i)_{i\in I}$   
	 satisfying $A\zeta_i = s_i \zeta_i$. 
	 Let $I_{+}$ be the index set of positive eigenvalues, and $\sigma$ be an ordering of the indices in $I_{+}$ such that  $s_{\sigma(1)}\geq \dots \geq s_{\sigma(\numcca)} >0$, with $\numcca$ possibly infinite (in which case $s_{\sigma(j)}\rightarrow 0$). Set $c_j = s_{\sigma(j)}$ and $(\psi_j,\xi_j) = \sqrt{2}\zeta_{\sigma(j)}$. Then if $c_j$ are all distinct, $(\psi_j)_{j= 1}^\numcca$ and $(\xi_j)_{j=1}^\numcca$ are orthonormal families of $L_2(\mathbb{P}_X)$ and $L_2(\mathbb{P}_Y)$ satisfying $\mathbb{E}\brackets{ \psi_j(x)\xi_{j'}(y)}= c_j \delta_{j=j'}$ and that exhaust all cross-correlated directions (see  \cref{def:cca}).  
	 Moreover, any index $i\in I$ corresponding to negative eigenvalue is of the form $\verts{s_i} = c_{j}$ for some $j\in \{1,\dots, \numcca\}$, and its eigenvector $\zeta_i$ satisfies $\sqrt{2}\zeta_i = (\psi_j,-\xi_j)$. The converse is also true: for any index  $j\in \{1,\dots, \numcca\}$ there exists an index $i\in I$ with  negative eigenvalue $s_i$ satisfying $\verts{s_i}=c_j$. 
\end{proposition}

\subsection{Top-$k$ subspaces and ratio-trace optimization.}\label{sec:appendix:topk}
As illustrated through CCA, one often only needs to recover a \emph{top-$k$ subspace representation}, i.e. an element $W = (U,V) \in \mathcal{U}\times \mathcal{V}$ whose $k$ components span the same subspace as the eigenvectors of the $k$ largest generalized eigenvalues of a given GEVD problem. 
For a GEVD problem of finite dimension $D$ and when $A\succeq 0$, this subspace was characterized as the span of rows of any  maximizer of the ratio-trace objective
\citep{Wang:2007,Jia:2009}:
\begin{align}\label{eq:ratio_trace_problem}
	\max_{W\in \mathbb{R}^{k\times D}} Tr\Big((WAW^{\top})(WBW^{\top})^{-1}\Big).
\end{align} 
While \cref{eq:ratio_trace_problem} appears to offer an appealing route to solving the GEVD via iterative optimization, it is \emph{basis-invariant}, i.e. for any invertible $Q\in\mathbb R^{k\times k}$, $W$ and $QW$ attain the same value. Consequently, the above objective is flat along re-parameterizations within the optimal subspace and is typically ill-conditioned for iterative optimization. 

The regularization we introduce in the linear predictor considered by \method results in a modified \emph{ratio-trace} objective that breaks this invariance and becomes smooth. The resulting criterion selects a particular subset of optimal subspaces, generally without enforcing $B$-orthonormality.

\section{General Proof Strategy and Main Results}\label[appendixsection]{sec:hilbert_schmidt_analysis}
\subsection{General Setting and Notations}\label{sec:appendix:setting}

\paragraph{Hilbert spaces of square-integrable functions.} 
Recall the product space $\mathcal{H} = L_2(\mathbb{P}_X)\times L_2(\mathbb{P}_Y)$ of real-valued square integrable functions w.r.t. $\mathbb{P}_X$ and $\mathbb{P}_Y$,  where $\mathbb P_X$ and $\mathbb P_Y$ denote the marginals of the joint data distribution $\mathbb P$ of pairs $(x,y)$. Let  $\mathcal{W}$ be the product space $ \mathcal{U}\times \mathcal{V}$, where $\mathcal{U}$ and $\mathcal{V}$ are the function spaces of square-integrable encoders, $\mathcal U=L^2(\mathbb{P}_X;\mathbb{R}^k)$, $\mathcal V=L^2(\mathbb{P}_Y;\mathbb{R}^k)$. 
We will view elements $W = (U,V)\in \mathcal{W}$ as bounded operators from $\mathcal{H}$ to $\mathbb{R}^{k}$ acting on any element $w = (u,v)\in \mathcal{H}$ as $Ww = \mathbb{E}\brackets{U(x)u(x)+V(y)v(y)}\in \mathbb{R}^k$ and denote by $W^{\top}$ the adjoint of $W$. For an element $W = (U,V)\in \mathcal{W}$, we denote by $WAW^{\top}$ and $WW^{\top}$ the $k\times k$ matrices given by:
\begin{align}
	WAW^{\top} = \mathbb{E}\brackets{U(x)V(y)^{\top} + V(y)U(x)^{\top}},\quad WW^{\top} = \mathbb{E}\brackets{U(x)U(x)^{\top} + V(y)V(y)^{\top}}. 
\end{align}

\paragraph{Hilbert space of square-integrable sequences.} 
Let $I$ be the index set defined in \cref{prop:compactness_operator_A} which is equal to $\mathbb{N}$ if $\mathcal{H}$ has infinite dimensions, or to $\{1,\dots,\dim(\mathcal{H})\}$ otherwise.
Denote by $\ell^2(I)$ the space of square integrable sequences indexed by $I$ and let $(e_i)_{i\in I}$ be its canonical orthonormal basis. Note that if $I$ is finite, then $\ell^2(I)$ can be identified with $\mathbb{R}^{|I|}$. 
Furthermore, fix $k\in\mathbb N$ and consider the space of Hilbert--Schmidt operators  $\mathbb{W} = \mathrm{HS}(\ell^2(I),\mathbb R^k)$, i.e. the set of bounded linear operators  that map elements in $\ell^2(I)$ to $\mathbb{R}^{k}$ and have a finite Hilbert-Schmidt norm:
$\Verts{\cdot}_{\mathbb{W}}$: 
\begin{equation}
\|\ucoor\|_{\mathbb{W}}^2=\sum_{i\in I}\|\ucoor e_i\|_{\mathbb R^k}^2,\qquad
\langle \ucoor,\ucoor'\rangle_{\mathbb{W}}:=\sum_{i\in I} \langle \ucoor e_i,\ucoor' e_i\rangle_{\mathbb{R}^k},\qquad \forall \ucoor,\ucoor' \in \mathbb{W}.
\end{equation}
Any element $\ucoor\in \mathbb{W}$ can therefore be formally identified with a $k\times |I|$ matrix $\ucoor=(w_{p i})_{1\le p\le k,\ i\in I}$ with $\sum_{p,i}|w_{p i}|^2<\infty$. We denote by $\ucoor^{\top}$ its adjoint. 
Let $(b_1,\dots,b_k)$ be the standard orthonormal basis of $\mathbb R^k$ and define the entry-wise rank-one
operators which form an orthonormal basis of $\mathbb{W}$.
\begin{equation}
E_{p i}:=b_p\otimes e_i^\ast\in \mathbb{W},\qquad
E_{p i}e_i=b_p,\ \ E_{p i}e_j=0\ (j\neq i).
\end{equation}

\paragraph{Coordinate representation in the eigenbasis of the operator $A$.}
Denote by $\ucoor = (\ucoor_i)_{i\in I}$ the components of an element $W\in \mathcal{W}$ in the eigenbasis $(\zeta_i)_{i\in I}$ of the operator $A$ defined in \cref{eq:operators}. 
Each component $\ucoor_i$ is a vector in $\mathbb{R}^k$  so that $\ucoor$ can be viewed as an element of $\mathbb{W}$. Then the following holds for any $W\in \mathcal{W}$ with component $\ucoor\in \mathbb{W}$ in the eigenbasis
\begin{align}\label{eq:decomposition_in_basis}
	WAW^{\top} = \ucoor S\ucoor^{\top},\qquad WW^{\top} =  \ucoor\ucoor^{\top},
\end{align}
where $S$ is a bounded compact  operator on $\ell^2(I)$ that is diagonal in its canonical basis and whose diagonal elements contain the eigenvalues of $A$. 

\paragraph{Orthogonal decomposition  induced by the compact operator $S$.}
Let $S$ be the operator considered in \cref{eq:decomposition_in_basis}. Since $S$ is a diagonal operator in the canonical basis of $\ell^2(I)$, it satisfies $Se_i=s_i e_i$, where $(s_i)_{i\in I}$  are the eigenvalues of $A$. Define the index sets
\begin{align}\label{eq:indices_decomposition}
	I_+:=\{i\in I:\ s_i>0\},\qquad
I_0:=\{i\in I:\ s_i=0\},\qquad
I_-:=\{i\in I:\ s_i<0\},
\end{align} 
so that $I=I_+\sqcup I_-\sqcup I_0$, and assume, without loss of generality, that $(\verts{s_i})_{i\in I_{+}}$ and $(\verts{s_i})_{i\in I_{-}}$ are decreasing in their index. 
The following orthogonal decomposition holds:  
\begin{align}\label{eq:orthogonal_space_decomposition}
\ell^2(I)=\ell^2(I_+)\oplus \ell^2(I_0)\oplus \ell^2(I_-),\qquad
\ell^2(I_{\sigma}):=\overline{\mathrm{span}}\{e_i:\ i\in I_\sigma\}\ \ (\sigma\in\{+,-,0\}).
\end{align}
The above decomposition induces a block decomposition on the space  $\mathbb{W}$ in the sense that
\begin{equation}
\mathbb{W} 
=\mathrm{HS}(\ell^2(I_+),\mathbb R^k)\ \oplus\ \mathrm{HS}(\ell^2(I_0),\mathbb R^k)\ \oplus\ \mathrm{HS}(\ell^2(I_-),\mathbb R^k),
\end{equation}
with corresponding orthonormal bases $\{E_{p i}:\ i\in I_{\sigma}\}$ for $\sigma\in \{+,-,0\}$.

\paragraph{Vector fields on $\mathbb{W}$.}
Let $\lambda$ be a positive number. We define the family of vector fields over $\mathbb{W}$
\begin{align}\label{eq:general_vector_field}
	F^{\kappa}(\ucoor) = \ucoor\parens{ I - \parens{\ucoor^{\top}\ucoor + \lambda I}^{-1}S(\ucoor^{\top}\ucoor S)^{\kappa}\parens{\ucoor^{\top}\ucoor + \lambda I}^{-1}}, \qquad \forall \ucoor \in \mathbb{W}, 
\end{align}
where $\kappa$ is  a non-negative integer. In particular, we will consider two values: $\kappa=0$ will recover the coordinate representation of the vector field driving the gradient flow of \method's objective, whereas $\kappa=1$ will recover the coordinate representations of the vector fields driving the self-distillation dynamics as shown next.

\subsection{Reduction to Hilbert--Schmidt coordinate representations}
\label[appendixsection]{sec:reduction}
In this section we provide two results, proven in \cref{sec:proof_coordinate_decomposition}, allowing to reduce the analysis to the space of coordinate representations of elements $W\in \mathcal{W}$ in the eigenbasis of $A$, as we do later on.

\textbf{Operator expression of \method's gradient.} 
The following proposition allows expressing the \method's objective and its gradient in terms of the operator $A$, provides an expression of the coordinates of $\nabla\mathcal{E}_{\method}$ in the eigenbasis of $A$ and establishes the equivalence between the gradient flow dynamics of $\mathcal{E}_{\method}$ and the dynamics of its components in the eigenbasis of $A$.
\begin{proposition}[Operator expressions of \method's objective and its gradient]\label{prop:Lipschitz_gradient}
For given $W=(U,V) \in \mathcal{W}$, the objective $\mathcal{E}_{\method}$ is smooth and  admits  
an $L$-Lipschitz gradient, with $L = 4\lambda^{-1}$ when $0<\lambda <1$. Additionally, the following expressions hold for $\mathcal{E}_{\method}$ and $\nabla\mathcal{E}_{\method}$  in terms of the operator $A$ in \cref{eq:operators}: 
 \begin{equation}
 \begin{aligned}
	\mathcal{E}_{\method}(W) &= -\frac{1}{2}\Tr\parens{WAW^{\top}\parens{WW^{\top}+\lambda I}^{-1}} + \frac{\lambda}{2}\Tr\parens{WW^{\top}}\\
		\nabla \mathcal{E}_{\method}(W) 	 &= \lambda W\parens{ I- \parens{W^{\top}W+\lambda I}^{-1}A\parens{W^{\top}W+\lambda I}^{-1}}.
\end{aligned}
\end{equation}
Moreover, assuming \ref{assump:density_ratio} holds,  let $\ucoor = (\ucoor_i)_{i\in I}\in \mathbb{W}$ denote the components of an element $W\in \mathcal{W}$ in the eigenbasis $(\zeta_i)_{i\in I}$ of $A$ and let $S$ be the diagonal operator $S$ defined in \cref{eq:decomposition_in_basis}. Then the components of $\nabla\mathcal{E}_{\method}(W)$ in such eigenbasis are of the form: 
\begin{align}\label{eq:vector_field_kappa_0_proof}
	\mathbf{g} = \lambda F^{0}(\ucoor), \qquad F^{0}(\ucoor) =  \ucoor\parens{I-\parens{\ucoor^{\top}\ucoor +\lambda I}^{-1}S\parens{\ucoor^{\top}\ucoor +\lambda I}^{-1} }.
\end{align}
Finally, if $(W_t)_{t\geq 0}$ is a trajectory in $\mathcal{W}$ and $(\ucoor_t)_{t\geq 0}$ are its components in the eigenbasis of $A$, then the following equivalence holds:
\begin{align}
	\dot{W}_t=-\nabla \mathcal{E}_{\method}(W_t) \iff \dot{\ucoor}_t = -\lambda F^{0}(\ucoor_t).
\end{align}

\end{proposition}

\textbf{Operator expression of the Self-distillation vector field.} 
Denote by $\mathcal{F}(U,V)$ the vector field arising from the self-distillation dynamics in   \cref{eq:ode_ssl_appendix}:
\begin{align}
	(\dot{U}_t,\dot{V}_t) =  -\mathcal{F}(U_t,V_t),\qquad  \mathcal{F}(U,V)= (\partial_U \mathcal{P}_{U\triangleright V}(P_{U,V},U,V), \partial_V \mathcal{P}_{V\triangleright U}(P_{U,V},V,U)).
\end{align}
The following proposition expresses $\mathcal{F}$ in terms of the operator $A$ and expresses the coordinates of  $\mathcal{F}(W)$ in the eigenbasis of $A$ and establishes the equivalence between the dynamics in \cref{eq:ode_ssl_appendix} and that of its components in the eigenbasis of $A$. 
\begin{proposition}\label{eq:expression_SSL_vector_field}
The vector field $\mathcal{F}(U,V)$ driving the dynamics in \cref{eq:ode_ssl_appendix} admits the expression:
	\begin{align}\label{eq:vector_field_kappa_1_proof}
		\mathcal{F}(W) = \lambda W\parens{I- \parens{W^{\top}W+\lambda I}^{-1} AW^{\top} W A\parens{W^{\top}W+\lambda I}^{-1}}, \qquad \forall W\in \mathcal{W}.
	\end{align}
Moreover, assuming \ref{assump:density_ratio} holds,  let $\ucoor = (\ucoor_i)_{i\in I}\in \mathbb{W}$ denote the components of an element $W\in \mathcal{W}$ in the eigenbasis $(\zeta_i)_{i\in I}$ of $A$ and let $S$ be the diagonal operator $S$ defined in \cref{eq:decomposition_in_basis}. Then the components of $\mathcal{F}(W)$ in such eigenbasis are of the form: 
\begin{align}
	\mathbf{g} = \lambda F^{1}(\ucoor), \qquad F^{1}(\ucoor) =  \ucoor\parens{I-\parens{\ucoor^{\top}\ucoor +\lambda I}^{-1}S\ucoor^{\top}\ucoor S\parens{\ucoor^{\top}\ucoor +\lambda I}^{-1} }.
\end{align} 

Finally, if $(W_{t})_{t\geq 0}$ is a trajectory in $\mathcal{W}$ and $(\ucoor_t)_{t\geq 0}$ are its components in the eigenbasis of $A$, then the following equivalence holds:
\begin{align}
	\dot{W}_t=-\mathcal{F}(W_t) \iff \dot{\ucoor}_t = -\lambda F^{1}(\ucoor_t).
\end{align}

\end{proposition}

\textbf{Reduction to coordinate representation.} The above results imply that it suffices to study  dynamical systems of the form $\dot{\ucoor}_t=-\lambda F^{\kappa}(\ucoor)$ defined in $\mathbb{W}$, with $F^{\kappa}$ being the vector fields defined in \cref{eq:general_vector_field} for $\kappa\in \{0,1\}$. Next, we study the properties of such vector fields, including stability of its critical points and convergence of the dynamics arising from them.

\subsection{Convergence of Dynamical Systems to Critical Points}
\label[appendixsection]{sec:convergence}
Here, we will show  convergence to a critical point of $F^{\kappa}$ for the following dynamical system: 
\begin{align}\label{eq:synamical_system_kappa}
	\dot{\ucoor}_t = -\lambda F^{\kappa}(\ucoor_t),\qquad \ucoor_0\in \mathbb{W}.
\end{align} 
The key step is to show that the above system admits a suitable Lyapunov function that makes it a gradient-like system \citep{Chill:2009}.
We will then apply a general convergence result for gradient-like systems in infinite dimensional spaces (here $\mathbb{W}$), stated below and proven in  \cref{sec:convergence_gradient_like_systems}, after verifying that all conditions of such a result are satisfied by  \cref{eq:synamical_system_kappa}. 
\begin{proposition}\label{prop:general_convergence_critical_point_0}
Let $F$ be a smooth vector field defined on a separable Hilbert space $\mathbb{W}$ and consider the following dynamical system:
\begin{align}\label{eq:dynamical_system_0}
	\dot{\ucoor}_t = - F(\ucoor_t),\qquad \ucoor_0\in \mathbb{W}.
\end{align}
Assume there exists a real analytic function $L$ defined over $\mathbb{W}$ such that the following holds:
\begin{assumplist}[resume]   
\item\label{assump_dissipation_0}\textbf{Dissipation.} $L$ is a Lyapunov function for the system, i.e. $\dot{L}(\ucoor_t)\leq 0$ for all $t\geq 0$.  
\item\label{assump_sufficient_decrease_0} \textbf{Sufficient decrease.} There exists $c_{max}>0$ such that $\Verts{F(\ucoor_t)}^2\leq - c_{max} \dot{L}(\ucoor_t)$ for  $t\geq 0$.
\item\label{assump_error_bound_0} \textbf{Relative  bound.} There exists  a non-increasing real-valued function $\Omega: [0,+\infty)\rightarrow (0,+\infty)$  such that $\Omega(\Verts{\ucoor_t})\Verts{\nabla L(\ucoor_t)}^2 \leq \Verts{F(\ucoor_t)}^2$.
\item\label{assump_coercive_0}
\textbf{Coercivity.} $L$ is coercive, i.e. $L(\ucoor)\rightarrow +\infty$ when $\Verts{\ucoor}\rightarrow +\infty$.
\item\label{assump_lower_bound_0} \textbf{Lower bound.} $L$ admits a finite lower bound.
\item\label{assump_bounded_on_bounded_set_0} \textbf{Radial growth.} The vector field $F$, gradient $\nabla L$ and Hessian $\nabla^2 L$ all have radial growth, i.e. there exists a non-decreasing real-valued function $\omega:[0,+\infty) \rightarrow [0,+\infty)$ such that $\max\parens{\Verts{F(\ucoor)},\Verts{\nabla L(\ucoor)}, \Verts{\nabla^2 L(\ucoor) }}\leq \omega(\Verts{\ucoor})$ for all $\ucoor\in \mathbb{W}$.
\item\label{assump_kl_0} 
\textbf{Fredholm property.} For any $\ucoor\in \mathbb{W}$, the hessian $\nabla^2 L(\ucoor)$ can be decomposed as the sum of a boundedly invertible operator $R$ and a bounded compact operator $K$.
\item\label{assump_palais_smale_condition_0} \textbf{Palais-Smale condition.} Any bounded sequence $(\ucoor_n)_{n\geq 1}$ of $\mathbb{W}$ such that $\Verts{\nabla L(\ucoor_n)}\rightarrow 0$ admits a convergent subsequence.
\end{assumplist}
The dynamical system in \cref{eq:dynamical_system_0} is defined at all times and converges to a critical point of $F$.
\end{proposition}
To apply \cref{prop:general_convergence_critical_point_0}, we need to identify a suitable Lyapunov function $L^{\kappa}$ for the system and check that it verifies all conditions of the result. Specifically, we  identify the following candidate function $L^{\kappa}$:
\begin{align}\label{eq:Lyapunov_func}
	L^{\kappa}(\ucoor) = \frac{1}{3}Tr\parens{\parens{\ucoor\ucoor^{\top} + \lambda I}^3} - \frac{1}{\kappa + 1}Tr\parens{\parens{\ucoor S\ucoor^{\top}}^{\kappa+1} }.
\end{align}
Note that $L^{1}$ is the coordinate version of the Lyapunov function $\mathcal{L}$ stated in \cref{eq:lyapunov_ssl}. 
The following proposition, proven in \cref{sec:proof_smale_etc}, 
establishes that \cref{eq:Lyapunov_func} is indeed a Lyapunov function of the system satisfying the alignment conditions with the vector field $F^{\kappa}$ required by  \cref{prop:general_convergence_critical_point_0}.     
\begin{proposition}[Existence of a Lyapunov function]\label{prop:lyapunov_function}
Let $\kappa\in \{0,1\}$ and $\lambda >0$. 
The function $L^{\kappa}$ defined in \cref{eq:Lyapunov_func} is a Lyapunov function of the  dynamical system in \cref{eq:synamical_system_kappa}, i.e. $\dot{L}^{\kappa}(\ucoor_t)\leq 0$ for all $t\geq 0$. Moreover, the following inequalities hold along the trajectory:
\begin{align}\label{eq:sufficient_decrease}
	 \frac{1}{4(\Verts{\ucoor_t}^2+\lambda)^4}\Verts{\nabla L^{\kappa}(\ucoor_t)}^2\leq \Verts{F^{\kappa}(\ucoor_t)}^2\leq -\frac{1}{2\lambda^3}\dot{L}^{\kappa}(\ucoor_t), \qquad \forall t\geq 0.
\end{align}
\end{proposition}
The following four lemmas, also proven in \cref{sec:proof_smale_etc}, establish geometric properties of the $L^{\kappa}$ and $F^{\kappa}$ that are needed for the application of \cref{prop:general_convergence_critical_point_0}.

\begin{lemma}[Growth property of $F^{\kappa}$.]\label{lem:growth_F_kappa}
$F^{\kappa}$ is radially bounded in the sense that for any $\ucoor\in \mathbb{W}$, it holds that:
\begin{align}
	\Verts{F^{\kappa}(\ucoor)} \leq \omega(\Verts{\ucoor}), \qquad \text{with}\qquad  \omega(R) = R\parens{1 + \lambda^{-2}\Verts{S}_{op}^{1+\kappa}R^{2\kappa}} \quad \forall R\in [0,+\infty)
\end{align}

\end{lemma}

\begin{lemma}[Growth, coercivity and lower boundedness of $L^{\kappa}$]\label{lem:energy_bounded_below}
 The function $L^{\kappa}$ defined in \cref{eq:Lyapunov_func} is coercive and lower-bounded over $\mathbb{W}$. 
 Moreover, the gradient $\nabla L^{\kappa}$ and Hessian $\nabla^2 L^{\kappa}$ both have a radial growth, i.e. there exists a non-decreasing real-valued function $\omega: [0,+\infty) \rightarrow [0,+\infty)$ such that:
 \begin{align}
 	\Verts{\nabla L^{\kappa}(\ucoor)}\leq \omega(\Verts{\ucoor}),\qquad \Verts{\nabla^2 L^{\kappa}(\ucoor)}\leq \omega(\Verts{\ucoor}),\qquad \forall \ucoor\in \mathbb{W}.
 \end{align}
\end{lemma}

\begin{lemma}[Hessian decomposition]\label{lem:kurdyka_simon_inequality}
The Hessian $\nabla^2 L^{\kappa}(\ucoor)$ at any point admits a decomposition of the form $\nabla^2 L^{\kappa}(\ucoor)=R+K$, where $R$ is a boundedly invertible operator (bounded with bounded inverse) and $K$ is a bounded compact operator.
\end{lemma}

\begin{lemma}[A Palais-Smale condition.]\label{lem:palais_smale}
The function $L^{\kappa}$ satisfies the following  condition: any bounded sequence $(\ucoor_n)_{n\geq 1}$ in $\mathbb{W}$ satisfying $\Verts{\nabla L^{\kappa}(\ucoor_n)}\rightarrow 0$ admits a convergent subsequence. 
\end{lemma}

We are now in position to prove the main result of this section stated below.
\begin{proposition}\label{prop:convergence_critial_point_F_kappa}
	The dynamical system \cref{eq:synamical_system_kappa} converges to a critical point of $F^{\kappa}$. 
\end{proposition}
\begin{proof}
We need to verify that assumptions \cref{assump_dissipation_0,assump_sufficient_decrease_0,assump_error_bound_0,assump_palais_smale_condition_0,assump_kl_0,assump_coercive_0,assump_lower_bound_0,assump_bounded_on_bounded_set_0} in \cref{prop:general_convergence_critical_point_0} hold for $F = \lambda F^{\kappa}$ and $L = L^{\kappa}$. Note that $F$ satisfies the same conditions in \cref{prop:general_convergence_critical_point_0} as $F^{\kappa}$ up to a positive factor. Hence, it suffices to establish such conditions for $F^{\kappa}$.   
By \cref{prop:lyapunov_function}, we know that $L^{\kappa}$ is a Lyapunov function of the dynamical system \cref{eq:synamical_system_kappa}. As a polynomial function, $L^{\kappa}$ is clearly real analytic. Hence, \cref{assump_dissipation_0} holds.  
	Moreover, the system is gradient-like, since  \cref{eq:sufficient_decrease} holds, so that \cref{assump_sufficient_decrease_0,assump_error_bound_0} are fulfilled. \cref{lem:energy_bounded_below} ensures that $L^{\kappa}$ is coercive and bounded below, so that \cref{assump_coercive_0,assump_lower_bound_0} hold. Moreover, it also ensures that $\nabla L^{\kappa}$ and $\nabla^2 L^{\kappa}$ are radially bounded, as required by \cref{assump_bounded_on_bounded_set_0}. By \cref{lem:growth_F_kappa}, we also have that $F^{\kappa}$ is radially bounded, as required by \cref{assump_bounded_on_bounded_set_0}. 
	By \cref{lem:kurdyka_simon_inequality}, we are guaranteed that \cref{assump_kl_0} on the Hessian decomposition holds. 
	Finally, \cref{lem:palais_smale} ensures that $L^{\kappa}$ satisfies the Palais-Smale condition of \cref{assump_palais_smale_condition_0}. 
	Therefore, we can apply \cref{prop:general_convergence_critical_point_0} to deduce that \cref{eq:synamical_system_kappa} converges to a critical point $\ucoor^{\star}$ of $F^{\kappa}$.
\end{proof}

\subsection{Characterization of Critical Points of $F^{\kappa}$}
\label[appendixsection]{sec:characterization}
The following proposition, proven in \cref{sec:proof_characterization_equilibria}, provides an explicit characterization of all critical points of the vector field $F^{\kappa}$.
\begin{proposition}\label{eq:characterization_fixed_point_general} 
Assume all non-zero eigenvalues of $S$ to be distinct. Then, critical points of $F^{\kappa}$ are exactly elements of $\mathbb{W}$ of the form:
\begin{align}\label{eq:critical_point_matrix_characterization}
	\ucoor = Q \parens{\sum_{p=1}^r E_{p,i(p)}} M^{\frac{1}{2}} = \sum_{p=1}^r m_{i(p)}^{\frac{1}{2}} q_{p}\otimes e_{i(p)}^{\star},
\end{align}
where $Q = (q_p)_{1\leq p\leq k}$ is any orthogonal matrix in $\mathbb{R}^{k\times k}$, $r$ is any integer between $0$ and $k$, $I_{sup} = \{i(1),\dots, i(r)\}\subset I_{+}\cup I_{-}$ is any increasing sequence of integers greater than or equal to $1$, and $M$ is a PSD operator on $\ell^2(I)$ diagonal in the canonical basis of $\ell^2(I)$ and whose only non-zero diagonal coefficients $(m_{i(p)})_{1\leq p\leq r}$ are positive roots of the equations $(x+\lambda)^{2} = x^{\kappa}s_{i(p)}^{\kappa+1}$ of unknown $x$ for $1\leq p\leq r$, where $(s_{i})_{i\in I}$ are the eigenvalues of the operator $A$. 
Specifically,  
\begin{align}
	m_{i(p)}^{\frac{1}{2}} = \begin{cases}
		g(s_{i(p)},\lambda)
		 & \kappa =0\\
		f_{\epsilon_{p}}(s_{i(p)},\lambda), & \kappa=1
	\end{cases},
\end{align}
where $(\epsilon_p)_{p=1}^r$ can take any values in $\{-1,1\}$ and $g$, $f_{\epsilon}(s,\lambda)$ are given by:
\begin{align}
	g(s,\lambda) = 
	\begin{cases}
		\parens{\sqrt{s}-\lambda}_+^{\frac{1}{2}} & s\geq 0\\
		0 & s<0
	\end{cases},\qquad 
	f_{\epsilon}(s,\lambda) = \frac{1}{2}\parens{ \verts{s} + \epsilon \sqrt{s^2 -4\lambda} }\mathds{1}_{s^2-4\lambda\geq 0}.
\end{align}
\end{proposition}
The functions $g$ and $f_{\epsilon}$ are spectral filters as they can filter out some eigendirections based on the value of the corresponding eigenvalue. For the case $\kappa=1$, we will call $f_{1}$ the \emph{stable filter}  and $f_{-1}$ the \emph{unstable filter} for reasons that will be clear in \cref{sec:stability}. 
 
The key idea behind the proof of \cref{eq:characterization_fixed_point_general} is to obtain an algebraic equation satisfied by the operator $M = (\ucoor^{\star})^{\top}\ucoor^{\star}$. From such equation, we then show that a certain polynomial of $M$ must commute with the diagonal operator $S$ (containing the eigenvalues of the cross-covariance operator $A$). The fact that all eigenvalues of $S$ are distinct ensures that such polynomial must also be diagonal. Then, after showing that the polynomial defines a bijective  spectral map, we are able to deduce that $M$ must also be diagonal. This diagonal structure  implies that the fixed point $\ucoor^{\star}$ must satisfy an orthogonality constraint which allows decoupling the critical point equation it satisfies into scalar equations. The fact that $M$ is of rank at most $k$ implies that at most $k$ of these equations can have non-zero solution. 
Solving these equations explicitly yields the desired characterization.

\subsection{Stability Analysis of Critical Points of $F^{\kappa}$}\label[appendixsection]{sec:stability}
We are now interested in characterizing which critical points of the vector field $F^{\kappa}$ are stable for a dynamics of the form $\dot{\ucoor} = -\lambda F^{\kappa}(\ucoor)$, i.e. critical points for which the Jacobian of $F^{\kappa}$ has only non-negative real parts of the eigenvalues, and those that are unstable (negative real part). 
The key ingredient to obtain such a characterization is to find a  diagonalization of the Jacobian $\nabla F^{\kappa}(\ucoor)$ at any critical point $\ucoor$ of $F^{\kappa}$ which allows studying stability: either verify that all eigenvalues are non-negative (stable equilibrium) or identify a negative eigenvalue (unstable equilibrium).

\textbf{Eigen decomposition of the Jacobian of the vector field $F^{\kappa}$.}
The following proposition, whose proof is deferred to \cref{sec:proof_diagonalization_jecobian}, provides the desired  decomposition.
\begin{proposition}[Diagonalization of the Jacobian at critical points]\label{prop:diagonalization_jacobian}
Assume all non-zero eigenvalues of $S$ to be distinct and let $\ucoor$ be a critical point of $F^{\kappa}(\ucoor)$, which, by virtue of \cref{eq:characterization_fixed_point_general},  necessarily has the form:
\begin{align}\label{eq:critical_point_matrix_characterization_2}
	\ucoor = Q \underbrace{\parens{\sum_{p=1}^r E_{p,i(p)}}}_{\ucoor_0}M^{\frac{1}{2}},
\end{align}
where $Q$ is an orthogonal matrix in $\mathbb{R}^{k\times k}$, $r$ is an integer between $0$ and $k$, $I_{sup} = \{i(1),\dots, i(r)\}\subset I_{+}\cup I_{-}$ an increasing sequence of integers greater or equal to $1$ and $M$ is a positive semi-definite diagonal operator on $\ell^{2}(I)$  whose only non-zero diagonal coefficients $(m_{i(p)})_{1\leq p\leq r}$ are positive roots of the equations $(x+\lambda)^{2} = x^{\kappa}s_{i(p)}^{\kappa+1}$ of unknown $x$ for $1\leq p\leq r$. 
Then the Jacobian $\nabla F^{\kappa}(\ucoor)$ at $\ucoor$ of the vector field admits an eigendecomposition of the form: $(\bar{E}_{p,j}, \mu_{p,j})_{ 1\leq p\leq k, j\in I}$, satisfying:
\begin{align}
	\bar{E}_{p,j} 
	= 
	\begin{cases}
		QE_{p,j}, & p>r \text{ or } j\notin I_{sup} \text{ or } (p\leq r \text{ and } j=i(p)) \\
		(1-c_{p,p'})^{\frac{1}{2}}QE_{p,i(p')} +c_{p,p'}^{\frac{1}{2}}QE_{p', i(p)},  &   p,p'\leq r  \text{ and } j=i(p')>i(p)\\
		c_{p,p'}^{\frac{1}{2}}QE_{p,i(p')} -  (1-c_{p,p'})^{\frac{1}{2}}QE_{p', i(p)},  &  p,p'\leq r  \text{ and } j=i(p')<i(p),
	\end{cases}
\end{align}
\begin{align}
	\mu_{p,j} = \begin{cases}
		1-\lambda^{-2}s_{j}\mathds{1}_{\kappa=0}, &  j\notin I_{sup} \text{ and } p>r\\
		\delta_{p,j}, &  j\notin I_{sup} \text{ and }  p\leq r\\
		\frac{4m_j}{m_j+\lambda} - 2\mathds{1}_{\kappa=1} , & j=i(p) \text{ and } p\leq r \\
		\gamma_{p,p'}, & p,p' \leq r \text{ and } j=i(p')> i(p)\\		
		0, & p' \leq r \text{ and } j=i(p') \text{ and }  \parens{ \parens{p\leq r \text{ and } j< i(p)} \text{ or } p>r}   
	\end{cases}
\end{align}
 with $c_{p,p'}$, $\gamma_{p,p'}$ and $\delta_{p,j}$ given by:
 \begin{equation}
 \begin{aligned}
 	c_{p,p'} &= \frac{m_{i(p')}}{m_{i(p')}+m_{i(p)}},\\
 	\gamma_{p,p'} &= \frac{\parens{m_{i(p)}+m_{i(p')}}\parens{ \verts{s_{i(p)}}^{\frac{\kappa+1}{2}}m_{i(p)}^{\frac{\kappa}{2}} + \verts{s_{i(p')}}^{\frac{\kappa+1}{2}}m_{i(p')}^{\frac{\kappa}{2}} - s_{i(p)}s_{i(p')}\mathds{1}_{\kappa=1} }  }{\parens{m_{i(p)}+\lambda}\parens{m_{i(p')}+\lambda} }\\
 	\delta_{p,j} &= \frac{(s_{i(p)}-s_j)}{\lambda (m_{i(p)} +\lambda)}\parens{1 + \parens{m_{i(p)}s_{i(p)}-1}\mathds{1}_{\kappa=1}}.
 \end{aligned}
 \end{equation}
\end{proposition}
A striking fact about the above proposition is that most  eigenvectors are simply \emph{rotated} versions of the canonical vectors $E_{p,j}$, with rotation $Q$ corresponding to the arbitrary orthogonal matrix defining the critical point $\ucoor$. The remaining ones are linear combinations of two rotated versions of paired canonical vectors $E_{p,i(p')}$ and $E_{p',i(p)}$ on which $\ucoor$ admits non-zero components. 
The main proof strategy here is to first derive the Jacobian expression at the critical point, and express it in terms of the operator $M=\ucoor^{\top}\ucoor$ which must be diagonal by virtue of \cref{eq:characterization_fixed_point_general}. The diagonal structure implies some terms involve commuting factors which are key to simplifying the expressions. Applying such Jacobian to each rotated canonical basis element of the form $\tilde{E}_{p,j}$ and employing suitable linear algebra identities results in two types of expressions: either an expression colinear to the original element $\tilde{E}_{p,j}$, or if $j=i(p')$ and $p\leq r$ it is a linear combination of it  and another rotated canonical basis element $\tilde{E}_{p',i(p)}$. We then further show that the Jacobian is stable under the subspace spanned by those two elements, which allows recovering the eigenvectors and eigenvalues corresponding to such subspace.

\textbf{Stability of equilibria.}
Using the above eigen decomposition, we can study the stability of any critical point $\ucoor$ by investigating the sign of the eigenvalues associated to each eigenvector of the Jacobian of the vector field $F^{\kappa}$ at that point. \cref{prop:diagonalization_jacobian} suggests restricting to $\lambda \leq 1$, otherwise the only acceptable equilibrium is the trivial representation. The following propositions, proven in \cref{sec:proofs_stability}, provide such characterization for each case $\kappa=0$ and $\kappa=1$ separately.

\begin{proposition}[Stability of maximal top-$r$ critical points $(\kappa=0)$]\label{prop:stability_kappa_0}
Assume $\kappa=0$, $0<\lambda<\sqrt{c_1}$ (i.e.\ $\lambda<1$, since $c_1=1$), $k\geq 1$, and that all non-zero eigenvalues of $S$ are distinct.
Let $\rank_{\mathrm{max}}$ be the largest integer such that $1\leq\rank_{\mathrm{max}}\leq \min(k,\numcca)$ and $\sqrt{s_{i(\rank_{\mathrm{max}})}} > \lambda $ and denote by $I_{sup} = \{i(1),\dots,i(\rank_{\mathrm{max}})\}$ the set of indices of the $\rank_{\mathrm{max}}$ largest positive eigenvalues of $A$ by decreasing order.
Let $M$ be the positive semi-definite diagonal operator on $\ell^2(I)$ whose only positive diagonal elements $(m_{i(p)})_{1\leq p \leq \rank_{\mathrm{max}}}$ are given by:
\begin{align}\label{eq:solution_m_i_p}
	m_{i(p)} = \sqrt{s_{i(p)}}-\lambda.
\end{align}
Then, for any orthogonal matrix $Q\in \mathbb{R}^{k\times k}$, the following critical point is stable for $F^{0}$:
\begin{align}\label{eq:decomposition_local_minimizers_kappa_0}
	\ucoor = Q \parens{\sum_{p=1}^{\rank_{\mathrm{max}}} E_{p,i(p)} }M^{\frac{1}{2}}.
\end{align}
Conversely, the only stable critical points of $F^{0}$ are of the above form for any orthogonal matrix $Q$. All other critical points of $F^{0}$ are unstable as the Jacobian of $F^{0}$ at these points necessarily admits a negative eigenvalue.  
\end{proposition}
The proof requires verifying that all eigenvalues of the identified stable critical points are non-negative. To verify all the other critical points are unstable, we identify precisely unstable directions which are of two types depending on the form of the critical points: type I or II.
Type I critical points are those that skip an eigen direction of $A$ of index $j_0$ associated to a positive eigenvalue $s_{j_0}$ while including another direction $i(p_0)$ with smaller eigenvalue $s_{i(p_0)}$. These possess an unstable direction associated to negative eigenvalue $\mu_{p_0,j_0}$ of $\nabla F^{\kappa}(\ucoor)$. 
Type II critical points contain  $r<k$ eigen components without gap, but still miss an eigen direction $j_0$ with small positive eigenvalue that is not filtered out by the spectral filter, i.e. $g(s_{j_0},\lambda)>0$, with $g(s,\lambda) = (\sqrt{s}-\lambda)_{+}^{\frac{1}{2}}$. These possess an unstable direction associated to negative eigenvalue $\mu_{r+1,j_0}$ of $\nabla F^{\kappa}(\ucoor)$.

\begin{proposition}[Stability of the top-$r$ critical points $(\kappa=1)$]\label{prop:stability_kappa_1}
Assume $\kappa=1$, $\lambda \leq 1$ and that all non-zero eigenvalues of $S$ are distinct. 
Let $r^{+}$ and $r^{-}$ be any non-negative integers in $[0,\numcca]$ such that  $r=  r^{+}+r^{-}\leq k$. 
Denote by $I_{sup}^{+} = \{i(1),\dots, i(r^{+})\}$ and  $I_{sup}^{-} = \{i(r^{+}+1),\dots, i(r)\}$, the  sets of indices of the $r^{+}$ largest positive eigenvalues and $r^{-}$ smallest negative eigenvalues of $A$ satisfying $\verts{s_{j}}\geq 2\sqrt{\lambda}$ for all $j\in I_{sup}= I_{sup}^{+}\cup I_{sup}^{-}$. 
We use the convention that $I_{sup}^{+} = \emptyset$ if $r^{+}=0$ and $I_{sup}^{-} = \emptyset$ if $r^{-}=0$. Let $M$ be the positive semi-definite diagonal operator on $\ell^2(I)$ whose diagonal elements $m_j=0$ for $j\notin I_{sup}$, whereas the elements  $(m_{i(p)})_{1\leq p \leq r}$ are given by:
\begin{align}\label{eq:solution_m_i_p_kappa_1}
	m_{i(p)}^{\frac{1}{2}} = \frac{1}{2}\parens{\verts{s_{i(p)}} + \sqrt{s_{i(p)}^2-4\lambda} }\mathds{1}_{s_{i(p)}^2-4\lambda\geq 0}.
\end{align}
Then, for any orthogonal matrix $Q\in \mathbb{R}^{k\times k}$, the following critical point is stable for $F^{1}$:
\begin{align}\label{eq:decomposition_local_minimizers_E}
	\ucoor = Q \parens{\sum_{p=1}^{r} E_{p,i(p)}}M^{\frac{1}{2}}.
\end{align}
Conversely, the only stable critical points of $F^{1}$ are of the above form for any orthogonal matrix $Q$ and any $0\leq r^{+},r^{-}\leq \numcca$ satisfying $r^{+}+r^{-}\leq k$. All other critical points are unstable. 
\end{proposition}
The proof of \cref{prop:stability_kappa_1} proceeds similarly as that of \cref{prop:stability_kappa_0}. To verify all the other critical points are unstable, we identify precisely unstable directions which are, again, of two types depending on the form of the critical points: type I or II.
Type I critical points are those that include an eigen direction of $A$ of index $i(p_0)$ but weighted by the unstable spectral filter $f_{-1}$ instead of the stable one $f_{1}$. These  possess an unstable direction associated to negative eigenvalue $\mu_{p_0,i(p_0)}$ of $\nabla F^{\kappa}(\ucoor)$. 
Type II critical points all use the stable spectral filter $f_{1}$ but present a gap in their spectra: they miss an eigen direction of index $j_0$ of eigenvalue $s_{j_0}$ but contain an eigendirection indexed by $i(p_0)$ of eigenvalue $s_{i(p_0)}$ with the same sign as $s_{j_0}$ but smaller in absolute value: $\verts{s_{j_0}}> \verts{s_{i(p_0)}}$. These critical points possess an unstable direction associated to negative eigenvalue $\mu_{p_0,j_0}$ of $\nabla F^{\kappa}(\ucoor)$.
\subsection{Degenerate canonical spectra and shared encoders}\label[appendixsection]{app:degenerate}

\paragraph{Repeated canonical correlations.}
\cref{prop:stability-peira,prop:peira-thresholded-subspace} assume distinct canonical correlations $c_i$. With repeated correlations, the characterization of critical points still holds: critical points remain in the span of the nonlinear CCA directions, but their expression involves an additional arbitrary rotation within each degenerate canonical subspace. The proof structure is unchanged, except that some diagonal operators, such as $M$ in \cref{prop:stability_kappa_0}, become block-diagonal, with one block per degenerate subspace. We present the non-degenerate case for clarity. For stability, we conjecture that the same conclusions hold; a rigorous proof requires a finer study of the singular vectors of the Jacobian operator $\nabla F^{\kappa}(\ucoor)$ within each degenerate block, which we leave for future work. The clustered spectrum of \cref{app:synthetic} probes this case empirically.

\paragraph{Shared encoder.}
\cref{prop:stability-peira,prop:peira-thresholded-subspace} are proved for the two-encoder objective, in which $U$ and $V$ vary independently. The shared-encoder variant of \cref{alg:peira} ($V=U$) is a practical specialization for symmetric views ($\mathbb{P}_X=\mathbb{P}_Y$ and a symmetric joint law). In that case, both encoders must be equal at the equilibria above, and stability follows by a similar analysis restricted to $V=U$. We present the general analysis because it does not rely on the symmetric-views assumption.

\section{Proofs of the main results}\label[appendixsection]{sec:proofs_main_results}
\subsection{Gradient flow of the \method objective}

\subsubsection{Convergence to critical points and characterization of critical points}\label[appendixsection]{proof:stability_method}

\begin{proof}[Proof of \cref{prop:stability-peira}] 
By \cref{prop:Lipschitz_gradient}, we have the following equivalence between the gradient flow $(W_t)_{t\geq 0}$ in $\mathcal{W}$ and an ODE defined on its components  $(\ucoor_t)_{t\geq 0}$ in the eigenbasis of $A$:
\begin{align}
	\dot{W}_t=-\nabla \mathcal{E}_{\method}(W_t) \iff \dot{\ucoor}_t = -\lambda F^{0}(\ucoor_t),
\end{align}
where the vector field $F^{0}$ is defined in \cref{eq:vector_field_kappa_0_proof}. 
Thus we can apply the results obtained for the dynamics arising from the vector field $F^{0}$. Specifically, we directly apply \cref{prop:convergence_critial_point_F_kappa} to such dynamics, which ensures convergence to a critical point $\ucoor^{\star}$.

Consider now the re-ordering $\sigma\colon \{1,\dots,\numcca\}\rightarrow I_{+}$ of indices associated to positive eigenvalues of $A$ from \cref{prop:compactness_operator_A}, so that $(s_{\sigma(i)})_{i=1}^\numcca$ are non-increasing with $\numcca=\verts{I_+}$ possibly infinite. Then, by \cref{eq:characterization_fixed_point_general}, we know that such critical point is necessarily of the form:
\begin{align}\label{eq:characterization_critical_points_kappa_0}
	\ucoor^{\star} = \sum_{i\in D}  \parens{\sqrt{s_{\sigma(i)}}-\lambda}^{\frac{1}{2}}_{+} q_i\otimes e_{\sigma(i)} 
	\end{align}
where $D$ is any, possibly empty, 
subset of $\{1,\dots,\numcca\}$ of at most $k$ elements, $(q_i)_{i\in D}$ is an orthonormal family in $\mathbb{R}^k$ and $(e_i)_{i\in I}$ is the canonical basis of $\ell^2(I)$.  
It remains to express critical points $W^{\star}= (U^{\star},V^{\star})$ of the gradient flow of $\mathcal{E}_{\method}$ in the eigenbasis of $A$ using the characterization of their components in \cref{eq:characterization_critical_points_kappa_0}. Let $W^{\star}$ be an element in $\mathcal{W}$ whose components in the eigenbasis $(\zeta_i)_{i\in I}$ are $\ucoor^{\star}$. Then, by definition, we directly have:
\begin{align}\label{eq:critical_point_proof_expression_kappa_0}
	W^{\star}(x,y) =  \sum_{i\in D} q_i \parens{\sqrt{s_{\sigma(i)}}-\lambda}^{\frac{1}{2}}_{+} \zeta_{\sigma(i)}(x,y).
\end{align}
Moreover, by \cref{prop:compactness_operator_A}, we know that for any $1\leq i\leq \numcca$, the elements $(\psi_i)_{i=1}^\numcca$ and $(\xi_i)_{i=1}^\numcca$ of $L_2(\mathbb{P}_X)$ and $L_2(\mathbb{P}_Y)$ given by $(\psi_i,\xi_i) = \sqrt{2}\zeta_{\sigma(i)}$ form orthonormal families satisfying $\mathbb{E}\brackets{\psi_i(x)\xi_j(y)}= c_i \mathds{1}_{i=j}$ with $c_i = s_{\sigma(i)}$. Consequently, we can write $W^{\star} = (U^{\star},V^{\star})$ with:
\begin{align}
	U^{\star}(x) =  \frac{1}{\sqrt{2}}\sum_{i\in D}  q_{i} \parens{\sqrt{c_i}-\lambda}^{\frac{1}{2}}_{+} \psi_i(x),\qquad V^{\star}(y) =  \frac{1}{\sqrt{2}}\sum_{i\in D} q_{i} \parens{\sqrt{c_i}-\lambda}^{\frac{1}{2}}_{+} \xi_i(y).
\end{align}
\end{proof}

\subsubsection{Minimizers span thresholded canonical correlation subspaces}\label[appendixsection]{proof:maximizers} 
\begin{proof}[Proof of \cref{prop:peira-thresholded-subspace}]
By \cref{prop:Lipschitz_gradient}, we have the following equivalence between the gradient flow $(W_t)_{t\geq 0}$ in $\mathcal{W}$ and an ODE defined on its components  $(\ucoor_t)_{t\geq 0}$ in the eigenbasis of $A$:
\begin{align}
	\dot{W}_t=-\nabla \mathcal{E}_{\method}(W_t) \iff \dot{\ucoor}_t = -\lambda F^{0}(\ucoor_t),
\end{align}
where the vector field $F^{0}$ is defined in \cref{eq:vector_field_kappa_0_proof}. Thus, we only need to study the stability of critical points of $F^{0}$. This is a direct consequence of   \cref{prop:stability_kappa_0} which ensures that the only stable critical points of $F^{0}$ are of the form:
\begin{align}
	\ucoor^{\star} = \sum_{i=1}^{\rank_{\mathrm{max}}} \parens{\sqrt{s_{\sigma(i)}}-\lambda}^{\frac{1}{2}}_{+} q_i \otimes e_{\sigma(i)} 
\end{align}
with $\rank_{\mathrm{max}}$ the largest integer such that $\sqrt{s_{\sigma(\rank_{\mathrm{max}})}}>\lambda$ and $\rank_{\mathrm{max}}\leq \min(\numcca,k)$. These are the coordinates of the following elements $(U^{\star},V^{\star})\in \mathcal{W}$ which are the critical points of $\nabla\mathcal{E}_{\method}$ described in \cref{prop:stability-peira}: 
\begin{align}
	U^{\star}(x) = \frac{1}{\sqrt{2}}\sum_{i=1}^{\rank_{\mathrm{max}}} q_i \parens{\sqrt{c_i}-\lambda}^{\frac{1}{2}}_{+} \psi_{i}(x),\qquad V^{\star}(y) =\frac{1}{\sqrt{2}} \sum_{i=1}^{\rank_{\mathrm{max}}} q_i \parens{\sqrt{c_i}-\lambda}^{\frac{1}{2}}_{+} \xi_{i}(y).
\end{align}
We will now show that these stable solutions are precisely the global minimizers of $W\mapsto\mathcal{E}_{\method}(W)$.  
Since \cref{assump:density_ratio} holds, we know by \cref{prop:compactness_operator_A} that $A$ is a compact operator and that there exists an orthonormal basis $\mathcal{B} = (\zeta_i)_{i\in I}$ of $\mathcal{H}$ consisting of eigenvectors of $A$, i.e. $A\zeta_i = s_i \zeta_i$, with $(s_i)_{i\in I}$ being the eigenvalues of $A$.

We first treat the  case where $k \leq dim(\mathcal{H})$. 
Consider $W\in \mathcal{W}$. Such element can be viewed as a Hilbert-Schmidt operator from $\mathcal{H}$ to $\mathbb{R}^k$. It is therefore compact, and thus, by the spectral theorem for compact operators \citep{Gohberg:2003}, it admits a singular value decomposition of the form:
\begin{align}
W=Q\Sigma T^*,
\end{align}
where \(Q\in O(k)\), \(\Sigma=\operatorname{diag}(\sigma_1,\dots,\sigma_k)\), and
\(T:\mathbb{R}^k\to \mathcal{H}\) is an isometric embedding.  Then
\begin{align}
WW^T=Q\operatorname{diag}(\sigma_1^2,\dots,\sigma_k^2)Q^*,
\end{align}
and, by cyclicity of the trace,
\begin{align}
\mathcal{E}_{\method}(W)
&=
-\frac12\operatorname{tr}\Bigl(T^*AT\,\operatorname{diag}\Bigl(\frac{\sigma_i^2}{\sigma_i^2+\lambda}\Bigr)\Bigr)
+\frac{\lambda}{2}\sum_{i=1}^k \sigma^2_i.
\end{align}
Moreover, let us decompose $A$ into the difference of two positive semi-definite operators $A_+$ and $A_{-}$, i.e. $A = A_{+} - A_{-}$, with $A_{+}$ containing all non-negative eigendirections, while $A_{-}$ contains the absolute value of the negative ones. More precisely, $A \zeta_i = A_{+}\zeta_i = s_i \zeta_i$ and $A_{-}\zeta_i =0$  for all indices $i$ such that $s_i\geq 0$, whereas $A \zeta_i = -A_{-}\zeta_i = s_i \zeta_i$  and $A_{+}\zeta_i=0$, for all indices such that $s_i<0$.  It is easy to see that 
\begin{equation}\label{eq:bound_trace}
\begin{aligned}
	\operatorname{tr}\Bigl(T^*AT\,\operatorname{diag}\Bigl(\frac{\sigma_i^2}{\sigma_i^2+\lambda}\Bigr)\Bigr) = &\operatorname{tr}\Bigl(T^*(A_+-A_{-})T\,\operatorname{diag}\Bigl(\frac{\sigma_i^2}{\sigma_i^2+\lambda}\Bigr)\Bigr)\\
	\leq & \operatorname{tr}\Bigl(T^*A_{+}T\,\operatorname{diag}\Bigl(\frac{\sigma_i^2}{\sigma_i^2+\lambda}\Bigr)\Bigr),
\end{aligned}
\end{equation}
since $A_{-}$ is positive semi-definite. 
For fixed \((\sigma_i)\), the numbers $d_i:=\frac{\sigma_i^2}{\sigma_i^2+\lambda}$ are non-negative, and we apply von Neumann's trace inequality  \citep{Grigorieff:1991}, which implies that the trace of the product of matrices $T^{*}A_{+}T$ and $diag(d_i)$  is upper-bounded by the sum of products of their singular values ordered by non-increasing order:
\begin{align}
	\operatorname{tr}\Bigl(T^*A_+T\,\operatorname{diag}\Bigl(\frac{\sigma_i^2}{\sigma_i^2+\lambda}\Bigr)\Bigr) \leq \sum_{i=1}^k \tilde{s}_id_{\beta(i)},
\end{align}
where  $(\tilde{s}_i)_{1\leq i\leq k}$ are the singular values of $T^*A_{+}T$ ordered  by non-increasing order, whereas $\beta(1), \dots, \beta(k)$ is a re-indexing to ensure that $d_{\beta(1)}\geq \dots \geq d_{\beta(k)}$. Since both matrices are PSD, these singular values are also eigenvalues. 
Thus, by Courant-Fisher's theorem~\citep[Corollary III.1.2]{Bhatia:2013}, such eigenvalues are upper-bounded by the top-$k$ eigenvalues of $A_{+}$ ordered by non-increasing order, which we denote as $s^{+}_1\geq \dots \geq s^{+}_k$, so that: 
\begin{align}\label{eq:courant_fisher}
	\sup_{T^{\star}T=I}\operatorname{tr}\Bigl(T^*A_+T\,\operatorname{diag}\Bigl(\frac{\sigma_i^2}{\sigma_i^2+\lambda}\Bigr)\Bigr) \leq \sum_{i=1}^k s^{+}_id_{\beta(i)}.  
\end{align}
The above inequality with \cref{eq:bound_trace} yields:
\begin{align}\label{eq:max_l_objective}
\inf_{W}\mathcal{E}_{\method}(W)
\geq
\min_{\sigma_1,\dots \sigma_k}\frac12\sum_{i=1}^k
\left(-s^{+}_i\frac{\sigma_{\beta(i)}^2}{\sigma_{\beta(i)}^2+\lambda}+\lambda\sigma^2_{\beta(i)}\right) = \min_{\sigma_1,\dots \sigma_k} \frac12\sum_{i=1}^k f_{s_i^{+}}(\sigma_{\beta(i)}),
\end{align}
for $f_{s}(\sigma)$ defined as follows:
\begin{align}
f_{s}(\sigma)
:=
\left(-s\frac{\sigma^2}{\sigma^2+\lambda}+\lambda\sigma^2\right)
\qquad \sigma\ge 0.
\end{align}
Thus the problem decouples into \(k\) scalar minimizations of the functions $\sigma \mapsto f_{s_i^+}(\sigma)$. By studying the variations of such a function, we see that the minimum is attained at $\sigma_i^{\star}= (\sqrt{s_i^{+}}-\lambda)_{+}^{\frac{1}{2}}$, whenever $s_i^{+}>0$ and by $\sigma_i^{\star}=0$ otherwise. Moreover, when $s_i^{+}>0$, this corresponds exactly to the positive eigenvalue $s_i$ of operator $A$ (without loss of generality we assumed that top-$k$ eigenvalues of $A$ are $s_1,\dots, s_k$). 
We know, by \cref{prop:compactness_operator_A}, that to each of the $\numcca$ positive eigenvalues of $A$ corresponds a negative one with opposite sign and vice versa, with $\numcca$ possibly infinite. Moreover, $\numcca\geq 1$ since we know that $1$ is necessarily an eigenvalue of $A$. 
Therefore $\mathcal{E}_{\method}$ satisfies:
\begin{align}
	\inf_{W\in \mathcal{W}} \mathcal{E}_{\method}(W) \geq  \frac{1}{2} \sum_{i=1}^{k} f_{s^{+}_i}(\sigma_i^{\star}) = \frac{1}{2} \sum_{i=1}^{\min(k,\numcca)} f_{s_i}(\sigma_i^{\star})  =-   \frac{1}{2}\sum_{i=1}^{\min(k,\numcca)} (\sqrt{s_i}-\lambda)_{+}^{2}. 
\end{align}
Furthermore, denoting by $\rank_{\mathrm{max}}\leq k$ the number of non-zero components, i.e. those for which $\sqrt{s_i}>\lambda$ and which we know satisfy $\rank_{\mathrm{max}}\leq \min(\numcca,k)$, we directly get 
\begin{align}\label{eq:opt_piera}
\inf_{W\in \mathcal{W}} \mathcal{E}_{\method}(W) \geq  -\frac{1}{2} \sum_{i=1}^{\rank_{\mathrm{max}}} (\sqrt{s_i}-\lambda)_{+}^{2}.
\end{align}
This lower bound is clearly attained for $W$ of the form 
\begin{align}\label{eq:minimizers_proof}
W^{\star}(x,y)=\sum_{i=1}^{\rank_{\mathrm{max}}} q_{i} (\sqrt{s_i}-\lambda)_{+}^{\frac{1}{2}} \zeta_i(x,y),
\end{align}
 where $(q_i)_{i=1}^{\rank_{\mathrm{max}}}$ are any orthonormal family of $\mathbb{R}^k$. Conversely, any minimizer of $\mathcal{E}_{\method}$ is of the form in \cref{eq:minimizers_proof}. To see this,  let $W$ be a minimizer of $\mathcal{E}_{\method}$. Since $\mathcal{E}_{\method}$ is differentiable, then $W$ must be a critical point of $\mathcal{E}_{\method}$. By \cref{prop:stability-peira}, we know such critical points must be of the form $W= (U,V)$ with
\begin{align}
	U(x) = \frac{1}{\sqrt{2}} \sum_{i\in D} q_i \parens{\sqrt{c_{i}}-\lambda}^{\frac{1}{2}}_{+} \psi_{i}(x), \qquad V(y) = \frac{1}{\sqrt{2}}  \sum_{i\in D} q_i \parens{\sqrt{c_{i}}-\lambda}^{\frac{1}{2}}_{+} \xi_{i}(y),
\end{align}
where $D$ is any, possibly empty, 
subset of $\{1,\dots, \numcca\}$ of at most $k$ elements, $(q_i)_{i\in D}$ is an orthonormal family in $\mathbb{R}^k$. Evaluating $\mathcal{E}_{\method}$ at such point yields:
\begin{align}
	\mathcal{E}_{\method}(W) = -\frac{1}{2} \sum_{i\in D} (\sqrt{c_i}-\lambda)_+^2.
\end{align}
 The above quantity is minimized exactly when $(c_i)_{i\in D}$ corresponds to the top $\rank_{\max}$-singular values of $A$, since these are assumed to be distinct. These are  also equal to the top-$\rank_{\max}$ positive eigenvalues of $A$ by recalling the structure of the eigenvalues of $A$ from \cref{prop:compactness_operator_A}. Since we assumed, without loss of generality, that $s_1, \dots, s_k$ are also non-increasing, then we may set $c_i = s_i$  and  $(\psi_i,\xi_i) = \sqrt{2}\zeta_i$ for $1\leq i\leq \rank_{\mathrm{max}}$ by application of \cref{prop:compactness_operator_A}, thus recovering an expression of the form in 
\cref{eq:minimizers_proof} for the minimizer $W$. 

The case when $k>\dim(\mathcal{H})$ is treated similarly with the difference that the SVD of $W$ contains at most $\dim(\mathcal{H})$ non-zero singular values, so that $k$ is replaced by $\dim(\mathcal{H})$ in \cref{eq:max_l_objective} and the same conclusion follows.

\end{proof}

\subsubsection{The gradient of \method's objective matches that of  the auxiliary loss $\mathcal{L}_{\mathrm{aux}}$}
\label[appendixsection]{sec:proof_aux_obj}
\begin{proof}[Proof of \cref{prop:gradient_expression_appendix}]
Recall the auxiliary loss defined in the proposition, which we can express in terms of the operator $A$:
\begin{equation}\label{eq:auxiliary-loss_appendix}
\begin{aligned}
\mathcal L_{\mathrm{aux}}(U,V;P,Q)
=&
\frac{1}{2} \Tr\parens{\parens{QP+\lambda I}\mathbb{E}\brackets{U(x)U(x)^{\top} + V(y)V(y)^{\top} }} \\
 &- \frac{1}{2}\Tr\parens{Q\mathbb{E}\brackets{U(x)V(y)^{\top} + V(y)U(x)^{\top}}}\\
=&\frac{1}{2} \Tr\parens{\parens{QP+\lambda I}WW^{\top} - QWAW^{\top}},
\end{aligned}
\end{equation}
where we replaced expectations by $WW^{\top}$ and $WAW^{\top}$. For $P$ and $Q$ such that $Q$ and $QP$ are symmetric, we get by direct differentiation w.r.t. $W$ that:
\begin{align}
	\langle \partial_{W}\mathcal L_{\mathrm{aux}}(W;P,Q), \delta W\rangle_{\mathcal{W}} =- 
	Tr\parens{\delta W\parens{QWA -(QP +\lambda I) W}^{\top} }.
\end{align}
In particular, the above expression holds when $P=P_W$ and $Q=Q_W$, since  $Q_{W}$ and $Q_WP_W$ are symmetric, by definition. Hence, we have shown that $\partial_{W}\mathcal{L}_{\mathrm{aux}}(W,P_W,Q_W)= -Q_{W}WA +(Q_WP_W+\lambda I) W$. 
Using the property $\parens{WW^{\top}+\lambda I}^{-1}W = W\parens{W^{\top}W+\lambda I}^{-1}$ and recalling that $A$ is bounded directly yields:
\begin{equation}
\begin{aligned}
	\partial_{W}\mathcal{L}_{\mathrm{aux}}(W,P_W,Q_W)
	&= -Q_{W}WA +(Q_WP_W+\lambda I) W\\
	&= \lambda W\parens{I- \parens{W^{\top}W+\lambda I}^{-1}A\parens{W^{\top}W+\lambda I}^{-1}}.
\end{aligned}
\end{equation}
The above expression matches that of $\nabla\mathcal{E}_{\method}(W)$ from  \cref{prop:Lipschitz_gradient}. Thus, we have shown that $\partial_{W}\mathcal{L}_{\mathrm{aux}}(W,P_W,Q_W) = \nabla\mathcal{E}_{\method}(W)$. Finally, by the same proposition we know that the objective $\mathcal{E}_{\method}$ is smooth, has a Lipschitz gradient with Lipschitz constant $\frac{4}{\lambda}$. 
\end{proof}

\subsection{Self-distillation dynamics}
\label[appendixsection]{sec:self_distillation_proof}

Here we provide a proof of \cref{prop:convergence_ssl} following a similar scheme as for \method. 
We will first express the vector field driving the self-distillation dynamics in terms of the operator $A$, then proceed to prove the main result. 

\subsubsection{Convergence of the self-distillation dynamics and characterization of equilibria}\label[appendixsection]{proof:stability_ssl}
Using the vector field expression in \cref{eq:expression_SSL_vector_field}, we will be able to prove \cref{prop:convergence_ssl} by reducing the problem to the components of the dynamics in the eigenbasis of the operator $A$.

\begin{proof}[Proof of \cref{prop:convergence_ssl}]
Denote by $\ucoor = (\ucoor_i)_{i\in I}$ the components of an element $W\in \mathcal{W}$ in the eigenbasis $(\zeta_i)_{i\in I}$ of $A$ (see \cref{eq:operators}). 
By \cref{eq:expression_SSL_vector_field}, we have the following equivalence between the dynamics $(W_{t})_{t\geq 0}$ in $\mathcal{W}$ driven by the vector field $\mathcal{F}$ and an ODE defined on its components  $(\ucoor_t)_{t\geq 0}$ in the eigenbasis of $A$ driven by a vector field $F^{1}$ defined in \cref{eq:vector_field_kappa_1_proof}:
\begin{align}
	\dot{W}_t=-\mathcal{F}(W_t) \iff \dot{\ucoor}_t = -\lambda F^{1}(\ucoor_t).
\end{align}
Thus it suffices to study the dynamics of the components, which is precisely the analysis in \cref{sec:hilbert_schmidt_analysis}.

Specifically, we directly apply \cref{prop:convergence_critial_point_F_kappa} to the dynamics driven by the vector field $F^{1}$, which ensures convergence to a critical point $\ucoor^{\star}$.

Then, by \cref{eq:characterization_fixed_point_general}, we know that such critical point is necessarily of the form:
\begin{align}\label{eq:characterization_critical_points_kappa_1}
	\ucoor^{\star} = \sum_{i=1}^{\rank}  f_{\epsilon_i}(s_{\alpha(i)},\lambda) q_i\otimes e_{\alpha(i)},\qquad f_{\epsilon}(s,\lambda) = \frac{1}{2}\parens{ \verts{s} + \epsilon \sqrt{s^2 -4\lambda} }\mathds{1}_{\verts{s}^2-4\lambda\geq 0}, 
	\end{align}
where $\rank$ is any non-negative integer such that $0 \leq \rank\leq k$, the vectors $(q_i)_{i=1}^{\rank}$ form an orthonormal family in $\mathbb{R}^k$, $(e_i)_{i\in I}$ is the canonical basis of $\ell^2(I)$, $(\alpha(i))_{i=1}^{\rank}$ are distinct indices associated to non-zero eigenvalues $(s_{\alpha(i)})_{i=1}^{\rank}$ and $\epsilon_{i}$ are numbers taking two possible values $-1$ or $1$. Hence, $\ucoor^{\star}$ are the components of the following critical point $W^{\star}$ of the vector field $\mathcal{F}$:
\begin{align}\label{eq:critical_point_proof_expression_kappa_1}
	W^{\star}(x,y) =  \sum_{i=1}^{\rank} q_i f_{\epsilon_i}(s_{\alpha(i)},\lambda) \zeta_{\alpha(i)}(x,y).
\end{align}
This is equivalent to the characterization in the statement of the result after noticing that $\verts{s_{\alpha(i)}} =c_{j(i)}$ for the corresponding canonical index $j(i)\in \{1,\dots, \numcca \}$, and that eigenvectors associated to positive eigenvalues are of the form $\zeta_{\alpha(i)} = \frac{1}{\sqrt{2}}(\psi_{j(i)},\xi_{j(i)})$, whereas those associated to negative eigenvalues are of the form $\frac{1}{\sqrt{2}}(\psi_{j(i)},-\xi_{j(i)})$.  
\end{proof}

\subsubsection{Stability of equilibria of the self-distillation dynamics}
\label[appendixsection]{sec:stability_ssl}
\begin{proof}[Proof of \cref{prop:stability_ssl}]
By  \cref{eq:expression_SSL_vector_field}, we have the following equivalence between the dynamics $(W_{t})_{t\geq 0}$ in $\mathcal{W}$ driven by the vector field $\mathcal{F}$ and an ODE defined on its components  $(\ucoor_t)_{t\geq 0}$ in the eigenbasis of $A$ driven by a vector field $F^{1}$ defined in \cref{eq:vector_field_kappa_1_proof}:
\begin{align}
	\dot{W}_t=-\mathcal{F}(W_t) \iff \dot{\ucoor}_t = -\lambda F^{1}(\ucoor_t).
\end{align}
In particular, a point $W^{\star}\in \mathcal{W}$ is stable for $\mathcal{F}$ iff its  coordinate representation $\ucoor^{\star}$ is stable for $F^{1}$. 

Stability of critical points of $F^{1}$ is studied in \cref{prop:stability_kappa_1}, which implies that, amongst all critical points of the vector field $F^{1}$, the only ones that are stable are of the form:
\begin{align}
	\ucoor^{\star} = \sum_{i=1}^{\rank^{+}}  f_{1}(s_{\alpha^{+}(i)},\lambda) q_i^{+}\otimes e_{\alpha^{+}(i)}  + \sum_{i=1}^{\rank^{-}}  f_{1}(s_{\alpha^{-}(i)},\lambda) q_i^{-}\otimes e_{\alpha^{-}(i)} 
\end{align}
where $\rank^{+}$ and $\rank^{-}$ are non-negative integers  in $[0,\numcca]$ satisfying  $0\leq \rank^++\rank^{-}\leq k$, $(q_i^{\sigma})_{\sigma\in \{+,-\}, i \in \{1,\dots,\rank^{\sigma}\}}$ are any orthonormal vectors in $\mathbb{R}^k$, and where $\alpha^{+}(1),\dots, \alpha^{+}(\rank^{+})$ are the indices of the top-$\rank^{+}$ eigenvalues of $A$ by decreasing order, whereas  $\alpha^{-}(1),\dots, \alpha^{-}(\rank^{-})$ are the indices of the bottom-$\rank^{-}$ eigenvalues of $A$ by increasing order. We get the desired result by expressing the element $W^{\star}$  based on its coordinates $\ucoor$. 

\end{proof}

\section{Proofs of the reduction to coordinate representations}
\label[appendixsection]{sec:proof_coordinate_decomposition}
\subsection{Proof of \cref{prop:Lipschitz_gradient}:  Expressions of the gradient of \method's objective}

\begin{proof}[Proof of \cref{prop:Lipschitz_gradient}]
Let $(U,V)$ be in $\mathcal{U}\times \mathcal{V}$. We start by expressing the objective in terms of the operator $A$, then its gradient.

\textbf{Expression of the objective.} 
Recall the expression of $\mathcal{E}_{\method}(U,V)$:
\begin{align}
	\mathcal{E}_{\method}(U,V) = -\frac{1}{2}\Tr\!\bigl(P_{U,V}\bigr)
+
\frac{\lambda}{2}\,
\mathbb{E}\!\left[
\|U(x)\|_2^2+\|V(y)\|_2^2
\right].
\end{align}
A first step is to express $P_{U,V}$ in closed form in terms of $W$ and operator $A$. 
As the solution of the linear regression problem in \cref{eq:csr-optimal-predictor}, $P_{U,V}$ admits the following expression:
\begin{equation}
\begin{aligned}
	P_{U,V} 
	&= \parens{\mathbb{E}\brackets{ U(x)V(y)^{\top} + V(y)U(x)^{\top}   }} \parens{ \mathbb{E}\brackets{U(x)U(x)^{\top} + V(y)V(y)^{\top}}  +\lambda I }^{-1}\\
	&= WAW^{\top}\parens{WW^{\top}+\lambda I}^{-1}.  
\end{aligned}
\end{equation}
Therefore, $\mathcal{E}_{\method}(U,V)$ can be written as:
\begin{align}
	\mathcal{E}_{\method}(U,V) = \mathcal{E}_{\method}(W) = -\frac{1}{2}\Tr\parens{WAW^{\top}\parens{WW^{\top}+\lambda I}^{-1}} + \frac{\lambda}{2}\Tr\parens{WW^{\top}}.
\end{align}
Hence, $\mathcal{E}_{\method}$ is a rational function of $W$. To prove smoothness, recall that by definition,  $A$ is a bounded operator. Hence, $WAW^{\top}$ is a smooth function of $W$, i.e. infinitely differentiable. 
Moreover, since $\lambda>0$, then $WW^{\top}+\lambda I$ is smooth and invertible for any $W$, therefore $W\mapsto Q_{W} = (WW^{\top}+\lambda I)^{-1}$ is smooth by composition with the inverse map. Finally,   $\mathcal{E}_{\method}$ is also smooth by composition of smooth functions.

\textbf{Gradient expression.}  
The expression of the gradient is obtained by direct differentiation along any direction $\delta W\in \mathcal{W}$
\begin{equation}
\begin{aligned}
	\langle\nabla \mathcal{E}_{\method}(W),\delta W\rangle_{\mathcal{W}} 
	=& -\frac{1}{2}Tr\parens{\parens{\delta WAW^{\top}+WA\delta W^{\top}}Q_W}+\lambda Tr(\delta WW^{\top})\\ 
	& + \frac{1}{2}Tr\parens{WAW^{\top}Q_W\parens{\delta WW^{\top}+ W\delta W^{\top}}Q_W}\\
	=& -Tr\parens{\delta WAW^{\top}Q_W}+\lambda Tr(\delta WW^{\top})\\ 
	& + Tr\parens{\delta W W^{\top}Q_WWAW^{\top}Q_W}\\
	=& -Tr\parens{\delta W\parens{AW^{\top}Q_{W} - W^{\top}Q_WP_W -\lambda W^{\top}} }\\
	=& -Tr\parens{\delta W\parens{Q_{W}WA -(Q_WP_W +\lambda I) W}^{\top} }.
\end{aligned}
\end{equation}
This shows that the gradient is given by: 
\begin{align}\label{eq:direct_gradient_expression}
	\nabla \mathcal{E}_{\method}(W) = -Q_{W}WA +Q_WP_WW  +\lambda W.
\end{align}

To get the second expression, recall that $P_W  = WAW^{\top}(WW^{\top}+\lambda I)^{-1} = WAW^{\top}Q_W$. Hence,
\begin{equation} 
\begin{aligned}
	\nabla \mathcal{E}_{\method}(W) 
	&= -Q_{W}WA +Q_WP_WW  +\lambda W\\
	&= Q_WWA\parens{W^{\top}(WW^{\top}+\lambda I)^{-1}W-I}+ \lambda W\\
	&= Q_WWA\parens{(W^{\top}W+\lambda I)^{-1}W^{\top}W - I}+ \lambda W\\
	&= -\lambda Q_WWA(W^{\top}W+\lambda I)^{-1}+ \lambda W\\
	&= -\lambda W \parens{W^{\top}W+\lambda I}^{-1}A\parens{W^{\top}W+\lambda I}^{-1}+ \lambda W\\
	&= \lambda W\parens{I- \parens{W^{\top}W+\lambda I}^{-1}A\parens{W^{\top}W+\lambda I}^{-1}},
\end{aligned}
\end{equation}
where we used the following property $(WW^{\top}+\lambda I)^{-1}W = W(W^{\top}W+\lambda I)^{-1}$ in the third and fifth lines.

To show that the gradient is Lipschitz, we will show that its Hessian is bounded over $\mathcal{W}$. Let us first express the action of the Hessian on a direction $\delta W$ by differentiating the second expression of the gradient along such direction. For conciseness we define $R = (W^{\top}W+\lambda I)^{-1}$ and get:
\begin{equation}
\begin{aligned}
	\nabla^{2}\mathcal{E}_{\method}(W)[\delta W] =&  \lambda \delta W \parens{I-RAR}+\lambda WR\parens{\delta W^{\top}W + W^{\top}\delta W}RAR\\
	 &+ \lambda WRAR  \parens{\delta W^{\top}W + W^{\top}\delta W}R.
\end{aligned}
\end{equation}
Let us now find an upper bound $C$ on the operator norm of the Hessian. To this end it suffices to find an upper bound $\Verts{\nabla^{2}\mathcal{E}_{\method}(W)[\delta W]}$ of the form $\Verts{\nabla^{2}\mathcal{E}_{\method}(W)[\delta W]}\leq C\Verts{\delta W}$.  Standard operator norm inequalities yield:
\begin{equation}
\begin{aligned}
	\Verts{\nabla^{2}\mathcal{E}_{\method}(W)[\delta W]}
	\leq & \lambda\parens{\Verts{RAR}_{op}+1}\Verts{\delta W}\\
	&+ \lambda  \Verts{WR}_{op}\Verts{WRAR}_{op}\Verts{\delta W} +
		\lambda \Verts{WRW^{\top}}_{op}\Verts{RAR}_{op}\Verts{\delta W}\\
	&+ \lambda \Verts{WRAR}_{op}\Verts{WR}_{op}\Verts{\delta W} + \lambda \Verts{WRARW^{\top}}_{op}\Verts{R}_{op}\Verts{\delta W}	.
\end{aligned}
\end{equation}
Define $C_0$ as:
\begin{equation}
\begin{aligned}
	C_0 =&  \lambda\parens{1+\Verts{RAR}_{op}  +  2\Verts{WR}_{op}\Verts{WRAR}_{op}}\\
	& + \lambda\parens{\Verts{WRW^{\top}}_{op}\Verts{RAR}_{op} +\Verts{WRARW^{\top}}_{op}\Verts{R}_{op}}
\end{aligned}
\end{equation}
We have shown, so far that $\Verts{\nabla^2\mathcal{E}_{\method}(W)}_{op}\leq C_0$. It remains to show that $C_0$ is uniformly bounded w.r.t. $W$. To this end, let us further upper-bound $C_0$:
\begin{align}
	C_0 \leq & \lambda\parens{1+ \Verts{A}_{op}\parens{\Verts{R}_{op}^2 + 3\Verts{WR}_{op}^2\Verts{R}_{op} + \Verts{R}_{op}^2 \Verts{WRW^{\top}}_{op}  }  } 
\end{align}
Recalling that $R = (W^{\top}W+\lambda I)^{-1}$, we directly have that $\Verts{R}_{op}\leq \lambda^{-1}$. Moreover, since $W$ has finite dimensional range (of dimension $k$), it must be a compact operator, so that it admits a singular-value decomposition, by the spectral theorem for compact operators \citep{Gohberg:2003}. Using such a decomposition, we can show that the singular values of $WR$ and $WRW^{\top}$ are given by $\frac{w_i}{w_i^2+\lambda}$ and $\frac{w_i^2}{w_i^2+\lambda}$ where $w_i$ are the singular values of $W$. These quantities are uniformly bounded by $\frac{1}{2\sqrt{\lambda}}$ and $1$ respectively. Consequently, we have $\Verts{WR}_{op}\leq \frac{1}{2\sqrt{\lambda}}$ and $\Verts{WRW^{\top}}_{op}\leq 1$. Finally, recalling that $\Verts{A}_{op}\leq 1$ by \cref{eq:bounded_op_norm_A} it follows that:
\begin{align}
	C_0\leq \lambda\parens{1 + \frac{11}{4\lambda^2}}\leq \frac{4}{\lambda},
\end{align}
where we used that $\lambda<1$. This allows us to conclude that the gradient $\nabla \mathcal{E}_{\method}$ is Lipschitz, with Lipschitz constant $\frac{4}{\lambda}$.

\textbf{Components of the gradient in the eigenbasis of $A$.}
Using the decomposition in \cref{eq:decomposition_in_basis} of $WAW^{\top}$ and $WW^{\top}$ in such basis and recalling the expression of the gradient $\nabla\mathcal{E}_{\method}(W)$ above, we directly get the desired expression for its components $\mathbf{g}$ in the basis $(\zeta_i)_{i\in I}$ in terms of $\ucoor$ and  the diagonal operator $S$ which is given by the vector field $F^{0}$. Equivalence between the ODE in $\mathcal{W}$ and the ODE of its coordinates in $\mathbb{W}$ follows directly. 
\end{proof}

\subsection{Proof of \cref{eq:expression_SSL_vector_field}:  Expressions of the self-distillation vector field}\label{sec:appendix:proof-ssl-vf}

\begin{proof}[Proof of \cref{eq:expression_SSL_vector_field}]
Let us first express the partial derivatives of  $\mathcal{P}_{U\triangleright V}(P, U, V)$  along some direction $\delta U$:
\begin{equation}
\begin{aligned}
	\langle \partial_U \mathcal{P}_{U\triangleright V}(P, U, V), \delta U\rangle_{\mathcal{U}} =&
	\mathbb{E}\brackets{ (P\delta U(x))^{\top}(PU(x)-V(y))} + \lambda \mathbb{E}\brackets{\delta U(x)^{\top} U(x)}\\
		=& \Tr\parens{P^{\top}P\mathbb{E}\brackets{U(x)(\delta U(x))^{\top}}- \mathbb{E}\brackets{V(y)(\delta U(x))^{\top}} P^{\top}}  \\
		&+ \lambda \Tr\parens{\mathbb{E}\brackets{U(x)(\delta U(x))^{\top}}}
\end{aligned}
\end{equation}
A similar expression can be  obtained for $\mathcal{P}_{V\triangleright U}(P, V, U)$ along some perturbation $\delta V$,  by exchanging the roles of $U$ and $V$. Hence, denoting by $W = (U,V)$ and $\delta W = (\delta U, \delta V)$  and summing both expressions above yields:
\begin{equation}
\begin{aligned}
	\langle \partial_U \mathcal{P}_{U\triangleright V}(P, U, V), \delta U\rangle_{\mathcal{U}}  +& \langle \partial_V \mathcal{P}_{V\triangleright U}(P, V, U), \delta V\rangle_{\mathcal{V}} \\= &\Tr\parens{P^{\top}P\mathbb{E}\brackets{U(x)(\delta U(x))^{\top} + V(y)(\delta V(y))^{\top}}}\\
	&- Tr\parens{\mathbb{E}\brackets{V(y)(\delta U(x))^{\top} + U(x)(\delta V(y))^{\top}} P^{\top}}  \\
		&+ \lambda \Tr\parens{\mathbb{E}\brackets{U(x)(\delta U(x))^{\top} +V(y)(\delta V(y))^{\top} } }\\
		= &Tr\parens{P^{\top}P W(\delta W)^{\top} } - Tr\parens{WA(\delta W)^{\top}P^{\top}}\\
		&+ \lambda Tr\parens{W(\delta W)^{\top}},
\end{aligned}
\end{equation}
where we viewed $W$ as a Hilbert-Schmidt operator from $\mathcal{H}$ to $\mathbb{R}^k$. 
Therefore, the vector field  $(\partial_U \mathcal{P}_{U\triangleright V}(P, U, V), \partial_V \mathcal{P}_{V\triangleright U}(P, V, U))$ is given by:
\begin{align}
	(\partial_U \mathcal{P}_{U\triangleright V}(P, U, V), \partial_V \mathcal{P}_{V\triangleright U}(P, V, U)) = P^{\top}PW-P^{\top}WA + \lambda W.
\end{align}
Thus we only need to evaluate the above expression at optimal regressor $P_{U,V}$ to get an expression for the vector field $\mathcal{F}(W)$ driving the self-distillation dynamics. To this end, note that $P_{U,V}$ can be expressed in closed form as follows:
\begin{equation}
\begin{aligned}
	P_{U,V} 
	&= \parens{\mathbb{E}\brackets{ U(x)V(y)^{\top} + V(y)U(x)^{\top}   }} \parens{ \mathbb{E}\brackets{U(x)U(x)^{\top} + V(y)V(y)^{\top}}  +\lambda I }^{-1}\\
	&= WAW^{\top}\parens{WW^{\top}+\lambda I}^{-1}.  
\end{aligned}
\end{equation}

Consequently, using the above expression along with the one on partial derivatives $(\partial_U \mathcal{P}_{U\triangleright V}, \partial_V \mathcal{P}_{V\triangleright U})$, yields an expression for the vector field  $\mathcal{F}(W)$:
\begin{equation}
\begin{aligned}
	\mathcal{F}(W) 
	&= \parens{P_{U,V}}^{\top}WA\parens{ W^{\top}\parens{WW^{\top}+\lambda I}^{-1}W - I } + \lambda W\\
	&= \parens{P_{U,V}}^{\top}WA\parens{ \parens{W^{\top}W+\lambda I}^{-1}W^{\top}W - I } + \lambda W\\
	&= -\lambda \parens{P_{U,V}}^{\top}WA \parens{W^{\top}W+\lambda I}^{-1} + \lambda W\\
	&= -\lambda \parens{WW^{\top}+\lambda I}^{-1} WA^{\top}W^{\top} W A\parens{W^{\top}W+\lambda I}^{-1} + \lambda W\\
	&= -\lambda W \parens{W^{\top}W+\lambda I}^{-1} AW^{\top}W A\parens{W^{\top}W+\lambda I}^{-1} + \lambda W\\
	&= \lambda W\parens{I- \parens{W^{\top}W+\lambda I}^{-1} AW^{\top} W A\parens{W^{\top}W+\lambda I}^{-1}}.
\end{aligned}
\end{equation}
 In the above, we used that $A$ is symmetric, i.e.  $A^{\top}=A$ in the fifth line. Moreover, we used the property that $W^{\top}\parens{WW^{\top}+\lambda I}^{-1} = \parens{W^{\top}W+\lambda I}^{-1}W^{\top}$ in the second line and that $\parens{WW^{\top}+\lambda I}^{-1}W = W\parens{W^{\top}W+\lambda I}^{-1}$ in the fifth line.
 
 \textbf{Components of the vector field in the eigenbasis of $A$.}
Using the decomposition in \cref{eq:decomposition_in_basis} of $WAW^{\top}$ and $WW^{\top}$ in such basis and recalling the expression of the vector field $\mathcal{F}(W)$ above, we directly get the desired expression for its components $\mathbf{g}$ in the basis $(\zeta_i)_{i\in I}$ in terms of $\ucoor$ and  the diagonal operator $S$ which is given by the vector field $F^{1}$. Equivalence between the ODE in $\mathcal{W}$ and the ODE of its coordinates in $\mathbb{W}$ follows directly. 
\end{proof}

\section{Proof of convergence to critical points}\label[appendixsection]{sec:convergence_gradient_like_systems}

\subsection{Proof of convergence to critical points for general gradient-like dynamical systems}

For convenience, let us recall here the statement of the general convergence result before proceeding with the proof. 
\begin{proposition}\label{prop:general_convergence_critical_point}
Let $F$ be a smooth vector field defined on a separable Hilbert space $\mathbb{W}$ and consider the following dynamical system:
\begin{align}\label{eq:dynamical_system}
	\dot{\ucoor}_t = - F(\ucoor_t),\qquad \ucoor_0\in \mathbb{W}.
\end{align}
Assume there exists a real analytic function $L$ defined over $\mathbb{W}$ such that the following holds:
\begin{assumplist}[start=2]
\item\label{assump_dissipation}\textbf{Dissipation.} $L$ is a Lyapunov function for the system, i.e. $\dot{L}(\ucoor_t)\leq 0$ for all $t\geq 0$.  
\item\label{assump_sufficient_decrease} \textbf{Sufficient decrease.} There exists $c_{max}>0$ such that $\Verts{F(\ucoor_t)}^2\leq - c_{max} \dot{L}(\ucoor_t)$ for  $t\geq 0$.
\item\label{assump_error_bound} \textbf{Relative  bound.} There exists  a non-increasing real-valued function $\Omega: [0,+\infty)\rightarrow (0,+\infty)$  such that $\Omega(\Verts{\ucoor_t})\Verts{\nabla L(\ucoor_t)}^2 \leq \Verts{F(\ucoor_t)}^2$.
\item\label{assump_coercive}
\textbf{Coercivity.} $L$ is coercive, i.e. $L(\ucoor)\rightarrow +\infty$ when $\Verts{\ucoor}\rightarrow +\infty$.
\item\label{assump_lower_bound} \textbf{Lower bound.} $L$ admits a finite lower bound.
\item\label{assump_bounded_on_bounded_set} \textbf{Radial growth.} The vector field $F$, gradient $\nabla L$ and Hessian $\nabla^2 L$ all have radial growth, i.e. there exists a non-decreasing real-valued function $\omega:[0,+\infty) \rightarrow [0,+\infty)$ such that $\max\parens{\Verts{F(\ucoor)},\Verts{\nabla L(\ucoor)}, \Verts{\nabla^2 L(\ucoor) }}\leq \omega(\Verts{\ucoor})$ for all $\ucoor\in \mathbb{W}$.
\item\label{assump_kl} 
\textbf{Fredholm property.} For any $\ucoor\in \mathbb{W}$, the hessian $\nabla^2 L(\ucoor)$ can be decomposed as the sum of a boundedly invertible operator $R$ and a bounded compact operator $K$.
\item\label{assump_palais_smale_condition} \textbf{Palais-Smale condition.} Any bounded sequence $(\ucoor_n)_{n\geq 1}$ of $\mathbb{W}$ such that $\Verts{\nabla L(\ucoor_n)}\rightarrow 0$ admits a convergent subsequence.
\end{assumplist}
Then the dynamical system in \cref{eq:dynamical_system} is defined at all times and converges to a critical point of the vector field $F$.
\end{proposition}

\begin{proof}[Proof of \cref{prop:general_convergence_critical_point}]
Let $\ucoor_t$ be a solution to the ODE and $L$ be a Lyapunov function, which exists by \cref{assump_dissipation}. 
   
\paragraph{Local existence.}
Since $F$ is a smooth vector field, it is therefore locally Lipschitz. Hence, by Cauchy-Lipschitz theorem~\citep[Theorem 1.9]{Teschl:2000}, the ODE \cref{eq:dynamical_system} admits a unique maximal solution starting from a point $\ucoor_0$, continuous in time and defined on a maximal interval $[0,T)$ with $T>0$ possibly equal to $+\infty$. We will later show that $T=+\infty$.

\paragraph{Bounded trajectories and velocities.}
We know by \cref{assump_dissipation} that $L$ is a Lyapunov function for the ODE \cref{eq:dynamical_system}. Therefore, $\mathbf{\ucoor}_t$ remains in the level set $\{\ucoor\in \mathbb{W}| L(\ucoor)\leq L(\ucoor_0)\}$. Moreover, since $L$ is coercive by \cref{assump_coercive}, such level set must be bounded so that $\ucoor_t$ is also bounded, i.e. there exists $b>0$ so that $\Verts{\ucoor_t}\leq b$ for all $0\leq t<T$. 
To show that $\dot{\ucoor}_t$ is bounded, recall that $F(\ucoor)$ is bounded over the centered ball of radius $b$ by \cref{assump_bounded_on_bounded_set}. In particular, $t\mapsto \Verts{F(\ucoor_t)}$ is also bounded, since we have just established that $\ucoor_t$ belongs to such a ball for all $0\leq t<T$. Consequently, $t\mapsto\dot{\ucoor}_t$ remains also bounded for all $0\leq t<T$, since, by definition $\dot{\ucoor}_t = -F(\ucoor_t)$.

\paragraph{Global existence.}
Let us now show that $T=+\infty$. By contradiction, assume that $T$ is finite. We have shown that $\dot{\ucoor}_t$ is bounded for all $0\leq t<T$. Hence, there exists a positive constant $M$ such that $\Verts{\dot{\ucoor}_t}\leq M$ for all $0\leq t<T$. Hence, for any $0\leq s<t<T$,
\begin{align}
\Verts{\ucoor_t-\ucoor_s}
\leq \int_s^t \Verts{\dot{\ucoor}_\tau}\,\diff \tau
\leq M(t-s).
\end{align}
Thus $(\ucoor_t)_{t<T}$ is Cauchy as $t\uparrow T$. Since $\mathbb W$ is complete, there exists $\ucoor_T\in\mathbb W$ such that $\ucoor_t\rightarrow \ucoor_T$ as $t\uparrow T$. 
Applying the local existence theorem again with initial condition $\ucoor_T$ extends the solution beyond time $T$, contradicting the maximality of $T$. Therefore, $T=+\infty$.

\paragraph{Convergence of the energy gradient $\nabla L(\ucoor_t)$ to $0$.} 
First, by \cref{assump_error_bound}, we know that $\Omega(\Verts{\ucoor_t})\Verts{\nabla L(\ucoor_t)}^2\leq \Verts{F(\ucoor_t)}^2$. Since $\Verts{\ucoor_t}$ is bounded by $b>0$ and $\Omega$ is non-increasing, it follows that $0<\Omega(b)\leq \Omega(\Verts{\ucoor_t})$. Denoting by $c_{min} = \Omega(b)$, we have shown that:
\begin{align}\label{eq:relative_error_general}
	c_{min}\Verts{\nabla L(\ucoor_t)}^2\leq \Verts{F(\ucoor_t)}^2.
\end{align}
Moreover, further using \cref{assump_sufficient_decrease}, it follows that $\Verts{\nabla L(\ucoor_t)}^2\leq -\frac{c_{max}}{c_{min}} \dot{L}(\ucoor_t)$. Therefore, by integrating such inequality between $0$ and any positive time $T$, it follows that:
\begin{align}\label{eq:integral_energy}
	\int_0^{T} \Verts{\nabla L(\ucoor_t)}^2\diff t\leq \frac{c_{max}}{c_{min}} \parens{L(\ucoor_0)-L(\ucoor_{T})}.
\end{align}
Now recall that, by \cref{assump_dissipation,assump_lower_bound}, $t\mapsto L(\ucoor_t)$ is non-increasing with $L$ being lower-bounded. Therefore, $L(\ucoor_t)$ must converge to a finite limit $L_{\infty}$ satisfying $L(\ucoor_t)\geq L_{\infty}$ for all $t\geq 0$. Using that $L(\ucoor_T)\geq L_{\infty}$ in \cref{eq:integral_energy} implies that $\int_0^{+\infty} \Verts{\nabla L(\ucoor_t)}^2\diff t$ is finite. Such an integrability condition alone is not enough to deduce that $\Verts{\nabla L(\ucoor_t)}^2$ converges to $0$. However, provided that $\Verts{\nabla L(\ucoor_t)}^2$ is uniformly continuous, we can deduce convergence to $0$ of $\Verts{\nabla L(\ucoor_t)}^2$ by application of Barbalat's theorem \citep{Barbalat:1959}. We obtain uniform continuity  by  \cref{lem:uniform_continuity} which applies here, as $\ucoor_t$ and $\dot{\ucoor}_t$ are bounded and $\ucoor\mapsto \nabla L(\ucoor)$ and $\ucoor\mapsto  \nabla^2 L(\ucoor)$ are radially bounded by \cref{assump_bounded_on_bounded_set}. 

\paragraph{Existence of cluster points.}
We wish to establish the existence of a cluster point for the dynamics, i.e., a point $\ucoor^{\star}$ such that there exists a subsequence along the trajectory $\ucoor_t$ converging towards it.
Let $(t_n)_{n\geq 1}$ be any increasing sequence of positive real numbers with $t_n\rightarrow +\infty$. We know that $(\ucoor_{t_n})_{n\geq 1}$ is bounded since we established that $t\mapsto \ucoor_t$ is bounded. By   \cref{assump_palais_smale_condition}, we deduce that $(\ucoor_{t_n})$ admits a convergent subsequence, thus establishing the existence of a cluster point as the limit point of such a subsequence. If $\ucoor^{\star}$ is a cluster point, then it must be a critical point of $\nabla L$. Indeed, we have already established that $\nabla L(\ucoor_t)\rightarrow 0$. Hence, by continuity of $\nabla L$, we directly get that $\nabla L(\ucoor^{\star})=0$.

\paragraph{A Lojasiewicz-Simon's inequality.} 
For any $\ucoor\in \mathbb{W}$, we know, by \cref{assump_kl},   that the Hessian $\nabla^2 L(\ucoor)$ can be decomposed as the sum of a boundedly invertible operator $R$ and a bounded compact operator $K$. 
The operator $R$ is necessarily a Fredholm operator of index $0$, i.e. $ker(R) = \{0\}$ and $ker(R^*)=\{0\}$ and $R$ has a closed range. Moreover, by stability of Fredholm operators to compact perturbations, e.g.~\citep[Theorem 3.11]{Conway:2019}, it follows that $\nabla^2 L(\ucoor) = R + K$ is also Fredholm with index $0$. Moreover, by assumption, $L$ is analytic. Therefore, by direct application of \citep[Theorem 1]{Feehan:2020b}, it follows  that $L$ satisfies a desired Lojasiewicz–Simon gradient inequality near its critical points. In other words, if $\ucoor^{\star}$ is a critical point of $\nabla L$, then there exist finite constants $Z>0$, $0<\sigma < 1$ and $\frac{1}{2}\leq \theta < 1 $ so that for any $\ucoor$ satisfying $\Verts{\ucoor-\ucoor^{\star}}\leq \sigma$, it must hold that:
\begin{align}\label{eq:Lojasiewicz-Simon_ssl}
	\Verts{\nabla L(\ucoor)}\geq Z \verts{L(\ucoor)-L(\ucoor^{\star})}^{\theta}.
\end{align}

\paragraph{Trapping near a cluster point.}
Let $\ucoor^{\star}$ be an accumulation point of $\ucoor_t$, which was shown to exist. We have also shown  that $L(\ucoor_t)$ converges to a finite value $L_{\infty}$ with $L(\ucoor_t)\geq L_{\infty}$. Here we are only interested in the case where $L(\ucoor_t)> L_{\infty}$ at all times. The other will be treated separately. To show that the trajectory remains close to the same cluster points, we will use a trapping argument near $\ucoor^{\star}$ following the scheme from \citep{Bolte:2014} for Lyapunov functions satisfying a Lojasiewicz-Simon inequality. 
Fix $\epsilon>0$. Since $\ucoor^{\star}$ is an accumulation point of $\ucoor_t$, and that $L(\ucoor_t)\rightarrow L_{\infty}>-\infty$,  there exists $T>0$ so that $\Verts{\ucoor_T-\ucoor^{\star}} \leq \epsilon $ and $0\leq L(\ucoor_t) - L_{\infty}\leq \epsilon$ for all $t\geq T$. By choosing $\epsilon< \sigma$, we can guarantee that \cref{eq:Lojasiewicz-Simon_ssl} holds. Let us show that  $\Verts{\ucoor_t-\ucoor^{\star}}\leq \sigma$ for any $t\geq T$. Consider $T_{max}$, the largest time greater than $T$ for which $\Verts{\ucoor_t-\ucoor^{\star}}\leq \sigma$ for all $t\in (T,T_{max})$. We need to show that $T_{max} = \infty$. It is clear by continuity of $\ucoor_t$ that $T_{max}>T$ so that the interval $(T, T_{max})$ is non-empty. Recalling that $\dot{\ucoor}_t= -F(\ucoor_t)$, the following holds, for all $t\in (T, T_{max})$:
\begin{equation}
\begin{aligned}
	\Verts{\dot{\ucoor}_t}= \Verts{F(\ucoor_t)}&\leq \sqrt{c_{min}^{-1}}\frac{\Verts{F(\ucoor_t)}^2}{\Verts{\nabla L(\ucoor_t)}}\\
	&\leq -\sqrt{c_{min}^{-1}} c_{max}\frac{\dot{L}(\ucoor_t)}{\Verts{\nabla L(\ucoor_t)}}\\
	&\leq -Z^{-1}\sqrt{c_{min}^{-1}} c_{max}\frac{\dot{L}(\ucoor_t)}{\verts{L(\ucoor_t)-L(\ucoor^{\star})}^{\theta}},\\
	&= -Z^{-1}\sqrt{c_{min}^{-1}} c_{max} \frac{\dot{L}(\ucoor_t)}{\parens{L(\ucoor_t)-L_{\infty}}^{\theta}},
\end{aligned}
\end{equation}
where we used \cref{eq:relative_error_general} for the first line, then \cref{assump_sufficient_decrease} for the second line. The third inequality, uses  \cref{eq:Lojasiewicz-Simon_ssl}, whereas the last one uses that $L(\ucoor_t)\geq L_{\infty}= L(\ucoor^{\star})$, where the equality $L_{\infty}= L(\ucoor^{\star})$ follows by continuity of $L$ and uniqueness of the limit $L_{\infty}$. As we are in the case where $L(\ucoor_t)-L_{\infty}>0$, we must have $\Verts{\nabla L(\ucoor_t)}>0$ thanks to \cref{eq:Lojasiewicz-Simon_ssl}, so that the above upper-bounds are always finite. Integrating the above inequality over an interval $(t,s)$ with $T\leq t\leq s < T_{max}$ yields:
\begin{equation}\label{eq:bound_lenght}
\begin{aligned}
	\int_t^{s} \Verts{\dot{\ucoor}_{\tau}}\diff \tau 
	&\leq \alpha \parens{ \parens{L(\ucoor_t)- L_{\infty}}^{1-\theta} -  \parens{L(\ucoor_{s})- L_{\infty}}^{1-\theta}}\\
	&\leq \alpha \parens{L(\ucoor_{t}) - L_{\infty} }^{1-\theta}
\end{aligned}
\end{equation}
where $\alpha = \frac{1}{1-\theta}Z^{-1}\sqrt{c_{min}^{-1}} c_{max} < \infty$. 
Consequently, the above inequality implies, in particular, that:
\begin{equation}
\begin{aligned}
	\Verts{\ucoor_t- \ucoor^{\star}} 
	&\leq \Verts{\ucoor_T- \ucoor^{\star}} + \int_T^{t} \Verts{\dot{\ucoor_{\tau}}}\diff \tau\\
	&\leq \Verts{\ucoor_T- \ucoor^{\star}} + \alpha\parens{L(\ucoor_{T}) - L_{\infty} }^{1-\theta}\leq \epsilon + \alpha\epsilon^{1-\theta}. 
\end{aligned}
\end{equation}
We can choose $\epsilon$ small enough so that $\epsilon + \alpha \epsilon^{1-\theta} \leq \frac{\sigma}{2}$, thus yielding $\Verts{\ucoor_t- \ucoor^{\star}}\leq \frac{\sigma}{2}$ for all $T\leq t<T_{max}$. If, by contradiction,  $T_{max}$ were finite, then, by continuity of $\ucoor_t$ we would also have $\Verts{\ucoor_{T_{max}}- \ucoor^{\star}}\leq \frac{\sigma}{2}$ contradicting the definition of $T_{max}$. Therefore $T_{max}= +\infty$ and $\Verts{\ucoor_t-\ucoor^{\star}}\leq \frac{\sigma}{2}$ for all $t\geq T$.

\paragraph{Convergence to an accumulation point.} 
To prove convergence towards $\ucoor^{\star}$, we consider two separate cases depending on whether $L(\ucoor_t)$ reaches its lower bound $L_{\infty}$ or not.

\begin{itemize}
	\item \textbf{Case 1}: There exists $t_0\geq 0$ so that $L(\ucoor_{t_0})= L_{\infty}$. 
\end{itemize}
Let $t\geq t_0$. Since $L(\ucoor_t)$ is non-increasing by \cref{assump_dissipation}, it holds that $L_{\infty} = L(\ucoor_{t_0})\geq L(\ucoor_t)$.  On the other hand, we have $L(\ucoor_t)\geq L_{\infty}$ so that $L(\ucoor_t)=L_{\infty}$. As this equality holds for any $t\geq t_0$, it must be that $\dot{L}(\ucoor_t)=0$ for all $t\geq t_0$. Along with  \cref{assump_sufficient_decrease}, we deduce that $\Verts{F(\ucoor_t)}^2\leq -c_{max} \dot{L}(\ucoor_t)=0$ for all $t\geq t_0$.  Consequently, $\ucoor_t$ remains constant, but since we know that $\ucoor^{\star}$ is an accumulation point of the trajectory, it must hold that $\ucoor_t = \ucoor^{\star}$ for all $t\geq t_0$, which, a fortiori, implies convergence of $\ucoor_t$ towards $\ucoor^{\star}$.

\begin{itemize}
	\item \textbf{Case 2}: $L(\ucoor_{t})> L_{\infty}$ for all $t\geq 0$. 
\end{itemize}
We have already established that $\Verts{\ucoor_t-\ucoor^{\star}}\leq \sigma$ for all $t\geq T$, so that \cref{eq:Lojasiewicz-Simon_ssl} always holds and we can use \cref{eq:bound_lenght}. 
Let $t\geq T$ and consider a sequence $\ucoor_{s_n}$ converging towards $\ucoor^{\star}$ with $s_{n}$ increasing and greater than $t$. Such sequence  always exists as $\ucoor^{\star}$ is an accumulation point of $\ucoor_{t}$. Recalling that $\ucoor_{s_n}- \ucoor_{t} = \int_t^{s_n}\dot{\ucoor}_{t'}\diff t'$  and using \cref{eq:bound_lenght}, it follows that for any $n\geq 1$:
\begin{align}
	\Verts{\ucoor_{s_n}- \ucoor_{t}}\leq \int_t^{s_n} \Verts{\dot{\ucoor}_{t'}}\diff t'\leq \alpha\parens{L(\ucoor_{t}) - L_{\infty} }^{1-\theta}.
\end{align}
Taking the limit $n\rightarrow +\infty$ yields: $\Verts{\ucoor^{\star}- \ucoor_t}\leq \alpha\parens{L(\ucoor_{t}) - L_{\infty} }^{1-\theta}$, which holds for any $t\geq T$. Furthermore, since $L(\ucoor_{t}) - L_{\infty} \rightarrow 0$, $\alpha>0$ and $1-\theta>0$, we deduce that $\ucoor_t\rightarrow \ucoor^{\star}$. By continuity of $\nabla L$ and $F$, we also have that $\nabla L(\ucoor_t)\rightarrow \nabla L(\ucoor^{\star})$ and $F(\ucoor_t)\rightarrow F(\ucoor^{\star})$. Recall also that $\ucoor^{\star}$  is a critical point of $\nabla L$, i.e. $\nabla L(\ucoor^{\star}) = 0$. Finally, by \cref{assump_sufficient_decrease}, we get 
\begin{align}
	\Verts{F(\ucoor_t)}^2\leq -c_{max} \dot{L}(\ucoor_t) = c_{max}\langle \nabla L(\ucoor_t),F(\ucoor_t)\rangle, 
\end{align}
 taking the limit to infinity yields $\Verts{F(\ucoor^{\star})}^2\leq 0$. Thus $\ucoor^{\star}$ is a critical point of $F$. 
\end{proof}

\begin{lemma}
\label{lem:uniform_continuity}
	Consider a differentiable map $t\mapsto \ucoor_t$ with values in a separable Hilbert space $\mathbb{W}$ with $\ucoor_t$ and $\dot{\ucoor}_t$ bounded. Let $G$ be a smooth map on $\mathbb{W}$ so that $G$ and $\nabla G$ have radial growth, i.e. there exists a non-decreasing function $\omega: [0,+\infty)\rightarrow [0,+\infty)$ so that $\Verts{G(\ucoor)},\Verts{\nabla G(\ucoor)}\leq \omega(\Verts{\ucoor})$ for all $\ucoor\in \mathbb{W}$. Then $t\mapsto \Verts{G(\ucoor_t)}^2$ is uniformly continuous.
\end{lemma}
\begin{proof}
We know that $\ucoor_t$ is bounded, hence, there exists a positive real number $b$ so that $\ucoor_t$ is included in a centered ball of radius $b$ for all $t$. As $G$ and $\nabla G$ have radial growth, by assumption, they must be bounded on such a ball. Therefore, one can find a positive constant $L$ so that the following holds:
\begin{align}
	\Verts{G(\ucoor_t) - G(\ucoor_s)}\leq \sup_{\substack{\ucoor\in \mathbb{W}\\
	\Verts{\ucoor}\leq b}} \Verts{\nabla G(\ucoor)} \Verts{\ucoor_t-\ucoor_s}\leq L \sup_{t'\geq 0}\Verts{\dot{\ucoor}_{t'}}\verts{t-s},\qquad \Verts{G(\ucoor_t)}\leq L
\end{align}
Since, by assumption, $\sup_{t'\geq 0}\Verts{\dot{\ucoor}_{t'}} <+\infty$, we deduce that  $t\mapsto G(\ucoor_t)$ is Lipschitz and bounded. We can therefore deduce that $t\mapsto \Verts{G(\ucoor_t)}^2$ is Lipschitz thanks to the inequalities:
\begin{equation}
\begin{aligned}
\verts{\Verts{G(\ucoor_t)}^2 - \Verts{G(\ucoor_s)}^2} 
	&\leq \Verts{G(\ucoor_t)-G(\ucoor_s)}\parens{\Verts{G(\ucoor_t)} + \Verts{G(\ucoor_s)}}\\
	&\leq 2L^2\sup_{t'\geq 0}\Verts{\dot{\ucoor}_{t'}}\verts{t-s}.
\end{aligned}
\end{equation}
A fortiori, $\Verts{G(\ucoor_t)}^2$ is uniformly continuous. 
\end{proof}

\subsection{Proof of the auxiliary results of  \cref{prop:lyapunov_function,lem:growth_F_kappa,lem:energy_bounded_below,lem:kurdyka_simon_inequality,lem:palais_smale}}\label[appendixsection]{sec:proof_smale_etc}

\paragraph{Existence of a Lyapunov function}
\begin{proof}[Proof of \cref{prop:lyapunov_function} (Existence of a Lyapunov function)]
Differentiating $L^{\kappa}$ w.r.t. $\ucoor$ yields:
\begin{align}
\nabla L^{\kappa}(\ucoor) = 2\ucoor\Phi(\ucoor^{\top}\ucoor),	\qquad \Phi(M) = (M+\lambda I)^{2} - S(MS)^{\kappa}.
\end{align}
Note the vector field $F^{\kappa}$ can be expressed in terms of $\Phi$ as follows:
\begin{align}
	F^{\kappa}(\ucoor) 
	=& \ucoor B(M) \Phi(M)B(M), \qquad M = \ucoor^{\top}\ucoor, \qquad B(M)= (M+\lambda I)^{-1}.
\end{align}
Introducing $M_t =\ucoor_t^{\top}\ucoor_t$ for conciseness and computing the time derivatives of $L^{\kappa}(\ucoor_t)$ yields: 
\begin{equation}
\begin{aligned}
	\dot{L}^{\kappa}(\ucoor_t) &= Tr\parens{\nabla L^{\kappa}(\ucoor_t)^{\top}\dot{\ucoor}_t}\\
	&=  -2\lambda Tr\parens{\Phi(M_t)M_tB(M_t)\Phi(M_t)B(M_t)}\\
	&= -2\lambda Tr\parens{\Phi(M_t)B(M_t)M_t\Phi(M_t)B(M_t)}\\
	&= -2\lambda Tr\parens{B(M_t)^{\frac{1}{2}}M_t\Phi(M_t)B(M_t)\Phi(M_t)B(M_t)^{\frac{1}{2}}}\\
	&= -2\lambda Tr\parens{M_tB(M_t)^{\frac{1}{2}}\Phi(M_t)B(M_t)\Phi(M_t)B(M_t)^{\frac{1}{2}}}\\
	&= -2\lambda Tr\parens{M_tQ_t^2} = -2\lambda \Verts{\ucoor_t Q_t}^2, 
\end{aligned}
\end{equation}
where we defined $Q_t  = B(M_t)^{\frac{1}{2}}\Phi(M_t)B(M_t)^{\frac{1}{2}}$ and used that $M_t$ always commutes with $B(M_t)$. 
\paragraph{Sufficient decrease.} 
First, recall that by definition of $F^{\kappa}$, we have that:
\begin{align}
\Verts{F^{\kappa}(\ucoor)}^2 
	&= Tr\parens{B \Phi BMB\Phi B},
\end{align}
where, we used that $M = \ucoor^{\top}\ucoor$, $B = B(M) = (M +\lambda I)^{-1}$ and $\Phi = \Phi(M)$ for conciseness. 
Recalling that $M$ and $B$ are PSD and  commute, by construction, and that $B\leq \frac{1}{\lambda} I$, it follows:
\begin{equation}
\begin{aligned}
	\Verts{F^{\kappa}(\ucoor)}^2 &=  Tr\parens{B\Phi B^{\frac{1}{2}}M^{\frac{1}{2}}BM^{\frac{1}{2}}B^{\frac{1}{2}}\Phi B}\\
	&\leq  \lambda^{-1} Tr\parens{B\Phi B^{\frac{1}{2}}M^{\frac{1}{2}}M^{\frac{1}{2}}B^{\frac{1}{2}}\Phi B}\\
	&= \lambda^{-1} Tr\parens{M^{\frac{1}{2}}B^{\frac{1}{2}}\Phi B^{\frac{1}{2}} B B^{\frac{1}{2}}\Phi B^{\frac{1}{2}}M^{\frac{1}{2}}}\\
	&\leq \lambda^{-2}  Tr\parens{M^{\frac{1}{2}}B^{\frac{1}{2}}\Phi B^{\frac{1}{2}} B^{\frac{1}{2}}\Phi B^{\frac{1}{2}}M^{\frac{1}{2}}}\\
	&=  \lambda^{-2}Tr\parens{MB^{\frac{1}{2}}\Phi B\Phi B^{\frac{1}{2}}} = -\frac{1}{2\lambda^3}\dot{L}^{\kappa}(\ucoor_t).
\end{aligned}
\end{equation}

\paragraph{Relative error.} 
Similarly, using again that $B$ and $M$ commute:
\begin{equation}
\begin{aligned}
\Verts{F^{\kappa}(\ucoor)}^2 &=  Tr\parens{B\Phi M^{\frac{1}{2}}B^2M^{\frac{1}{2}}\Phi B}\\
	&\geq   \sigma_{min}(B^2) Tr\parens{B\Phi M^{\frac{1}{2}}M^{\frac{1}{2}}\Phi B}\\
	&\geq \sigma_{min}(B^2) Tr\parens{M^{\frac{1}{2}}\Phi B^2\Phi M^{\frac{1}{2}}}\\
	&\geq   \sigma_{min}(B^2)^2 Tr\parens{M^{\frac{1}{2}}\Phi\Phi M^{\frac{1}{2}}}\\
	&=  \sigma_{min}(B^2)^2 Tr\parens{\Phi M\Phi}\\
	&= \frac{1}{4}\sigma_{min}(B^2)^2 \Verts{\nabla L^{\kappa}(\ucoor)}^2,
\end{aligned}
\end{equation}
where we used that $B^2\geq \sigma_{min}(B^2) I$ the smallest singular value of $B^2$.  Moreover, simple estimates show that $\sigma_{min}(B^2)\geq \frac{1}{(\Verts{\ucoor}^2+\lambda)^2}$. Hence, we have established the relative error bound:
$$\Verts{F^{\kappa}(\ucoor)}^2\geq \frac{1}{4(\Verts{\ucoor}^2+\lambda)^4}\Verts{\nabla L^{\kappa}(\ucoor)}^2. 
$$

\end{proof}

\paragraph{Growth of the vector field $F^{\kappa}$.}
\begin{proof}[Proof of \cref{lem:growth_F_kappa}]
The bound holds by direct application of operator norm properties. Specifically, for $\ucoor\in \mathbb{W}$, and denoting by $R = \parens{\ucoor^{\top}\ucoor +\lambda I}^{-1}$,  we have:
\begin{equation}
\begin{aligned}
	\Verts{F^{\kappa}(\ucoor)} &\leq \Verts{\ucoor}\parens{1 + \Verts{RS(\ucoor^{\top}\ucoor S)^{\kappa}R}_{op}}\\
	 & \leq  \Verts{\ucoor}\parens{1 + \underbrace{\Verts{R}_{op}^2}_{\leq \lambda^{-2}}\Verts{S}_{op}^{1+\kappa}\Verts{\ucoor^{\top}\ucoor}_{op}^{\kappa}}\\
	 & \leq  \Verts{\ucoor}\parens{1 + \lambda^{-2}\Verts{S}_{op}^{1+\kappa}\Verts{\ucoor}^{2\kappa}}.
\end{aligned}
\end{equation}
This establishes that $F^{\kappa}$ is radially bounded by the function $\omega(\Verts{\ucoor}) = \Verts{\ucoor}\parens{1 + \lambda^{-2}\Verts{S}_{op}^{1+\kappa}\Verts{\ucoor}^{2\kappa}}$.
\end{proof}

\paragraph{Growth, coercivity and lower boundedness of the energy $L^{\kappa}$.}
\begin{proof}[Proof of \cref{lem:energy_bounded_below} (Finite lower bound on the energy)]
Let us first show that $L^{\kappa}$ is lower-bounded. Let $\ucoor\in\mathbb{W}$. Then $\ucoor\ucoor^{\top}$ is a $k\times k$ PSD matrix. Denote by $\sigma_1,\dots,\sigma_k$ its singular values. Then
\begin{align}
Tr\parens{(\ucoor\ucoor^{\top}+\lambda I)^3}=\sum_{i=1}^k(\sigma_i+\lambda)^3
\;\ge\;\sum_{i=1}^k\sigma_i^3+k\lambda^3
\;\ge\;\frac{1}{k^2}\|\ucoor\|_{\mathbb{W}}^6+k\lambda^3,
\end{align}
using $\sum\sigma_i^3\ge k^{-2}(\sum\sigma_i)^3=k^{-2}\|\ucoor\|_{\mathbb{W}}^6$ by Jensen's inequality. For the second term of $L^{\kappa}$, we use the following upper-bound:
\begin{align}
	\verts{\Tr\parens{\parens{\ucoor S\ucoor^{\top}}^{\kappa+1} }}\leq
		\Verts{S}_{op}^{\kappa+1}\Verts{\ucoor}_{\mathbb{W}}^{2(\kappa+1)}
\end{align}
Hence, with $r=\|\ucoor\|_{\mathbb{W}}$,
\begin{equation}
L^{\kappa}(\ucoor)\ge \frac{1}{3k^2}r^6-\frac{1}{\kappa +1}\Verts{S}_{op}^{\kappa+1}r^{2(\kappa+1)}+\frac{k\lambda^3}{3}.
\end{equation}
The right-hand side is bounded below for $r\ge0$, as long as $\kappa<2$ (positive leading coefficient), so $L^{\kappa}$ is bounded below. From the above inequality, we also get that $L^{\kappa}$ is coercive. 

To show that $\nabla L^{\kappa}$ and $\nabla^2 L^{\kappa}$ are radially bounded, recall that $L^{\kappa}$ is a smooth polynomial, so that $\nabla L^{\kappa}$ and $\nabla^2 L^{\kappa}$ are also smooth polynomials and are thus, a fortiori, radially bounded.
\end{proof}

\paragraph{Hessian decomposition.}
\begin{proof}[Proof of \cref{lem:kurdyka_simon_inequality} (Hessian decomposition)]
First, by direct differentiation, we have that:
\begin{align}
\nabla L^{\kappa}(\ucoor) = 2\ucoor \parens{\parens{\ucoor^{\top}\ucoor + \lambda I}^2 - S\parens{\ucoor^{\top}\ucoor S}^{\kappa}}= 2\parens{\ucoor\ucoor^{\top} + \lambda I}^2\ucoor - 2\ucoor S\parens{\ucoor^{\top}\ucoor S}^{\kappa}. 
\end{align}
Hence, differentiating the above along a direction $v$ yields:
\begin{equation}
\begin{aligned}
	\nabla^2L^{\kappa}(\ucoor)(\ucoor') =& \underbrace{2\parens{\parens{\ucoor'\ucoor^{\top} + \ucoor\ucoor'^{\top}}\parens{\ucoor\ucoor^{\top} +\lambda I} + \parens{\ucoor\ucoor^{\top} +\lambda I} \parens{\ucoor'\ucoor^{\top} + \ucoor\ucoor'^{\top}} }}_{Q(\ucoor')} \ucoor \\
	&+ 2\parens{\ucoor\ucoor^{\top} +\lambda I}^2\ucoor' - \underbrace{\parens{2\ucoor'S\parens{\ucoor^{\top}\ucoor S}^{\kappa} +2\delta_{\kappa=1} \ucoor S\parens{\ucoor^{\top}\ucoor' + \ucoor'^{\top}\ucoor}S}}_{K_1(\ucoor')}\\
	=& 2\parens{\ucoor\ucoor^{\top} +\lambda I}^2\ucoor' + Q(\ucoor')\ucoor - K_1(\ucoor')
\end{aligned}
\end{equation}
By assumption, we know that $S$ is a bounded compact operator, hence, $\ucoor'\mapsto K_1(\ucoor')$ is also bounded and compact as a sum of compositions of bounded and compact operators. Moreover $Q(\ucoor')$ is a $k\times k$ matrix, so that $K_2: \ucoor'\mapsto Q(\ucoor')\ucoor$ lives in the range of the $k$ components of $\ucoor$. Hence, $\ucoor'\mapsto Q(\ucoor')\ucoor$ has a finite range and is thus, a fortiori, compact. Finally, note that $2\parens{\ucoor\ucoor^{\top} +\lambda I}^2$ is an invertible $k\times k$ matrix. Hence, the map $R: \ucoor'\mapsto 2\parens{\ucoor\ucoor^{\top} +\lambda I}^2\ucoor'$ is invertible with a bounded inverse. We have shown, so far, that $\nabla^2 L^{\kappa}(\ucoor)$ is decomposed as the sum of a boundedly invertible operator $R$ and a bounded compact operator $K_1 + K_2$.

\end{proof}

\paragraph{Palais-Smale condition.}
\begin{proof}[Proof of \cref{lem:palais_smale} (Palais-Smale condition)]
	Consider a bounded sequence $(\ucoor_n)_{n\geq 1}$ in $\mathbb{W}$ satisfying $\Verts{\nabla L^{\kappa}(\ucoor_n)}\rightarrow 0$.  We will show that it must admit a cluster point. By direct calculations, we know that $\nabla L^{\kappa}(\ucoor_n)$ admits the following expression:
\begin{align}\label{eq:palais_smale_gradient}
	r_n = \nabla L^{\kappa}(\ucoor_n) = 2\ucoor_n \parens{ (M_n + \lambda I)^2 - S (M_n S)^{\kappa}}, \qquad M_n = \ucoor_n^{\top}\ucoor_n.
\end{align}
We will first prove that $M_n$ admits a convergent subsequence, then deduce that $\ucoor_n$ also does. 

\textbf{Existence of a convergent subsequence of $M_n$.} Left-multiplying the above equation by $\ucoor_n^{\top}$ we get:
\begin{align}
\ucoor_n^{\top}r_n = 2M_n(M_n+\lambda I)^2 - 2(M_nS)^{\kappa+1}
\end{align}
Define $X_n = 2M_n(M_n+\lambda I)^2$ and $B_n= 2(M_nS)^{\kappa}M_n$. Then since $(\ucoor_n)_{n\geq 1}$ is bounded, it follows that $(X_n)_{n\geq 1}$ and  $(B_n)_{n\geq 1}$ form  bounded sequences of bounded operators on $\ell^2(I)$. Moreover, $X_n$ is self-adjoint by construction and $X_n-B_nS = \ucoor_n^{\top}r_n\rightarrow 0$, since, by assumption $r_n\rightarrow 0$ and $\ucoor_n$ is bounded. Thus, recalling that $S$ is a bounded compact operator, we can apply \cref{lem:relative_compactness} to directly deduce that $X_n$ admits a convergent subsequence in operator norm. 

Consider the scalar function $\phi: t\mapsto 2t(t+\lambda)^2$ which is strictly increasing on $[0,+\infty)$ and thus admits a continuous inverse $\phi^{-1}$ defined on $[0,+\infty)$. Consequently, the map $\Psi: M\mapsto 2M(M+\lambda I)^{2}$ defined over bounded self-adjoint semi-definite operators on $\ell^2(I)$ is a homeomorphism, by continuous functional calculus for bounded PSD self-adjoint operators,    and thus admits a continuous inverse $\Phi^{-1}$. 
Let now $(X_{n_l})_{l\geq 1}$ be a convergent subsequence of $(X_n)_{n\geq 1}$, then we have that $X_{n_l} = \Psi(M_{n_l})$ for all $l\geq 1$. Moreover, $M_{n_l} = \Psi^{-1}(X_{n_l})$ converges by continuity of $\Psi^{-1}$. Therefore, $(M_n)_{n\geq 1}$ admits a convergent subsequence in operator norm.  

\textbf{Existence of a convergent subsequence for $(\ucoor_n)_{n\geq 1}$.} 
Multiplying \cref{eq:palais_smale_gradient} on the right by $(M_n + \lambda I)^{-2}$ and re-arranging the equality yields:
\begin{align}\label{eq:compactness_sequence}
	\ucoor_n &= \frac{1}{2}\parens{ r_n(M_n + \lambda I)^{-2} } + \ucoor_n S (M_n S)^{\kappa}(M_n + \lambda I)^{-2}.
\end{align}
Since $\ucoor_n$ is bounded, so is $M_n$, hence, the first term $r_n(M_n + \lambda I)^{-2}$ converges to $0$ as $r_n\rightarrow 0 $ by assumption. 
For the second term, we know that $M_n$ admits a convergent subsequence $(M_{n_l})_{l\geq 1}$. Moreover, by compactness of $S$, we know that $(\ucoor_{n_{l}} S)_{l\geq 1}$ admits a convergent subsequence. Thus, we have shown that one can extract a subsequence of $\ucoor_n$ so that the right hand side of \cref{eq:compactness_sequence} converges by continuity. Denote this limit by $\ucoor^{\star}$. This precisely means that such subsequence of $(\ucoor_{n})_{n \geq 1}$ converges itself to $\ucoor^{\star}$ thanks to the identity in  \cref{eq:compactness_sequence}. 
\end{proof}

\begin{lemma}\label{lem:relative_compactness}
Let $S$ be a bounded compact operator on $\ell^2(I)$, 
and \((B_n)_{n\geq 1}\) a bounded sequence of bounded operators on $\ell^2(I)$. Let \((X_n)_{n\geq 1}\) be a family of self-adjoint bounded operators satisfying:
\[
    \|X_n-B_nS\|\to0 .
\]
Then \((X_n)\) admits a convergent subsequence in operator norm. 
\end{lemma}

\begin{proof}
Let $(P_{m})_{m\geq 1}$ be a family of finite-rank orthogonal projections in $\ell^2(I)$ such that $P_m$ converges to the identity as $m$ increases. Since $S$ is compact, necessarily, $\|S(I-P_m)\|\to0$.
Moreover, we have that \(X_n-B_nS\to0\) and \((B_n)\) is bounded. Hence,  \((X_n)\) is bounded and
\begin{align}\label{eq:limsup}
\limsup_{n\to\infty}\|X_n(I-P_m)\|
 \le \sup_n\|B_n\|\,\|S(I-P_m)\|\to0
 \qquad (m\to\infty).
\end{align}
Moreover, by assumption $X_n$ is self-adjoint, i.e., \(X_n=X_n^\ast\), so that 
\[
    \|(I-P_m)X_n\|=\|X_n(I-P_m)\|,
\]
and 
\begin{align*}
\|X_n-P_mX_nP_m\|  &= \|X_n(I-P_m) + (I-P_m)X_nP_m\|\\
&\leq \Verts{X_n(I-P_m)} + \Verts{(I-P_m)X_n}\Verts{P_m}\\
&\leq 2\Verts{X_n(I-P_m)}.
\end{align*}
Therefore, combining the above upper-bound with \cref{eq:limsup} yields:
\begin{align}\label{eq:limsup_2}
    \lim_{m\to\infty}\limsup_{n\to\infty}
    \|X_n-P_mX_nP_m\|=0 .
\end{align}
We now show that the sequence $(X_n)_{n\geq 1}$ is relatively compact in the operator norm, meaning that its closure is a compact set. To achieve this, we only need to show that $(X_n)_{n\geq 1}$ admits a finite $\epsilon$-cover for any $\epsilon>0$. Fix $\epsilon >0$, then \cref{eq:limsup_2} ensures there exists $m>0$ so that $\limsup_{n\to\infty}
    \|X_n-P_mX_nP_m\|\leq \epsilon/4$. 
Consequently, there exists $N>0$ so that for all $n\geq N$:
\begin{align}\label{eq:bound_x_n_compactness}
\|X_n-P_mX_nP_m\|\leq \epsilon/2
\end{align}
Denote by $P_m\ell^2(I)$ the finite dimensional space of projections of $\ell^2(I)$ by the operator $P_m$. Then the restrictions $Q_n$ of $P_mX_nP_m$ to $P_m\ell^2(I)$ are stable, i.e., $Q_n$ are operators from $P_m\ell^2(I)$ to itself. Moreover, $(Q_n)_{n\geq 1}$ is a bounded family of bounded operators on the finite-dimensional space $P_m\ell^2(I)$. Therefore, $(Q_n)_{n}$ is relatively compact and thus admits a finite $\epsilon/2$-cover $\tilde{Q}_1,\dots,\tilde{Q}_K$, for some $K>0$ with $\tilde{Q}$ bounded operators on $P_m\ell^2(I)$ satisfying:
\begin{align}\label{eq:compactness_Q}
\min_{1\leq l\leq K}\Verts{Q_n-\tilde{Q}_l}\leq \frac{\epsilon}{2}.
\end{align}
By abuse of notation, we identify $\tilde{Q}_l$ with its unique extension on $\ell^2(I)$ satisfying $\tilde{Q}_l\ucoor =  \tilde{Q}_l P_m\ucoor$ for any $\ucoor\in \ell^2(I)$. Hence, combining \cref{eq:bound_x_n_compactness,eq:compactness_Q}, we have shown that for all $n\geq N$:
\begin{align}
\min_{1\leq l\leq K}\Verts{X_n-\tilde{Q}_l}\leq \Verts{X_n- P_mX_nP_m} +  \min_{1\leq l\leq K}\Verts{Q_n-\tilde{Q}_l}\leq \epsilon.
\end{align}
Consequently, $\{\tilde{Q}_1,\dots,\tilde{Q}_K\}\cup \{X_1,\dots,X_N\}$ is a finite $\epsilon$-cover of $(X_{n})_{n\geq 1}$ w.r.t. the operator norm. This establishes that $(X_n)_{n\geq 1}$ is relatively compact and a fortiori that $X_n$ admits a convergent subsequence. 
\end{proof}

\section{Proof of the characterization of equilibria and their stability}
\label[appendixsection]{sec:proof_equilibria_stability}

\subsection{Proof of \cref{eq:characterization_fixed_point_general} on characterization of equilibria}
\label[appendixsection]{sec:proof_characterization_equilibria}

\begin{proof}[Proof of \cref{eq:characterization_fixed_point_general}]
Let $\ucoor$ be a critical point of $F^{\kappa}$, i.e. $F^{\kappa}(\ucoor)=0$. 
Left-multiplying  $F^{\kappa}(\ucoor)$ by $\ucoor^{\top}$ and defining $M = \ucoor^{\top}\ucoor$, the following condition must hold for $M$:
\begin{align}\label{eq:fixed_point_M}
	M\parens{I- \parens{M+\lambda I}^{-1}G(M)\parens{M + \lambda I}^{-1}}=0, \qquad G(M)=S(MS)^{\kappa}.
\end{align}
Recalling that $M$ necessarily commutes with $(M+\lambda I)^{-1}$ and multiplying both sides by $(M+\lambda I)$, we directly deduce that:
\begin{align}
	(M+\lambda I)M(M+\lambda I) = MG(M).
\end{align}
Since the l.h.s. of the above expression is always symmetric, so must be its r.h.s., i.e. $MG(M)  = (MG(M))^{\top}$.
We will use such symmetry to prove that  $MG(M)$ must be diagonal.
\paragraph{$MG(M)$ is diagonal.}
Using that $G(M) = S(MS)^{\kappa}$, the symmetry condition on $MG(M)$ becomes $MS(MS)^{\kappa} = (MS(MS)^{\kappa})^{\top}$. Using the associativity of the product, it is clear that $(MS)^{\kappa}MS= SM (SM)^{\kappa}$. Furthermore, by setting $T = (MS)^{\kappa}M$, and noting that $T = M (SM)^{\kappa}$ by associativity of the product, we have shown that $TS=ST$, i.e. $T$ commutes with $S$. 
Moreover, we know, by assumption, $S$ is compact and diagonal in the canonical orthonormal basis of $\ell^2(I)$, thus it admits the following block representation with respect to the orthogonal decomposition of $\ell^2(I)$ in \cref{eq:orthogonal_space_decomposition}:  
\begin{align}
	S \equiv \begin{pmatrix}
		S^{+} & 0 & 0\\
		0 & 0 & 0\\
		0 & 0 & S^{-}\\
	\end{pmatrix},
\end{align}
where $\equiv$ denotes operator similarity,  $S^{+}$ and $S^{-}$ collect the positive and negative eigenvalues of $S$ in decreasing order w.r.t. their absolute values. For any canonical basis element $e_i$, the following relation holds $STe_i = TSe_i= s_i Te_i$ for all $i\in I$.  The previous equation implies $Te_i$ is an eigenvector of $S$ associated to eigenvalue $s_i$.  By assumption, eigenspaces of $S$ associated to non-zero eigenvalues have dimension $1$. Therefore, $Te_i$ belongs to the span of $e_i$. On the other hand, the null space of $S$ is invariant under $T$. In other words, $T$ admits the following block diagonal structure w.r.t. the orthogonal decomposition of $\ell^2(I)$ in \cref{eq:orthogonal_space_decomposition}:   
\begin{align}
	T \equiv \begin{pmatrix}
		T^{+} & 0 & 0\\
		0 & T^0 & 0\\
		0 & 0 & T^{-}
	\end{pmatrix},
\end{align}
where $T^{+}$ and $T^{-}$ are diagonal, while $T^0$ is a block associated to $0$ eigenvalues of $S$. 
Hence, we have shown that $M$ satisfies an equation of the form:
\begin{align}\label{eq:fixed_point_M_2}
	(M+\lambda I)M(M+\lambda I) = ST \equiv \begin{pmatrix}
		S^{+}T^{+} & 0 & 0\\
		0 & 0 & 0\\
		0 & 0 & S^{-}T^{-}
	\end{pmatrix}.
\end{align}
The above equation establishes that  $(M+\lambda I)M(M+\lambda I)$ is diagonal in the canonical basis. It remains to show that $M$ is itself diagonal.

$M$ is positive semi-definite (PSD), and is compact, as an operator of rank at most $k$.  Therefore, by the spectral theorem, it admits an eigenvalue decomposition of the form: $M = R^{\top}\Lambda R$, where $R$ is orthogonal ($R^{\top}R=I$) and $\Lambda$ is diagonal with non-negative components. Consequently, $(M+\lambda I)M(M+\lambda I) = \Gamma(M) = R^{\top}\Gamma(\Lambda)R$, where $\Gamma(\Lambda)$ is a diagonal matrix whose $i$-th diagonal element is given by $\gamma(\ell_i)$ where $\ell_i$ is the $i$-th diagonal element of $\Lambda$ and $\gamma(\ell) = \ell(\ell+\lambda)^2$ for any $\ell$ in $[0,+\infty)$. Since $\gamma$ is continuous and strictly increasing over $[0,+\infty)$, it must have an inverse $\gamma^{-1}$. 
Applying continuous functional calculus for bounded self-adjoint PSD operators, we can define the inverse  $\Gamma^{-1}$ of $\Gamma$ which applies $\gamma^{-1}$ to the eigenvalues of any PSD matrix $D$ without affecting the eigenvectors. Consequently, we can write $M = \Gamma^{-1}(\Gamma(M)) = \Gamma^{-1}(ST)$. Since we established that $ST$ is diagonal, $\Gamma^{-1}(ST)$ must also be a diagonal matrix. Therefore, $M$ is diagonal and admits the following block representation w.r.t. the orthogonal decomposition of $\ell^{2}(I)$ from \cref{eq:orthogonal_space_decomposition}:
\begin{align}
	M \equiv \begin{pmatrix}
		M^{+} & 0 & 0\\
		0 & M^{0} & 0\\
		0 & 0 & M^{-}
	\end{pmatrix}.
\end{align}

\paragraph{Characterizing fixed points.} 
Since $M$ is diagonal, we only need to compute its diagonal coefficients. From \cref{eq:fixed_point_M_2}, we directly see that $M^{0}=0$. For $M^{+}$ and $M^{-}$, we use \cref{eq:fixed_point_M}, which implies the following equations on its diagonal components $(m_i)_{i\in I}$:
\begin{align}
	m_i\parens{1- \frac{m_i^{\kappa}s_i^{\kappa+1}}{\parens{m_i+\lambda}^2}} = 0,\qquad m_i\geq 0.
\end{align}
In particular, we can solve in closed form for $\kappa \in \{0,1\}$. 

\begin{itemize}
	\item \textbf{Case $\kappa=0$.} In this case, $m_i$ has the following admissible solutions:
\end{itemize}
	\begin{align}\label{eq:diag_elements_kappa_0}
	m_i \in \{0, (g(s_i,\lambda))^2\} ,\qquad 
	g(s,\lambda) = (\sqrt{s}-\lambda )_{+}^{\frac{1}{2}}\mathds{1}_{s\geq 0}.
	\end{align}

\begin{itemize}
	\item \textbf{Case $\kappa=1$.} In this case, $m_i$ has the following admissible solutions:
\end{itemize}
\begin{align}\label{eq:diag_elements_kappa_1}
	m_i \in \{0,  (f_{1}(s_i,\lambda))^2 , (f_{-1}(s_i,\lambda))^2\}, \quad 
	f_{\epsilon}(s,\lambda) = \frac{1}{2}\parens{ \verts{s} + \epsilon \sqrt{s^2 -4\lambda} }\mathds{1}_{\verts{s}^2-4\lambda\geq 0}.
	\end{align}
Recalling that $M = \ucoor^{\top}\ucoor$, it follows that $M$ is of rank at most $k$, as $\ucoor$ is itself of rank at most $k$, and therefore there can be at most $k$ non-zero diagonal elements in $M$. Let $r $ be the number of such non-zero elements. We have that $0\leq r\leq k$. If $r>0$, let $\mathcal{K} = \{i(1), \dots, i(r)\}$ be the set of indices corresponding to such non-zero elements and set $\mathcal{K}= \emptyset$ if $r=0$. When $r=0$, clearly $\ucoor=0$. Otherwise, when $r>0$, the condition $M = \ucoor^{\top}\ucoor$ along with the constraints on $M$ imply that $\ucoor$ is of the form:
\begin{align}
	\ucoor = Q \parens{\sum_{p=1}^r E_{p,i(p)}}M^{\frac{1}{2}}, \text{ with } 1\leq r\leq k,
\end{align}
where $Q$ is an orthogonal matrix in $\mathbb{R}^{k\times k}$ and $i(1),\dots,i(r)$ is an increasing sequence of positive integers. Conversely, we can verify that any element $\ucoor$ of the form \cref{eq:critical_point_matrix_characterization} is a critical point of the vector field $F^{\kappa}$. 
\end{proof}

\subsection{Proof of the diagonalization of the Jacobian of the vector field $F^{\kappa}$ at critical points}\label[appendixsection]{sec:proof_diagonalization_jecobian}

\begin{proof}[Proof of \cref{prop:diagonalization_jacobian}]
Let $\ucoor$ be a critical point of $F^{\kappa}$, which is necessarily of the form:
\begin{align}
	\ucoor = Q \underbrace{\parens{\sum_{p=1}^r E_{p,i(p)}}}_{\ucoor_0}M^{\frac{1}{2}},
\end{align}
Let us first express the Jacobian of $F^{\kappa}$ at $\ucoor$. By direct differentiation, one gets that the vector field $F^{\kappa}$ defined in \cref{eq:general_vector_field} admits the following Jacobian representation when acting on elements $\ucoor'$ in $\mathbb{W}$: 
\begin{equation}
\begin{gathered}
	\nabla F^{\kappa}(\ucoor)[\ucoor'] =  \ucoor'\parens{I - RG^{\kappa}(M)R}+ \ucoor R\parens{DRG^{\kappa}(M) + G^{\kappa}(M)RD- \nabla G^{\kappa}[M](D)}R,\\
	\text{with }\quad M=\ucoor^{\top}\ucoor,\qquad R = (M +\lambda I)^{-1},\qquad D = \ucoor'^{\top}\ucoor + \ucoor^{\top}\ucoor',\\
	\text{and } G^{\kappa}(M) = S(MS)^{\kappa},  \quad \nabla G^{0}[M](M') = 0,\qquad \nabla G^{1}[M](M')= SM'S.
\end{gathered}
\end{equation}
We can further express $\nabla F^{\kappa}(\ucoor)[\ucoor']$ as the sum of three terms:
\begin{equation}
\begin{aligned}
	\nabla F^{\kappa}(\ucoor)[\ucoor'] 
	=& \underbrace{\ucoor'\parens{I - RG^{\kappa}(M)R}}_{H_1} \\
	&+ \underbrace{\ucoor R\parens{\ucoor'^{\top}\ucoor RG^{\kappa}(M) + G^{\kappa}(M)R\ucoor'^{\top}\ucoor- \nabla G^{\kappa}[M](\ucoor'^{\top}\ucoor)  }R}_{H_2}\\
	&+ \underbrace{\ucoor R\parens{\ucoor^{\top}\ucoor'RG^{\kappa}(M) + G^{\kappa}(M)R\ucoor^{\top}\ucoor'- \nabla G^{\kappa}[M](\ucoor^{\top}\ucoor')}R}_{H_3}.
\end{aligned}
\end{equation}
We will establish a general diagonal  decomposition of $\nabla F^{\kappa}(\ucoor)$.  Let $\ucoor$ be a critical point of $F^{\kappa}$ as in the statement. 
By definition of the critical point $\ucoor$ in \cref{eq:critical_point_matrix_characterization_2}, $M = \ucoor^{\top}\ucoor$ is diagonal, so is $R$ and $G^{\kappa}(M)$. Let us denote by $m_{i}$, $r_i$ and $g_i$ the corresponding $i$-th diagonal coefficient for $i\in I$ with $r_i$ and $g_i$ satisfying:
\begin{align}
	r_i = (m_i+\lambda)^{-1}, \qquad g_i = s_i^{\kappa+1}m_i^{\kappa}
\end{align}

We will show that the desired decomposition can be obtained using the orthonormal basis $(QE_{p,j})_{1\leq p\leq k, j\in I}$ of $\mathbb{W}$. To this end, we will need to compute the action of $\nabla F^{\kappa}(\ucoor)$ on an element $\tilde{E}_{p,j} = QE_{p,j}$ for a given $1\leq p\leq k$ and $j\in I$.  As a first step, we need to compute  $\tilde{E}_{p,j}^{\top}\ucoor$ as it appears in both $H_2$ and $H_3$. Using the decomposition of $\ucoor$ from \cref{eq:critical_point_matrix_characterization_2} and properties of the canonical basis, it holds that:
\begin{align}
	E_{p,j}^{\top}\ucoor_0 &= \sum_{q=1}^r E_{p,j}^{\top}E_{q,i(q)} =  E_{p,j}^{\top}E_{p,i(p)}\mathds{1}_{1\leq p\leq r}, &\tilde{E}_{p,j}^{\top}\ucoor &= m_{i(p)}^{\frac{1}{2}} E_{p,j}^{\top}E_{p,i(p)}\mathds{1}_{1\leq p\leq r}.
\end{align} 
Furthermore, recall that right-multiplication of $E_{p,j}$ with a diagonal operator $T$ yields $E_{p,j}T = t_j E_{p,j}$, where $t_j$ is the $j$-th diagonal coefficient of $T$. Hence, we can express $H_1$, $H_2$ and $H_3$ as:
\begin{equation}\label{eq:three_terms_jacobian}
\begin{gathered}
	H_1 = \parens{1-r_j^2g_j}\tilde{E}_{p,j}\\
	H_2 =  m_j^{\frac{1}{2}}m_{i(p)}^{\frac{1}{2}}\beta_{i(p),j} Q\ucoor_0E_{p,j}^{\top}E_{p,i(p)}\mathds{1}_{1\leq p\leq r}\\
	H_3 = m_{i(p)}\beta_{i(p),j}  Q\ucoor_0E_{p,i(p)}^{\top}E_{p,j}\mathds{1}_{1\leq p\leq r},
\end{gathered}
\end{equation}
where we introduced $\beta_{i,j}$ for conciseness:
\begin{align}\label{eq:expression_beta}
	\beta_{i,j}= r_{i} r_j^2g_j + r_jr_{i}^2g_{i}- r_j r_{i}s_j s_{i} \mathds{1}_{\kappa =1}.
\end{align}
Moreover, defining $i^{-1}(i(p)) = p$ for $1\leq p\leq r$, the following identity holds:
\begin{align}
	\ucoor_0E_{p,j}^{\top}E_{p,i(p)} &=  E_{i^{-1}(j),i(p)}\mathds{1}_{j\in I_{sup}},  &\ucoor_0E_{p,i(p)}^{\top}E_{p,j} &= E_{p,j}.
\end{align}
Using the above in \cref{eq:three_terms_jacobian} and writing $q = i^{-1}(j)$ and $l = i(p)$ allows us to express the Jacobian as follows:

\begin{align}\label{eq:expression_nabla_F_eigen}
	\nabla F^{\kappa}(\ucoor)(\tilde{E}_{p,j}) 
		=& \parens{1- r_j^2g_j + m_{l}\beta_{l,j} \mathds{1}_{p\leq r}}\tilde{E}_{p,j} + m_j^{\frac{1}{2}}m_{l}^{\frac{1}{2}}\beta_{l,j}\tilde{E}_{q, l}\mathds{1}_{p\leq r, j\in I_{sup}}.
\end{align}
We distinguish two cases depending on whether $j\in I_{sup}$ and $p\leq r$ or not. 

\paragraph{Case $j\notin I_{sup}$ or $p>r$.}
The second term in \cref{eq:expression_nabla_F_eigen} vanishes and we get that $\tilde{E}_{p,j}$ is an eigenvector of $\nabla F^{\kappa}(\ucoor)$ with eigenvalue $\mu_{p,j}$ given by:
\begin{equation}
\begin{aligned}
	\mu_{p,j} &= \parens{1- r_j^2g_j + m_{i(p)}\beta_{i(p),j} \mathds{1}_{1\leq p\leq r}}\\
	&= \begin{cases}
		0, & j\in I_{sup}, p>r\\
		1-\lambda^{-2}s_{j}^{\kappa+1}m_{j}^{\kappa}, &  j\notin I_{sup}, p>r\\
		1-\lambda^{-1} \frac{s_{j}^{\kappa+1}m_{j}^{\kappa}}{m_{i(p)}+\lambda} + \lambda^{-1}m_{i(p)}\parens{1- \frac{s_{i(p)}s_j}{m_{i(p)}+\lambda} \mathds{1}_{\kappa=1}}, &  j\notin I_{sup} , p\leq r
	\end{cases},
\end{aligned}
\end{equation}
where we used that $m_j=0$ and $m_{i(p)}$ satisfies the equation $(m_{i(p)}+\lambda)^2 = s_{i(p)}^{\kappa+1}m_{i(p)}^{\kappa}$. 
\paragraph{Case $j\in I_{sup}$ and $p\leq r$.}
If $j = i(p)$, then $j'=j$ and $p'=p$ so that $\tilde{E}_{p, j} = \tilde{E}_{p', j'}$. Hence, $\tilde{E}_{p,j}$ is again an eigenvector of   $\nabla F^{\kappa}(\ucoor)$ associated to eigenvalue $\mu_{p,j}$ given by:
\begin{equation}
\begin{aligned}
	\mu_{p,j} 
		&=  1- \underbrace{r_j^2g_j}_{=1} + 2 m_{j}\beta_{j,j}
		 = 2m_{j}\parens{2r_j^3g_j- r_j^2 s_j^2\mathds{1}_{\kappa=1} } \\
		&= 4r_jm_{j}- 2\underbrace{r_j^2 m_{j}s_j^2\mathds{1}_{\kappa=1}}_{=\mathds{1}_{\kappa=1}}  
		=  \frac{4m_j}{m_j+\lambda}-2\mathds{1}_{\kappa=1}.
\end{aligned}
\end{equation}

If $j\neq i(p)$, then $\tilde{E}_{p,j}$ and $\tilde{E}_{i^{-1}(j),i(p)}$ are two different orthogonal elements of $\mathbb{W}$. Moreover, setting $p'= i^{-1}(j)$ and $j' = i(p)$, and noticing that $i^{-1}(j') = p$ and $i(p') = j$, it follows that $\tilde{E}_{i^{-1}(j'),i(p')} =  \tilde{E}_{p,j}$.  Therefore, applying $\nabla F^{\kappa}(\ucoor)$ to $\tilde{E}_{p',j'}$ yields the following expression:

\begin{equation}
\begin{aligned}
			\nabla F^{\kappa}(\ucoor)[\tilde{E}_{p',j'}]
		=& \parens{1- r_{j'}^2g_{j'} + m_{i(p')}\beta_{i(p'),j'}}\tilde{E}_{p',j'}+ m_{j'}^{\frac{1}{2}}m_{i(p')}^{\frac{1}{2}}\beta_{i(p'),j'} \tilde{E}_{i^{-1}(j'), i(p')}\\
		=& \parens{1- r_{j'}^2g_{j'} + m_{i(p')}\beta_{i(p'),j'}}\tilde{E}_{i^{-1}(j),i(p)} + m_{j'}^{\frac{1}{2}}m_{i(p')}^{\frac{1}{2}}\beta_{i(p'),j'}\tilde{E}_{p, j}.
\end{aligned}
\end{equation}
We have shown so far that $span\parens{\tilde{E}_{p,j},\tilde{E}_{p',j'}}$ is a stable subspace of dimension $2$ for $\nabla F^{\kappa}(\ucoor)$. 
Moreover, $i^{-1}(j)\leq r$ and $p\leq r$, so that, by definition, $m_{j}$ and $m_{j'}$ are positive and satisfy the equation $r_{i}^{-2} = (m_{i}+\lambda)^2 = s_{i}^{\kappa+1}m_{i}^{\kappa} = g_{i}$ for $i\in \{j,j'\}$. Consequently, the terms $1-r_{j'}^2g_{j'}$ and  $1-r_{j}^2g_{j}$ appearing in the expressions of $\nabla F^{\kappa}(\ucoor)[\tilde{E}_{p',j'}]$ and $\nabla F^{\kappa}(\ucoor)[\tilde{E}_{p,j}]$ vanish and $\nabla F^{\kappa}(\ucoor)$ restricted to the $2$-dimensional subspace $span\parens{\tilde{E}_{p,j},\tilde{E}_{p',j'}}$ admits the matrix representation:
\begin{align}
	\beta_{j',j} \begin{pmatrix}
		m_{j'},& m_{j'}^{\frac{1}{2}}m_{j}^{\frac{1}{2}}\\
		m_{j'}^{\frac{1}{2}}m_{j}^{\frac{1}{2}} & m_{j}
	\end{pmatrix}.
\end{align}
The above matrix admits eigenvalues $\mu_{p,j}^{0}= 0$ and $\mu_{p,j}^{1} = \beta_{j',j}\parens{m_j + m_{j'}}= \gamma_{p,p'}$ with associated eigenvectors $\bar{E}^{0}_{p,j}$ and $\bar{E}^{1}_{p,j}$:
\begin{align}
	\bar{E}^{0}_{p,j} = 
\frac{1}{(m_{j}+ m_{j'})^{\frac{1}{2}}}\parens{m_{j}^{\frac{1}{2}}\tilde{E}_{p,j} -  m_{j'}^{\frac{1}{2}}\tilde{E}_{p',j'}}, \quad \bar{E}^{1}_{p,j} =  \frac{1}{(m_{j}+ m_{j'})^{\frac{1}{2}}}\parens{m_{j'}^{\frac{1}{2}}\tilde{E}_{p,j} +  m_{j}^{\frac{1}{2}}\tilde{E}_{p',j'}}.
\end{align}
\end{proof}

\subsection{Proof of \cref{prop:stability_kappa_0,prop:stability_kappa_1} on stability of equilibria}
\label[appendixsection]{sec:proofs_stability}

\textbf{Stability of critical point of $F^{0}$.}
\begin{proof}[Proof of \cref{prop:stability_kappa_0}]
	We will first prove that for any critical point $\ucoor$ of the form in \cref{eq:decomposition_local_minimizers_kappa_0}, the Jacobian $\nabla F^0(\ucoor)$ has only non-negative eigenvalues, whereas for any other critical point, there exists a negative eigenvalue. 
We will heavily rely on the eigen-decomposition of $\nabla F^{0}(\ucoor)$ from \cref{prop:diagonalization_jacobian} which establishes that all its eigenvalues are of the form:
\begin{align}\label{eq:eigenvalues_Jacobian_kappa_0}
	\mu_{p,j} = \begin{cases}
		1-\lambda^{-2}s_j, &  j\notin I_{sup} \text{ and } p>r\\
		\delta_{p,j}, &  j\notin I_{sup} \text{ and }  p\leq r\\
		\frac{4m_j}{m_j+\lambda}, & j=i(p) \text{ and } p\leq r \\
		\gamma_{p,p'}, & p,p' \leq r \text{ and } j=i(p')> i(p)\\
		0, & p' \leq r \text{ and } j=i(p') \text{ and }  \parens{ \parens{p\leq r \text{ and } j< i(p)} \text{ or } p>r},   
	\end{cases}
\end{align}
 with  $\delta_{p,j}$ and $\gamma_{p,p'}$ given by:
 \begin{equation}\label{eq:delta_gamma_kappa_0}
 \begin{aligned}
  	\delta_{p,j} &= \frac{(s_{i(p)}-s_j)}{\lambda (m_{i(p)} +\lambda)}\\
 	\gamma_{p,p'} &= \frac{\parens{m_{i(p)}+m_{i(p')}}\parens{ \sqrt{s_{i(p)}}  + \sqrt{s_{i(p')}}}  }{\parens{m_{i(p)}+\lambda}\parens{m_{i(p')}+\lambda} }.
 \end{aligned}
 \end{equation}
\paragraph{ Critical points of the form in \cref{eq:decomposition_local_minimizers_kappa_0} are stable.} 
To see this, it suffices to check that all eigenvalues $\mu_{p,j}$ of the Jacobian of $F^0$ in \cref{eq:eigenvalues_Jacobian_kappa_0} are  non-negative. We treat each case separately. We only need to check the cases where $j\notin I_{sup}$ as, for all the other ones, $\mu_{p,j}$ is clearly non-negative.

\begin{itemize}
	\item \textbf{$j\notin I_{sup} \text{ and } p>r$.} In this case we have $\mu_{p,j} = 1-\lambda^{-2}s_j$. We need to show that necessarily $s_j \leq \lambda^2$. Recall that $r$ is the maximal integer such that $s_{i(r)}>\lambda^2$ and $0\leq r\leq k$. We distinguish two cases, either $r=k$ or $r<k$. If $r<k$, then having $s_j> \lambda^2$ contradicts the maximality of $r$, since we could include $s_j$ to $I_{sup}$. The case $r=k$ is excluded here, since we would have  $p>k$, which is impossible since $p$ must always be at most $k$. Consequently, $\mu_{p,j}\geq 0$.   
	\item \textbf{$j\notin I_{sup} \text{ and } p\leq r$.} In this case $\mu_{p,j}= \delta_{p,j}$. We need to show that $s_{i(p)}\geq s_j$. 
	Indeed, $s_{i(1)},\dots, s_{i(r)}$ consists of the $r$ highest eigenvalues, hence if $s_j>s_{i(p)}$ then $j$ must be in $I_{sup}$ which contradicts the condition $j\notin I_{sup}$.
\end{itemize}

\paragraph{Strict saddles characterization.}
Consider a critical point $\ucoor$ of $F^0$. By \cref{eq:characterization_fixed_point_general},  it admits a general decomposition of the form:
\begin{align}
	\ucoor = Q\parens{\sum_{p=1}^{r} E_{p,i(p)}}M^{\frac{1}{2}},
\end{align}
where $Q$ is an orthogonal matrix in $\mathbb{R}^{k\times k}$, $r$ is an integer between $0$ and $k$, $I_{sup} = \{i(1),\dots, i(r)\}$ is an increasing sequence of integers greater than or equal to $1$, and $M$ is a positive semi-definite diagonal operator on $\ell^2(I)$ whose only non-zero diagonal coefficients $(m_{i(p)})_{1\leq p\leq r}$ are given by:
\begin{align}
	m_{i(p)}^{\frac{1}{2}} = g(s_{i(p)},\lambda),\qquad g(s,\lambda) = (\sqrt{s}-\lambda)^{\frac{1}{2}}_{+}.
\end{align}
We are interested in critical points that are not of the form of \cref{eq:decomposition_local_minimizers_kappa_0}. Two possibilities arise: 
\begin{itemize}
	\item \textbf{Case 1}: There exists $j_0\notin I_{sup}$  and $1\leq p_0\leq r$ with $s_{j_0}>s_{i(p_0)}$. 
	\item \textbf{Case 2}: There exists $j_0\notin I_{sup}$ so that $s_{i(p)}>s_{j_{0}}> \lambda^2$  for all  $p$ such that $1\leq p \leq r$ and the rank $r$ satisfies $r<k$.
\end{itemize}
We will treat each case separately by identifying a negative eigenvalue of the Jacobian $\nabla F^0(\ucoor)$.

\paragraph{Case 1.} Consider the eigenvalue $\mu_{p_0,j_0}$. Since $p_0\leq r$ and $j_0\notin I_{sup}$, $\mu_{p_0,j_0}$ must take the form: 
\begin{align}
	\mu_{p_0,j_0} &= \delta_{p_0,j_0}=  \frac{(s_{i(p_0)}-s_{j_0})}{\lambda (m_{i(p_0)} +\lambda)}<0.
\end{align}
$\mu_{p_0,j_0}$ is negative since by assumption $s_{j_0}>s_{i(p_0)}$. 
Therefore, the considered critical points cannot be stable since the Jacobian admits a negative eigenvalue.

\paragraph{Case 2.} Set $p_0=r+1\leq k$ and  consider the eigenvalue $\mu_{p_0,j_0}$ of the Jacobian $\nabla F^0(\ucoor)$. Since $j_0\notin I_{sup}$ and $p_0> r$, then $\mu_{p_0,j_0}$ must be equal to:
\begin{align}
	\mu_{p_0,j_0} = 1-\lambda^{-2}s_{j_0}. 
\end{align}
By assumption $s_{j_0}>\lambda^2$, therefore $\mu_{p_0,j_0} <0$. Therefore, such critical point cannot be stable. 
\end{proof}

\textbf{Stability of critical point of $F^{1}$.}
\begin{proof}[Proof of \cref{prop:stability_kappa_1}]
We will first prove that for any critical point $\ucoor$ of the form in \cref{eq:decomposition_local_minimizers_E}, the Jacobian $\nabla F^1(\ucoor)$ has only non-negative eigenvalues, whereas for any other critical point, there exists a negative eigenvalue. 
We will heavily rely on the eigen-decomposition of $\nabla F^{1}(\ucoor)$ from \cref{prop:diagonalization_jacobian} which establishes that all its eigenvalues are of the form:
\begin{align}\label{eq:eigenvalues_Jacobian}
	\mu_{p,j} = \begin{cases}
		1, &  j\notin I_{sup} \text{ and } p>r\\
		\delta_{p,j}, &  j\notin I_{sup} \text{ and }  p\leq r\\
		\frac{4m_j}{m_j+\lambda} - 2, & j=i(p) \text{ and } p\leq r \\
		\gamma_{p,p'}, & p,p' \leq r \text{ and } j=i(p')> i(p)\\ 
		0, & p' \leq r \text{ and } j=i(p') \text{ and }  \parens{ \parens{p\leq r \text{ and } j< i(p)} \text{ or } p>r},
	\end{cases}
\end{align}
 with  $\delta_{p,j}$ and $\gamma_{p,p'}$ given by:
 \begin{equation}
  \label{eq:delta_gamma}
 \begin{aligned}
  	\delta_{p,j} &= \frac{1}{\lambda (m_{i(p)} +\lambda)}m_{i(p)}s_{i(p)}(s_{i(p)}-s_j),\\
 	\gamma_{p,p'} &= \frac{\parens{m_{i(p)}+m_{i(p')}}\parens{ \verts{s_{i(p)}} m_{i(p)}^{\frac{1}{2}} + \verts{s_{i(p')}}m_{i(p')}^{\frac{1}{2}} - s_{i(p)}s_{i(p')} }  }{\parens{m_{i(p)}+\lambda}\parens{m_{i(p')}+\lambda} }.
 \end{aligned}
\end{equation}
\paragraph{ Critical points of the form in \cref{eq:decomposition_local_minimizers_E} are stable.} 
Here $I_{sup}=I_{sup}^{+}\cup I_{sup}^{-}$.
To see why these critical points are stable, it suffices to check that all eigenvalues $\mu_{p,j}$ of the Jacobian of $F^1$ in \cref{eq:eigenvalues_Jacobian} are non-negative. We treat each case separately.

\begin{itemize}
	\item \textbf{$j\notin I_{sup} \text{ and } p>r$.} In this case $\mu_{p,j}=1> 0$.
	\item \textbf{$j\in I_{sup} \text{ and } \parens{p>r \text{ or } \parens{p \leq r \text{ and } j< i(p)}}$.} This corresponds to the last case in \cref{eq:eigenvalues_Jacobian} for which $\mu_{p,j}=0\geq 0$.
	\item \textbf{$j=i(p) \text{ and } p\leq r$.} By definition of the critical point, $m_{i(p)}$ is assumed to be positive and of the form $ m_{i(p)}^{\frac{1}{2}}=\frac{1}{2}\parens{\verts{s_{i(p)} } + \sqrt{s_{i(p)}^2-4\lambda}}$. This is only possible if $\verts{s_{i(p)}} \geq 2\sqrt{\lambda}$, which, in turn, can only happen if $\lambda \leq  \frac{1}{4}$ since  $\verts{s_{i}}\leq 1$ for all $i\in I$.  
    In this case,  $m_{i(p)}^{\frac{1}{2}}\geq \frac{1}{2}\verts{s_{i(p)}}\geq \sqrt{\lambda}$, so that $m_{i(p)}\geq \lambda$. Recalling that $\mu_{p,j} = \frac{4m_j}{m_j+\lambda}-2 = 2\frac{m_{i(p)}-\lambda}{m_{i(p)}+\lambda}$ it follows that  $\mu_{p,j}\geq 0$.
	\item \textbf{$j\notin I_{sup} \text{ and }  p\leq r$.} In this case, by assumption $m_{i(p)}>0$ which is only possible if $s_{i(p)}\neq 0$. If $s_{i(p)}s_j\leq 0$, then  necessarily $s_{i(p)}(s_{i(p)}-s_j)>0$ so that $\mu_{p,j} = \delta_{p,j}>0$. If   $s_{i(p)}s_j>0$, then $s_j$ cannot be such that $\verts{s_j}>\verts{s_{i(p)}}$. Otherwise, $j$ would belong to the set $I_{sup}$, which contains the $r^{+}$ highest positive and $r^{-}$ smallest negative eigenvalues. This is not possible by assumption since $j\notin I_{sup}$. Thus, necessarily $\verts{s_j}<\verts{s_{i(p)}}$ (non-zero eigenvalues are assumed to be distinct). Therefore, $s_{i(p)}(s_{i(p)}-s_j)>0$  and $\mu_{p,j} = \delta_{p,j}>0$.
	\item \textbf{$p,p' \leq r \text{ and } j=i(p')> i(p)$.} In this case, $\mu_{p,j}=\gamma_{p,p'}$ as defined in \cref{eq:delta_gamma}. By definition of the critical point, $m_{i(p)}$ and $m_{i(p')}$ are positive and given by an equation of the form in \cref{eq:solution_m_i_p_kappa_1}, so that $m_{i(p)}^{\frac{1}{2}}\geq \frac{1}{2}\verts{s_{i(p)}}$ and $m_{i(p')}^{\frac{1}{2}}\geq \frac{1}{2}\verts{s_{i(p')}}$. Consequently, it follows that: 
	$$\mu_{p,j}=\gamma_{p,p'}\geq \frac{\parens{m_{i(p)}+m_{i(p')}}\parens{ \verts{s_{i(p)}}^2  + \verts{s_{i(p')}}^2  - 2s_{i(p)}s_{i(p')} }  }{2\parens{m_{i(p)}+\lambda}\parens{m_{i(p')}+\lambda} }\geq 0.$$
\end{itemize}

\paragraph{Strict saddles characterization.}
Consider a critical point $\ucoor$ of $F^1$. By \cref{eq:characterization_fixed_point_general},  it admits a general decomposition of the form:
\begin{align}
	\ucoor = Q\parens{\sum_{p=1}^{r} E_{p,i(p)}}M^{\frac{1}{2}},
\end{align}
where $Q$ is an orthogonal matrix in $\mathbb{R}^{k\times k}$, $r$ is an integer between $0$ and $k$, $I_{sup} = \{i(1),\dots, i(r)\}$ is an increasing sequence of integers greater than or equal to $1$, and $M$ is a positive semi-definite diagonal operator on $\ell^2(I)$ whose only non-zero diagonal coefficients $(m_{i(p)})_{1\leq p\leq r}$ are given by:
\begin{align}
	m_{i(p)}^{\frac{1}{2}} = f_{\epsilon_p}(s_{i(p)},\lambda),\qquad f_{\epsilon}(s,\lambda) =
		\frac{1}{2}\parens{ \verts{s} + \epsilon \sqrt{s^2 -4\lambda} }\mathds{1}_{\verts{s}^2-4\lambda\geq 0}.
\end{align}
where $(\epsilon_p)_{p=1}^r$ can take any values in $\{-1,1\}$. We are interested in critical points that are not of the form of \cref{eq:decomposition_local_minimizers_E}. Two possibilities arise: 
\begin{itemize}
	\item \textbf{Case 1}: There exists $1\leq p_0\leq r$ so that $m_{i(p_0)}^{\frac{1}{2}} = f_{-1}(s_{i(p_0)},\lambda)$ and $s_{i(p_0)}^2> 4\lambda$.
	\item \textbf{Case 2}: For all $1\leq p \leq r$, we have $m_{i(p)}^{\frac{1}{2}} = f_{1}(s_{i(p)},\lambda)$ and there exists $j_0$ and $p_0$ so that $\verts{s_{j_0}} > \verts{s_{i(p_0)}} $, $s_{j_0}s_{i(p_0)}>0$ and $j_0\notin I_{sup}$.
\end{itemize}
We will treat each case separately by identifying a negative eigenvalue of the Jacobian $\nabla F^1(\ucoor)$.

\paragraph{Case 1.} Set $j_0= i(p_0)$ and consider the eigenvalue $\mu_{p_0,j_0}$. Since $p_0\leq r$ and $j_0=i(p_0)$, $\mu_{p_0,j_0}$ must take the form:
\begin{equation}
\begin{aligned}
	\mu_{p_0,j_0} &= \frac{4m_{j_0}}{m_{j_0}+\lambda} - 2\\
	&= 2\frac{m_{j_0}^{\frac{1}{2}} + \sqrt{\lambda}}{m_{j_0}+\lambda}\parens{ m_{j_0}^{\frac{1}{2}} - \sqrt{\lambda} }\\
	&= \frac{m_{j_0}^{\frac{1}{2}} + \sqrt{\lambda}}{m_{j_0}+\lambda}\parens{ \verts{s_{j_0}} - 2\sqrt{\lambda} - \sqrt{s_{j_0}^2- 4\lambda} }\\
	&= \frac{m_{j_0}^{\frac{1}{2}} + \sqrt{\lambda}}{m_{j_0}+\lambda}\parens{\verts{s_{j_0}} - 2\sqrt{\lambda}}^{\frac{1}{2}}\parens{ \parens{\verts{s_{j_0}} - 2\sqrt{\lambda}}^{\frac{1}{2}} - \parens{\verts{s_{j_0}}+ 2\sqrt{\lambda} }^{\frac{1}{2}} }<0.
\end{aligned}
\end{equation}
In the above, we used that $m_{j_0}=f_{-1}^2(s_{j_0},\lambda)$ for the third line and that $\verts{s_{j_0}} > 2\sqrt{\lambda}$ and $\lambda>0$ to ensure the last inequality is strict. Since the Jacobian admits a negative eigenvalue, the considered critical points cannot be stable.

\paragraph{Case 2.} Consider the eigenvalue $\mu_{p_0,j_0}$ of the Jacobian $\nabla F^1(\ucoor)$. Since $j_0\notin I_{sup}$ and $p_0\leq r$, $\mu_{p_0,j_0}$ must be equal to:
\begin{align}
	\mu_{p_0,j_0} = \frac{1}{\lambda (m_{i(p_0)} +\lambda)}m_{i(p_0)}s_{i(p_0)}(s_{i(p_0)}-s_{j_0}).
\end{align}
Moreover, since $s_{i(p_0)}$ and $s_{j_0}$ have the same sign, it holds that $s_{i(p_0)}(s_{i(p_0)}-s_{j_0}) = \verts{s_{i(p_0)}}(\verts{s_{i(p_0)}}-\verts{s_{j_0}})$. Moreover, by assumption, $\verts{s_{j_0}}  > \verts{s_{i(p_0)}}$, so that $\mu_{p_0,j_0}<0$. This establishes that such critical point cannot be stable. 
\end{proof}

\newpage
\section{Experimental details and additional results}
\label[appendixsection]{sec:appendix:experimental_details}

\subsection{Experimental details}
\label[appendixsection]{sec:training_details}
\label[appendixsection]{app:exp-details}

\begin{figure}[t]
\centering
\begin{tikzpicture}[
    font=\small,
    >=stealth,
    box/.style={draw, rounded corners=2pt, minimum height=6mm, inner sep=3pt, align=center},
    enc/.style={box, fill=NavyBlue!12, minimum width=9mm},
    stat/.style={box, fill=Goldenrod!20},
    loss/.style={box, fill=BrickRed!10},
    fwd/.style={->, thick},
    grad/.style={->, thick, dashed, BrickRed},
    sg/.style={draw=black, fill=white, rounded corners=1pt, inner sep=1pt, font=\scriptsize\sffamily},
]
\node (x) {$x$};
\node (y) [below=9mm of x] {$y$};
\node[enc] (U) [right=3mm of x] {$U_\theta$};
\node[enc] (V) [right=3mm of y] {$V_\psi$};
\node (PX) [right=3mm of U] {$\Phi_X$};
\node (PY) [right=3mm of V] {$\Phi_Y$};
\draw[fwd] (x) -- (U); \draw[fwd] (U) -- (PX);
\draw[fwd] (y) -- (V); \draw[fwd] (V) -- (PY);

\node[stat, anchor=west] (ema) at ([xshift=10mm, yshift=5mm]PX.east)
    {\textbf{EMA statistics} (rate $\eta$)\\
     $\Sigma \leftarrow (1-\eta)\Sigma + \frac{\eta}{B}(\Phi_X^{\top}\Phi_Y + \Phi_Y^{\top}\Phi_X)$\\
     $N \leftarrow (1-\eta)N + \frac{\eta}{B}(\Phi_X^{\top}\Phi_X + \Phi_Y^{\top}\Phi_Y)$};
\node[stat, inner xsep=5pt] (solve) [right=5mm of ema]
    {\textbf{$k\times k$ ridge solve}\\
     $\hat Q^\star = (N+\lambda I)^{-1}$\\
     $\hat P^\star = \Sigma\,\hat Q^\star$};
\draw[fwd] (ema) -- (solve);

\node[loss, anchor=west] (L) at ([xshift=10mm, yshift=-5mm]PY.east)
    {\textbf{Auxiliary loss} $\widehat{\mathcal{L}}_{\mathrm{aux}}$ \quad ($u=U_\theta(x)$, $v=V_\psi(y)$)\\
     $\hat{\mathbb{E}}_{\mathcal{B}}\bigl[\tfrac12 u^{\top}\hat Q^\star(\hat P^\star u-v)+\tfrac12 v^{\top}\hat Q^\star(\hat P^\star v-u)+\tfrac{\lambda}{2}(\|u\|^2+\|v\|^2)\bigr]$};

\coordinate (bus) at ([xshift=5mm]PX.east);
\coordinate (busY) at ([xshift=5mm]PY.east);
\draw[thick] (PX.east) -- (bus);
\draw[thick] (PY.east) -- (busY);
\draw[thick] (bus |- ema.west) -- (busY |- L.west);
\draw[fwd] (bus |- ema.west) -- node[sg, midway] {sg} (ema.west);
\draw[fwd] (busY |- L.west) -- (L.west);
\draw[fwd] (solve.south) -- node[sg, midway] {sg} (solve.south |- L.north);

\path (U.south) -- (V.north) coordinate[midway] (gm);
\coordinate (gb) at ([yshift=-4mm]L.south);
\draw[thick, dashed, BrickRed] (L.south) -- (gb) -| node[pos=0.3, fill=white, inner sep=1pt, font=\scriptsize]
    {$\nabla_{\theta,\psi}\widehat{\mathcal{L}}_{\mathrm{aux}}$} ([xshift=-3mm]y.west) |- (gm);
\draw[grad] (gm) -- (U.south);
\draw[grad] (gm) -- (V.north);
\end{tikzpicture}
\vspace{-0.5em}
\caption{\textbf{One step of \cref{alg:peira}.} Both views of a mini-batch $\mathcal{B}$ are encoded into features $\Phi_X,\Phi_Y\in\mathbb{R}^{B\times k}$. Through a stop-gradient (sg), their second moments update EMA estimates of the signal $\Sigma_{U,V}$ and noise $N_{U,V}$, which give $\hat P^\star$ and $\hat Q^\star$ in closed form by a $k\times k$ inversion. With $\hat P^\star,\hat Q^\star$ held fixed, the auxiliary loss has the same gradient as $\mathcal{E}_{\method}$ (\cref{prop:gradient_expression_appendix}); only the encoder parameters $\theta,\psi$ are updated (dashed). There is no target encoder and no backpropagation through the inverse; with a shared encoder, $V_\psi=U_\theta$.}
\label{fig:algo}
\end{figure}

\subsubsection*{ImageNet-1K (ResNet-50, 100 epochs)}

This subsection details the pretraining and evaluation pipeline used to
produce the linear-probe top-1 accuracies reported in \cref{tab:in1k} of the
main text; the corresponding hyperparameters are summarized in
\cref{tab:in1k_config}.

\paragraph{Setup.}
We pretrain on the $1.28$M-image ImageNet-1K training set (ILSVRC2012 split,
distributed under the ImageNet Terms of Access, a custom non-commercial
research license), using the two-view asymmetric augmentation stack of
VICReg / BYOL~\citep{Bardes:2022,Grill:2020a}: two $224^2$ crops per image
with asymmetric Gaussian blur ($p{=}1.0$ on view~1, $p{=}0.1$ on view~2)
and solarization ($p{=}0.0$ on view~1, $p{=}0.2$ on view~2). All three
methods use a three-layer MLP projector of hidden width $8192$ ($2048$ for SIGReg); the
output-width disparity reported in \cref{tab:in1k_config} is theoretically
motivated: SIGReg's isotropic-Gaussian target is unsatisfiable in
too-high-dimensional spaces, whereas VICReg's decorrelation objective
benefits from wider projectors. 

\paragraph{Optimization.}
All three methods follow the VICReg LARS recipe of \citet{Bardes:2022}
(Layer-wise Adaptive Rate Scaling, LARS~\citep{You:2017}, with linear warm-up followed by cosine decay; full
values in \cref{tab:in1k_config}). Gradient clipping is enabled for
\method{} only. For \method{}, the EMA rate $\eta$ of \cref{alg:peira} is
cosine-annealed from $0.99$ to $0.95$ (\cref{remark:ema-schedule}) for both
$k{=}256$ and $k{=}1024$. We use a \emph{decoupled} online linear probe trained with
SGD on stop-gradient features; coupling the online probe to the backbone's
cosine-decayed LR under-reads the final top-1 accuracy by $1$--$5$\,pp
in our measurements, so we decouple for all methods and
rely on an offline probe as the canonical comparison metric.

\paragraph{Evaluation.}
Following Appendix~E of \citet{Bardes:2022}, we train an offline linear
classifier on frozen features for $100$ epochs with SGD (batch size $1024$,
base learning rate $0.25$, cosine schedule, weight decay $10^{-6}$).
Augmentations at probe-training time are random-resized crop + horizontal
flip; evaluation uses bicubic resize to $256$ followed by a $224^2$ center
crop. We report top-1 accuracy on the ImageNet-1K validation set.

\paragraph{Compute.}
A single $100$-epoch ResNet-50 pretraining run takes $17$ to $23$ hours on
one node of $8\times$ NVIDIA H200 GPUs (\cref{tab:app-wallclock}). We report $3$ seeds per
configuration, and preliminary hyperparameter exploration consumed roughly
$20$ times the compute budget of the final reported runs.

\begin{table}[h]
\centering
\small
\caption{ImageNet-1K training configurations (RN50, $100$ epochs). All three methods use LARS, batch size $2048$, and a separate online linear probe (SGD, fixed LR $0.3$, no weight decay, no schedule).}
\label{tab:in1k_config}
\begin{tabular}{lccc}
\toprule
Hyperparameter & VICReg & SIGReg & \method{} (ours) \\
\midrule
\multicolumn{4}{l}{\emph{Optimizer / scheduler (LARS)}} \\
Base LR                  & $2.4$        & $2.4$        & $2.4$        \\
Weight decay             & $10^{-6}$    & $10^{-6}$    & $10^{-6}$    \\
Momentum                 & $0.9$        & $0.9$        & $0.9$        \\
LARS trust coeff.        & $10^{-3}$    & $10^{-3}$    & $10^{-3}$    \\
Warmup epochs            & $10$         & $10$         & $10$         \\
Cosine min LR            & $2\times10^{-3}$ & $2\times10^{-3}$ & $2\times10^{-3}$ \\
Gradient clip            & --           & --           & $1.0$        \\
\midrule
\multicolumn{4}{l}{\emph{Projector}} \\
Hidden dim               & $8192$       & $2048$       & $8192$       \\
Output dim               & $8192$       & $128$        & $256$ / $1024$ \\
Final-layer bias         & False        & True         & True         \\
\midrule
\multicolumn{4}{l}{\emph{Loss-specific}} \\
Loss type                & VICReg       & SIGReg       & \method{}    \\
Coefficients             & $\alpha=25$, $\beta=25$, $\gamma=1$ & $\lambda=0.01$ & $\lambda=0.1$, $\eta_{\mathrm{init}}=0.99$, $\eta_{\min}=0.95$ \\
\bottomrule
\end{tabular}
\end{table}

\subsubsection*{CIFAR-10 (ResNet-18, 1000 epochs)}

This subsection details the pretraining and evaluation pipeline used to
produce the CIFAR-10 linear-probe accuracies reported in \cref{tab:in1k};
the corresponding hyperparameters are summarized in \cref{tab:cifar_config}.
The VICReg and SIGReg baselines use the default hyperparameters of the
\texttt{eb\_jepa} codebase~\citep{ebjepa}.

\paragraph{Setup.}
We pretrain a CIFAR-style ResNet-18 backbone (3$\times$3 first convolution,
no max-pool) on the $50$k-image CIFAR-10\footnote{The CIFAR-10 dataset~\citep{Krizhevsky:2009} is
released by the University of Toronto with no explicit license and is freely
available for research use.} training set~\citep{Krizhevsky:2009},
using a symmetric two-view augmentation pipeline shared by all three
methods: a $32^2$ random-resized crop with scale $\in[0.2, 1.0]$, color
jitter ($p{=}0.8$, brightness/contrast $0.4$, saturation $0.2$, hue $0.1$),
random grayscale ($p{=}0.2$), random solarization ($p{=}0.1$), and
horizontal flip ($p{=}0.5$); no Gaussian blur is used at this resolution.
All three methods share a three-layer MLP projector of hidden width $2048$;
the per-method output widths are reported in \cref{tab:cifar_config}.

\paragraph{Optimization.}
All three methods use LARS~\citep{You:2017} with linear warm-up followed by
cosine decay; full values in \cref{tab:cifar_config}. Gradient clipping is
disabled for VICReg and SIGReg\footnote{Our current SIGReg implementation
multiplies the Epps--Pulley statistic by the batch size, in accordance with
\citet{Balestriero:2025}; the $\lambda{=}0.1$ reported in
\cref{tab:cifar_config} is therefore comparable, after this rescaling, to
the $\lambda{=}10$ reported in Table~5 of \citet{ebjepa}.}, and set to
$1.0$ from epoch~$4$ onwards for \method{}: applying clipping from the
start plateaus performance, so a $4$-epoch warm-up period without clipping
lets the gradient estimator stabilize before the constraint is enforced.

\paragraph{Evaluation.}
A linear classifier is trained online during pretraining on detached backbone
features, so that no gradient flows to the encoder, and we report its top-1
accuracy on the $10$k-image CIFAR-10 validation set at the end of training. Numbers in \cref{tab:in1k} are mean $\pm$
one-sigma sample standard deviation over $3$ seeds.

\paragraph{Compute.}
A single $1000$-epoch ResNet-18 CIFAR-10 run takes approximately
$3.5$~hours on one H100 GPU. We report $3$ seeds per configuration,
and preliminary hyperparameter exploration consumed roughly
$10\times$ the compute budget of the final reported runs.

\begin{table}[h]
\centering
\small
\caption{CIFAR-10 training configurations (ResNet-18, $1000$ epochs).
All three methods share the backbone, batch size $256$, the two-view
symmetric augmentation pipeline above, and the online linear-probe
evaluation above.}
\label{tab:cifar_config}
\begin{tabular}{lccc}
\toprule
Hyperparameter & VICReg & SIGReg & \method{} (ours) \\
\midrule
\multicolumn{4}{l}{\emph{Optimizer / scheduler (LARS)}} \\
Base LR                  & $0.3$        & $0.3$        & $0.04$              \\
Weight decay             & $10^{-4}$    & $10^{-4}$    & $10^{-4}$           \\
Warmup epochs            & $10$         & $10$         & $10$                \\
Warmup start LR          & $3{\times}10^{-5}$ & $3{\times}10^{-5}$ & $3{\times}10^{-5}$ \\
Cosine min LR            & $0$          & $0$          & $0$                 \\
Gradient clip            & --           & --           & $1.0$ from epoch~$4$ \\
\midrule
\multicolumn{4}{l}{\emph{Projector}} \\
Hidden dim               & $2048$       & $2048$       & $2048$       \\
Output dim               & $2048$       & $128$        & $256$ / $1024$ \\
\midrule
\multicolumn{4}{l}{\emph{Loss-specific}} \\
Loss type                & VICReg       & SIGReg       & \method{}    \\
Coefficients             & $\alpha=1$, $\beta=1$, $\gamma=80$ & $\lambda=0.1$ & $\lambda=0.7$, $\eta_{\mathrm{init}}=0.99$, $\eta_{\min}=0.95$ \\
\bottomrule
\end{tabular}
\end{table}

\subsubsection*{Additional experiments}

\Cref{tab:app-extra} summarizes the additional experiments of
\cref{app:width,app:longtrain,app:ema-bs,app:wallclock,app:transfer,app:synthetic,app:cifar-initvar}.
Each ImageNet-1K ablation changes a single axis of the \method{} recipe and
keeps the pretraining pipeline and the offline linear-probe protocol above;
seeds are drawn from $\{1,1000,10000\}$.
Effective ranks use one of two estimators. The \emph{in-training} estimator
computes $\erank$ of backbone embeddings accumulated over $10$ steps of the
per-GPU batch ($N=10\,B/8$ vectors), so its ceiling $\min(N,2048)$ depends on
the batch size; we normalize it by the $\erank$ of i.i.d.\ Gaussian features
under the same estimator. The \emph{held-out} estimator computes $\erank$ of
backbone and projector embeddings of the final checkpoint on the first
$20{,}480$ ImageNet-1K validation images, which exceeds every width $k$ we
use. Protocol details specific to each experiment are given in its section.

\paragraph{Setting of the ImageNet-1K ablations.}
\label{par:ablation-recipe}
The ImageNet-1K ablations of \method{} (projector width, EMA schedule, batch
size, longer pretraining and the $\lambda$ study of \cref{fig:in1k-peira} and
\cref{sec:appendix:eff_rank}) anneal $\eta$ from $0.9$ to $0.5$ and use the
mini-batch estimate of $\Sigma$ in $\hat{P}^{\star}$; everything else follows
\cref{tab:in1k_config}. Their $k{=}256$ baseline matches \cref{tab:in1k}
within seed noise.

\begin{table}[h]
\centering
\small
\footnotesize
\setlength{\tabcolsep}{3pt}
\caption{\textbf{Additional experiments.} Base recipe, varied factor, number
of seeds per configuration and reported metrics. IN1K: ImageNet-1K, ResNet-50,
$100$ epochs, recipes of \cref{tab:in1k_config} (\method{} ablations: setting of
\cref{par:ablation-recipe}); top-1: offline linear probe.}
\label{tab:app-extra}
\begin{tabular}{llll}
\toprule
Experiment & Base recipe & Varied factor (seeds) & Metrics \\
\midrule
Width (\cref{app:width})          & IN1K \method{} & $k\in\{64,\dots,2048\}$ ($1$--$3$) & top-1, $\erank$, transfer \\
Long training (\cref{app:longtrain}) & IN1K \method{}, VICReg & $300$ epochs ($1$) & top-1 \\
EMA schedule (\cref{app:ema-bs})  & IN1K \method{} & $(\eta_{\mathrm{init}},\eta_{\min})$, $6$ cells ($1$) & top-1, $\erank$ \\
Batch size (\cref{app:ema-bs})    & IN1K \method{} & $B\in\{512,\dots,4096\}$, LR $\propto B$ ($1$) & top-1, $\erank$ \\
Wall-clock (\cref{app:wallclock}) & IN1K, all three & none & s / epoch \\
Transfer (\cref{app:transfer})    & IN1K, all three & $8$ target datasets ($1$--$3$) & accuracy \\
Synthetic (\cref{app:synthetic})  & linear encoders & $\lambda$; init.\ scale ($1$) & recovered rank \\
CIFAR-10 (\cref{app:cifar-initvar}) & \method{}, SD & init.\ scale $10^{-2}$--$10$ ($3$) & accuracy, $\erank$, align. \\
\bottomrule
\end{tabular}
\end{table}

\subsection{Per-epoch correlation of $\mathcal{E}_{\method}$ with downstream accuracy}
\label{sec:appendix:loss_vs_acc}

This subsection records the closed-form values that the auxiliary loss
$\mathcal{L}_{\mathrm{aux}}$ and the trace-form objective
$\mathcal{E}_{\method}$ take at equilibrium, and explains why
\cref{fig:in1k-peira} (right) plots $\mathcal{E}_{\method}$ rather than
$\mathcal{L}_{\mathrm{aux}}$ against linear-probe top-1.

\begin{lemma}[Closed-form value of $\mathcal{L}_{\mathrm{aux}}$ at the closed-form regressor]
\label{lem:laux_closed_form}
For any $(U,V)\in\mathcal{U}\times\mathcal{V}$, evaluating $\mathcal{L}_{\mathrm{aux}}$ of
\cref{prop:gradient_expression_appendix} at the closed-form
$P^{\star}=\Sigma_{U,V}(N_{U,V}+\lambda I)^{-1}$ and
$Q^{\star}=(N_{U,V}+\lambda I)^{-1}$ of \cref{eq:closed-form_expression_P}
gives
\begin{equation}\label{eq:laux_closed_form}
\mathcal{L}_{\mathrm{aux}}(U,V;P^{\star},Q^{\star})
=
\tfrac{\lambda}{2}\bigl[\Tr(N_{U,V})-\Tr(\Sigma_{U,V}(Q^{\star})^{2})\bigr].
\end{equation}
\end{lemma}
\begin{proof}
From \cref{eq:auxiliary-loss_appendix},
$\mathcal{L}_{\mathrm{aux}}(U,V;P,Q)=\tfrac{1}{2}\Tr\bigl((QP+\lambda I)N_{U,V} - Q\Sigma_{U,V}\bigr)$.
Setting $P=P^{\star}=\Sigma_{U,V}Q^{\star}$ and using the algebraic identity
$N_{U,V}Q^{\star}=I-\lambda Q^{\star}$ (which follows from $N_{U,V}=(N_{U,V}+\lambda I)-\lambda I$),
\begin{align*}
\Tr(Q^{\star}P^{\star}N_{U,V})
&=\Tr\bigl(Q^{\star}\Sigma_{U,V}Q^{\star}N_{U,V}\bigr)
=\Tr\bigl(Q^{\star}\Sigma_{U,V}(I-\lambda Q^{\star})\bigr)\\
&=\Tr(Q^{\star}\Sigma_{U,V})-\lambda\Tr(\Sigma_{U,V}(Q^{\star})^{2}),
\end{align*}
where the last equality uses trace cyclicity ($Q^{\star}$ symmetric). Substituting,
\begin{equation}
\mathcal{L}_{\mathrm{aux}}(U,V;P^{\star},Q^{\star})
=\tfrac{1}{2}\bigl[\Tr(Q^{\star}\Sigma_{U,V})-\lambda\Tr(\Sigma_{U,V}(Q^{\star})^{2})+\lambda\Tr(N_{U,V})-\Tr(Q^{\star}\Sigma_{U,V})\bigr],
\end{equation}
which simplifies to \eqref{eq:laux_closed_form}.
\end{proof}

\paragraph{Equilibrium values of $\mathcal{L}_{\mathrm{aux}}$ and $\mathcal{E}_{\method}$.}
The identity~\eqref{eq:laux_closed_form} of \cref{lem:laux_closed_form}
holds for any $(U,V)$. At the stable equilibrium of \cref{prop:stability-peira},
the spectral filter of \cref{prop:peira-thresholded-subspace} forces
$N_{U,V}$ and $\Sigma_{U,V}$ to be simultaneously diagonalizable, with
eigenvalues $(\sqrt{c_i}-\lambda)$ and $c_i(\sqrt{c_i}-\lambda)$
respectively on every active mode~$i$ and zero elsewhere; in particular,
$\Sigma_{U,V}(Q^{\star})^2$ shares its eigenvalues with $N_{U,V}$ on the
active subspace and both matrices vanish on its orthogonal complement, so
$\Tr(N_{U,V})=\Tr(\Sigma_{U,V}(Q^{\star})^{2})$ and
$\mathcal{L}_{\mathrm{aux}}(U,V;P^{\star},Q^{\star})=0$. The trace-form
objective of \cref{eq:csr-objective} reaches its minimum
\eqref{eq:optim_val_peira},
$\mathcal{E}_{\method}=-\tfrac{1}{2}\sum_{i=1}^{\rank_{\mathrm{max}}}(\sqrt{c_i}-\lambda)^2<0$.
\Cref{fig:in1k-peira} (right) plots $\mathcal{E}_{\method}$ (evaluated
post-hoc at the closed-form regressor $P_{U,V}$ from the running covariance
buffers of \cref{alg:peira}) against online linear-probe top-1 across our
ResNet-50 / IN1K runs.

\paragraph{SIGReg counterpart.}
For comparison, \cref{fig:in1k-sigreg} (right) shows the analogous per-epoch
correlation between the negative SIGReg total training loss
$-\mathcal{L}_{\mathrm{train}}$ and online linear-probe
top-1, across the SIGReg ResNet-50 / IN1K runs of
\cref{sec:appendix:sigreg_sensitivity}; the negative loss is also strongly
correlated with downstream accuracy in the $\lambda$ range covered by
our sweep, mirroring the convention used for $-\mathcal{E}_{\method}$ in
the main paper.

\subsection{Effective rank vs.\ $\lambda$ on ImageNet-1K}
\label{sec:appendix:eff_rank}

This subsection tracks the entropy-based effective rank of \method{}'s
backbone and projector embeddings during ImageNet-1K pretraining as a
function of $\lambda$, both at end of training and along the trajectory.

\paragraph{Metric.}
We track the entropy-based effective rank as a basis- and label-free summary of representational diversity, sensitive to partial collapse without requiring a labelled probe; this is the standard SSL diagnostic introduced by \citet{RankMe}. Concretely, for an embedding batch $Z\in\mathbb{R}^{N\times d}$ with singular values
$\sigma_1\geq\cdots\geq\sigma_d$ and $p_i=\sigma_i/\sum_j\sigma_j$,
\begin{equation}
  \erank(Z) \;=\; \exp\!\left(-\!\sum_{i:p_i>0} p_i\log p_i\right) \;\in\; [1,d].
\end{equation}
Since $\sigma_i(Z)=\sqrt{N\,\lambda_i(\hat\Sigma_Z)}$ with
$\hat\Sigma_Z=Z^{\!\top}\!Z/N$, $\erank$ is a one-number summary of the empirical
covariance eigenspectrum.

\paragraph{What is reported.}
\Cref{fig:in1k-lambda-eff-rank} reports end-of-training (epoch $\geq 95$)
$\erank$ of the backbone ($d{=}2048$) and the projector ($d{=}256$) across
$\lambda\in\{0.025,0.05,0.1,0.2,0.5\}$ in the ablation setting of
\cref{par:ablation-recipe} ($1$--$4$ seeds per cell, $13$ runs total). The
per-epoch backbone-rank trajectory is in \cref{fig:in1k-peira} (middle) of
the main text.

\paragraph{Observations.}
\begin{itemize}[leftmargin=1.2em,topsep=2pt,itemsep=2pt]
\item Projector $\erank$ sits at $\sim\!248$ across every $\lambda$,
  within $\sim\!3\%$ of the $256$-d cap.
\item Backbone $\erank$ is $\sim\!1314$ at $\lambda\in\{0.025,0.05\}$ and
  $\sim\!1283$--$1290$ at $\lambda\in\{0.1,0.2,0.5\}$, a $\sim\!30$-unit
  ($\sim\!2\%$) drop, then a plateau.
\item $\erank$ stays clear of any collapse regime across the full
  $20\times$ $\lambda$ span.
\end{itemize}

\paragraph{Interpretation.}
Recall from \cref{sec:dynamics} that $\numcca$ denotes the number of
canonical correlations of the data distribution (\cref{def:cca}) and
$\rank_{\mathrm{max}}\leq\min(k,\numcca)$ is the largest rank supporting a
non-collapsed equilibrium of \cref{prop:peira-thresholded-subspace}. We draw three conclusions from the $\erank$ measurements above. First, the projector ($d{=}256$) sits within $\sim\!3\%$ of its dimensionality cap, i.e.\ the projector output dimension $k=256$ that bounds $\rank_{\mathrm{max}}$: the projector spectrum is set by this cap rather than by $\lambda$ whenever $k$ binds, which is the regime our sweep operates in. Second, for $\lambda$ small relative to the leading $\sqrt{c_i}$, \cref{prop:peira-thresholded-subspace} predicts a gradual contraction of the active eigenvalues $\tfrac{1}{4}(\sqrt{c_i}-\lambda)^2$ without a strict-rank drop; the mild $\sim\!2\%$ backbone contraction we observe across the $20\times$ $\lambda$ span is consistent with this regime, although we do not claim to observe a change in $\rank_{\mathrm{max}}$ itself. Third, \cref{prop:peira-thresholded-subspace} together with \cref{prop:stability-peira} rules out collapse as a stable equilibrium of the \method{} dynamics, and the absence of any $\erank$ collapse on real data corroborates this. Two caveats: the $\lambda$ window stays well below the regime in which \cref{prop:peira-thresholded-subspace} would predict visible mode killing ($\lambda\to\sqrt{c_1}$), so these numbers do not test the strict-rank prediction $\rank_{\mathrm{max}}=\#\{i:\sqrt{c_i}>\lambda\}$; and $\erank$ is a one-number entropy summary, so we do not resolve the full eigenspectrum here.

\begin{figure}[h]
\centering
\includegraphics[width=0.95\linewidth]{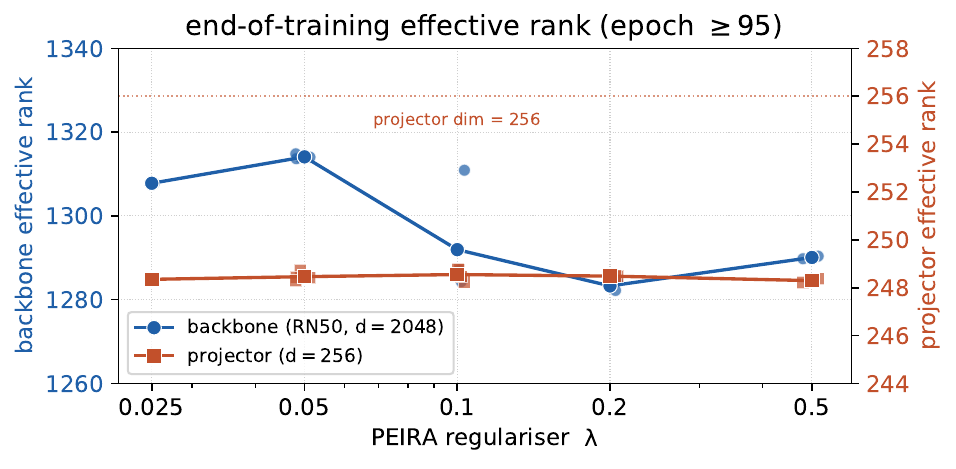}
\caption{\textbf{Effective rank vs.\ $\lambda$ on ImageNet-1K.} End-of-training $\erank$~\citep{RankMe} of the encoder during
ImageNet-1K pretraining (RN50, $100$ epochs, projector recipe of
\cref{tab:in1k_config}; epoch $\geq 95$, mean across seeds, scatter shows
individual seeds) for the backbone (left axis, blue, $d{=}2048$) and
projector (right axis, orange, $d{=}256$) across
$\lambda\in\{0.025,0.05,0.1,0.2,0.5\}$. The per-epoch backbone-rank
trajectory is in \cref{fig:in1k-peira} (middle) of the main text.}
\label{fig:in1k-lambda-eff-rank}
\end{figure}

\subsection{Comparison with SIGReg on ImageNet-1K}
\label{sec:appendix:sigreg_sensitivity}

We report the SIGReg counterparts to the three \method{} diagnostics of
\cref{fig:in1k-peira} of the main text in a single 3-panel figure
(\cref{fig:in1k-sigreg}). The runs use the SIGReg ablation configuration
($p_h{=}2048$, $p_o{=}128$), which we use as SIGReg's frozen baseline on
ResNet-50 / IN1K.

\paragraph{Sensitivity to $\lambda$.}
\Cref{fig:in1k-sigreg} (left) sweeps SIGReg's coefficient $\lambda$ over the
range we tested on ResNet-50 / IN1K. Within the tested span
$\lambda\in[0.005, 0.02]$ the mean linear-probe top-1 swings by $0.80$\,pp;
values $\lambda\geq 0.1$ collapsed in our preliminary tests and were excluded
from the reported sweep. Outside the tested range,
\citet{Balestriero:2025} report empirical stability over
$\lambda\in[10^{-3}, 5{\cdot}10^{-1}]$ on ImageNet-100 with $V{=}8$
multi-crop views.

\paragraph{Loss as a model-selection signal.}
\Cref{fig:in1k-sigreg} (right) plots the per-epoch negative SIGReg total
training loss $-\mathcal{L}_{\mathrm{train}}$ against online linear-probe
top-1; the negative loss is strongly correlated with downstream accuracy
in the $\lambda$ range covered by our sweep, mirroring the convention
used for $-\mathcal{E}_{\method}$ in \cref{fig:in1k-peira}~/
\cref{sec:appendix:loss_vs_acc}.

\paragraph{Backbone effective-rank trajectory.}
\Cref{fig:in1k-sigreg} (middle) tracks the entropy-based effective rank of
the RN50 backbone embeddings throughout pretraining for
$\lambda\in\{0.005, 0.01, 0.02\}$. The narrow span tested here (a $4\times$
ratio) is too small to expose a clear separation between curves; this stands
in contrast to the \method{} trajectory of \cref{fig:in1k-peira} (middle),
where the $20\times$ span $\lambda\in[0.025, 0.5]$ produces a visible
ordering of early-epoch ranks. We report this for completeness rather than as
evidence of a qualitative gap between the two methods.

\begin{figure}[h]
\centering
\includegraphics[width=0.32\linewidth]{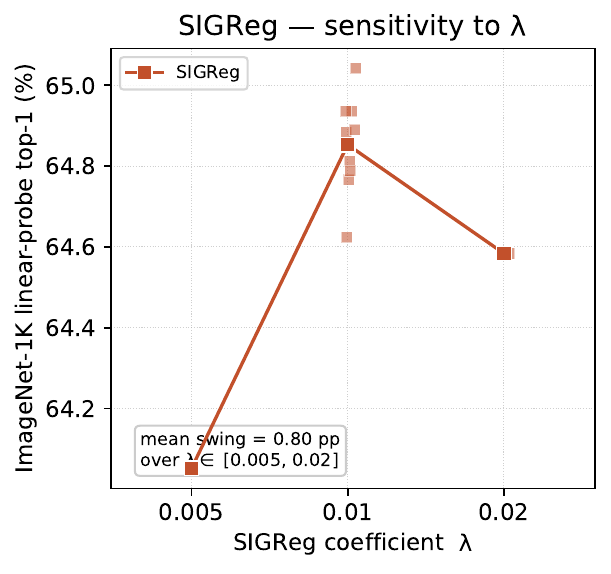}\hfill
\includegraphics[width=0.32\linewidth]{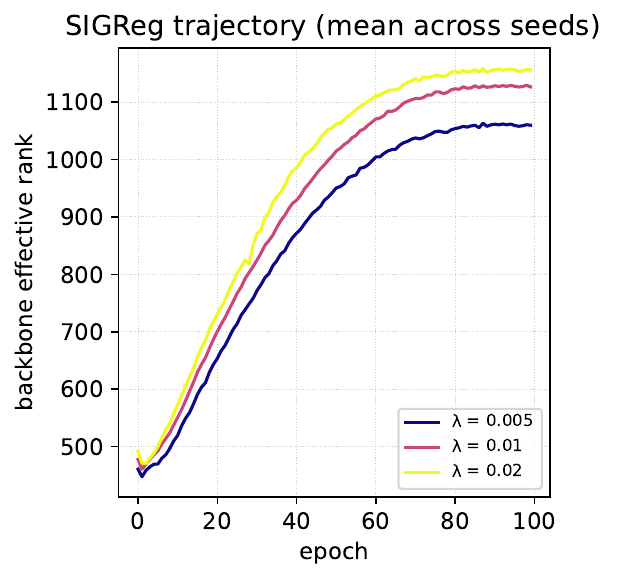}\hfill
\includegraphics[width=0.32\linewidth]{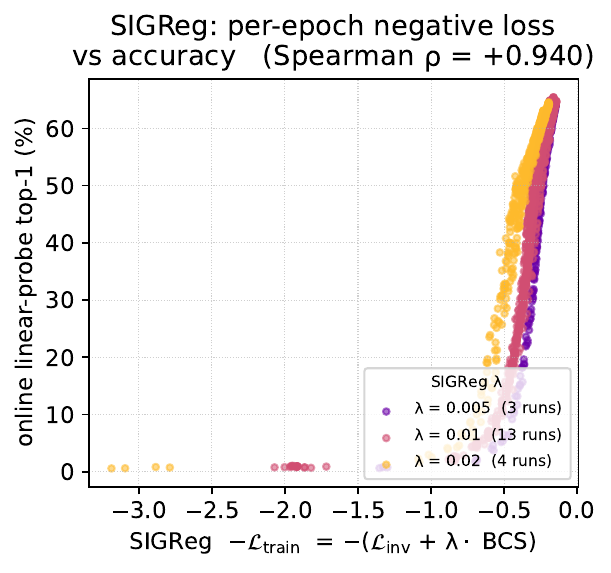}
\caption{\textbf{ImageNet-1K diagnostics for SIGReg.} ResNet-50, $100$~epochs, frozen
ablation projector $p_h{=}2048$, $p_o{=}128$; $1$--$3$ random seeds per
cell. Mirrors \cref{fig:in1k-peira} of the main text. \emph{Left:}
sensitivity of the online linear-probe top-1 to $\lambda$ at end of training;
each dot is one sweep cell, $\lambda\geq 0.1$ collapsed in preliminary tests
and is not shown. \emph{Middle:} backbone entropy-based
effective rank~\citep{RankMe} versus epoch, colored by~$\lambda$.
\emph{Right:} per-epoch negative SIGReg total training loss
$-\mathcal{L}_{\mathrm{train}}$ versus online linear-probe top-1
(one point per epoch per run); per-run Spearman correlation reported
in the legend.}
\label{fig:in1k-sigreg}
\end{figure}

\subsection{Comparison with VICReg on ImageNet-1K}
\label{sec:appendix:vicreg_sensitivity}

We compare \method{}'s diagnostics with VICReg's at the published RN50 / $100$-epoch recipe, restricted to the analogues of the middle and right panels of \cref{fig:in1k-peira} (\cref{fig:in1k-vicreg}, $18$ runs: $3$ seeds across $6$ design-axis ablations of the frozen recipe, all using the \texttt{vicreg\_ref} loss with $\alpha=\beta=25$, $\gamma=1$).

\paragraph{Coefficient-sensitivity ablation.} We do not include a sensitivity panel: our reproduction did not sweep any of $(\alpha, \beta, \gamma)$ at the published recipe. \citet{Bardes:2022}'s Table~7 (App.~D.4; their $(\lambda, \mu, \nu)$ are our $(\alpha, \beta, \gamma)$) reports a joint coefficient ablation at the same backbone and budget: turning on the covariance term lifts top-1 by $+11.1$\,pp ($57.5\% \to 68.6\%$), while the diagonal sweep $\lambda=\mu \in \{5, 10, 25, 50\}$ at fixed $\nu=1$ produces only a $\sim\!0.5$\,pp spread, consistent with VICReg being largely insensitive to its coefficients in the stable regime $\nu<\mu$, $\lambda\!\approx\!\mu$.

\paragraph{Loss as a model-selection signal.} Backbone effective rank (\cref{fig:in1k-vicreg}, left) climbs to $\sim\!1300$--$1400$ across all recipes without collapsing; per-epoch negative training loss $-\mathcal{L}_{\mathrm{train}}$ is strongly correlated with online linear-probe top-1 within every recipe (per-recipe Spearman $\rho \geq +0.998$, \cref{fig:in1k-vicreg}, right), mirroring \method{} (\cref{sec:appendix:loss_vs_acc}) and SIGReg (\cref{sec:appendix:sigreg_sensitivity}) and supporting use of the total loss as a label-free model-selection signal within a recipe.

\begin{figure}[t]
\centering
\includegraphics[width=0.48\linewidth]{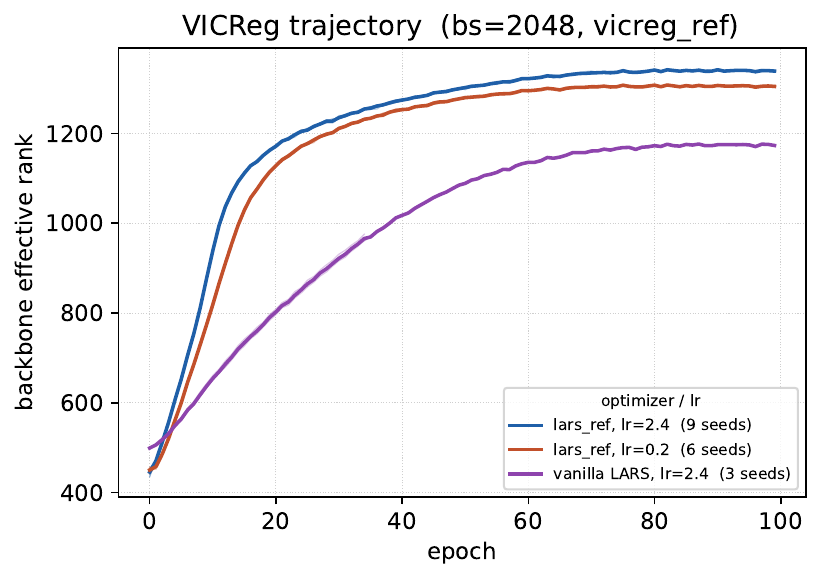}\hfill
\includegraphics[width=0.48\linewidth]{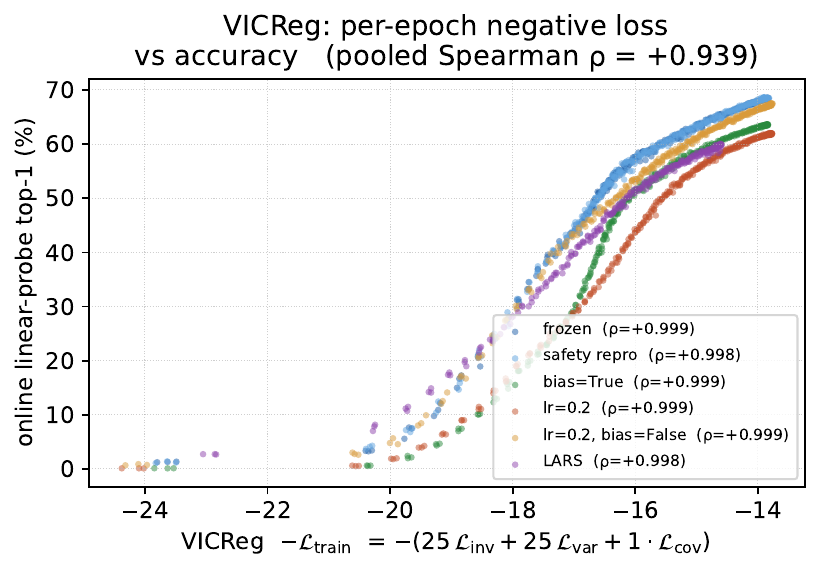}
\caption{\textbf{ImageNet-1K diagnostics for VICReg.} ResNet-50, $100$~epochs, batch size $2048$, projector $p_h{=}p_o{=}8192$, VICReg loss with $\alpha=\beta=25$, $\gamma=1$; $3$ seeds across $6$ design-axis ablations of the frozen recipe. Mirrors the middle and right panels of \cref{fig:in1k-peira} and \cref{fig:in1k-sigreg}. \emph{Left:} backbone effective rank~\citep{RankMe} versus epoch (faint lines: per-seed; bold: per-group mean). The six recipes are pooled into three groups along the true axis of variation (optimizer~/~lr); the bias and probe-protocol toggles within each group leave backbone effective rank unchanged within seed-spread. \emph{Right:} per-epoch negative training loss $-\mathcal{L}_{\mathrm{train}}$ versus online linear-probe top-1, colored by recipe, with per-recipe Spearman in the legend.}
\label{fig:in1k-vicreg}
\end{figure}

\subsection{Projector output width}
\label[appendixsection]{app:width}

\Cref{tab:app-width} varies the projector output width $k$ in the ablation
setting (\cref{par:ablation-recipe}), all other hyperparameters fixed.
ImageNet-1K top-1 rises with $k$ up to $k{=}1024$, where it approaches VICReg,
and eases at $k{=}2048$. In this setting, $k{=}1024$ transfers slightly worse
on average than $k{=}256$; with the recipe of \cref{tab:in1k_config}, it is the
strongest \method{} model on both ImageNet-1K and transfer (\cref{tab:in1k,tab:app-transfer}). The other
ImageNet-1K results use $k{=}256$.

The projector effective rank stays close to its cap $k$ at every width,
consistent with $\rank_{\mathrm{max}}\leq\min(k,\numcca)$ of
\cref{prop:peira-thresholded-subspace} with $k$ binding, and the backbone
effective rank grows monotonically with $k$. The $k\times k$ system solved by
\cref{alg:peira} stayed well conditioned at every width.

\begin{table}[h]
\centering
\small
\caption{\textbf{Projector output width.} \method{} on ImageNet-1K (ResNet-50,
$100$ epochs, output width $k$ varied).
Top-1: offline linear probe, mean $\pm$ sample standard deviation over seeds
$\{1,1000,10000\}$ ($1$ seed at $k{=}64$). Projector and backbone
$\erank$~\citep{RankMe}: end-of-training checkpoints, $20{,}480$ held-out
ImageNet-1K validation images, mean over seeds. Normalized $\erank$: backbone
$\erank$ of the in-training estimator divided by that of i.i.d.\ Gaussian
features under the same estimator. Transfer: mean linear-probe accuracy over
the eight datasets of \cref{tab:app-transfer} ($1$ seed at $k{=}256$, $2$
seeds at $k{=}1024$). All rows use the ablation setting
(\cref{par:ablation-recipe}). VICReg ($k{=}8192$) reaches $68.81\pm0.09\%$.}
\label{tab:app-width}
\begin{tabular}{lcccccc}
\toprule
$k$ & Top-1 (\%) & Seeds & Proj.\ $\erank/k$ & Backbone $\erank$ & Norm.\ $\erank$ & Transfer (\%) \\
\midrule
$64$   & $61.74$          & $1$ & --      & --       & $0.70$ & -- \\
$256$  & $66.50 \pm 0.02$ & $3$ & $0.987$ & $1461$   & $0.73$ & $70.72$ \\
$512$  & $67.88 \pm 0.10$ & $3$ & $0.983$ & $1479$   & $0.73$ & -- \\
$1024$ & $68.50 \pm 0.13$ & $3$ & $0.976$ & $1495$   & $0.74$ & $69.63$ \\
$2048$ & $68.14 \pm 0.15$ & $3$ & $0.965$ & $1504$   & --     & -- \\
\bottomrule
\end{tabular}
\end{table}

\subsection{Longer pretraining}
\label[appendixsection]{app:longtrain}

To test whether the ImageNet-1K gap to VICReg is an artifact of the
$100$-epoch budget, we pretrain \method{} (ablation setting,
\cref{par:ablation-recipe}) and VICReg for $300$ epochs with the cosine
schedule stretched accordingly, and probe offline at epochs $200$ and $300$.
Both methods gain about as much from the longer schedule, so the gap stays
nearly constant (\cref{tab:app-longtrain}): it is an offset, not a budget
effect.

\begin{table}[h]
\centering
\small
\caption{\textbf{Longer pretraining on ImageNet-1K.} ResNet-50, VICReg recipe
of \cref{tab:in1k_config}, \method{} ablation setting
(\cref{par:ablation-recipe}), offline linear-probe top-1 (\%). The $100$-epoch
row is the $3$-seed mean of each recipe. The $200$- and $300$-epoch rows come
from one $300$-epoch run per method (seed $1$); the $200$-epoch checkpoint is
taken mid-schedule, with the learning rate still at $0.63$ (peak $2.4$), so
it is comparable across methods but understates both.}
\label{tab:app-longtrain}
\begin{tabular}{lccc}
\toprule
Epochs & \method{} & VICReg & Gap (pp) \\
\midrule
$100$ & $66.5$ & $68.8$ & $2.3$ \\
$200$ & $66.1$ & $68.4$ & $2.3$ \\
$300$ & $68.3$ & $70.7$ & $2.4$ \\
\bottomrule
\end{tabular}
\end{table}

\subsection{EMA schedule and batch size}
\label[appendixsection]{app:ema-bs}

\paragraph{EMA schedule.}
\method{} tracks the covariance buffers that define $\hat{P}^{\star}$ and
$\hat{Q}^{\star}$ with an EMA whose rate $\eta$ is cosine-annealed from
$\eta_{\mathrm{init}}$ to $\eta_{\min}$ (\cref{remark:ema-schedule}). Top-1
changes little across the schedules of \cref{tab:app-ema} and the backbone
$\erank$ does not move: \method{} is robust to the quality of the predictor
estimate over this range.

\paragraph{Batch size.}
\Cref{tab:app-batch} varies the global batch size with the learning rate
scaled linearly ($0.3\times B/256$). Top-1 peaks at moderate batch sizes
($1024$--$2048$) and eases at both extremes under this fixed scaling rule.
The raw backbone $\erank$ grows with the batch size because the in-training
estimator accumulates $N=10\,B/8$ embeddings, so its ceiling $\min(N,2048)$
grows with $B$; normalized by the $\erank$ of i.i.d.\ Gaussian features under
the same estimator, it is flat ($0.71$--$0.75$).

\begin{table}[h]
\centering
\small
\begin{minipage}[t]{0.35\linewidth}
\centering
\caption{\textbf{EMA schedule.} \method{}, ImageNet-1K, ResNet-50, $100$
epochs, ablation setting (\cref{par:ablation-recipe}) with
$(\eta_{\mathrm{init}},\eta_{\min})$ varied; offline top-1 (\%), seed $1$.
Default: $(0.9, 0.5)$. The backbone $\erank$ stays at $\sim\!1280$
(normalized $0.72$) in every cell.}
\label{tab:app-ema}
\begin{tabular}{lcc}
\toprule
$\eta_{\mathrm{init}}$ & $\eta_{\min}{=}0.3$ & $\eta_{\min}{=}0.5$ \\
\midrule
$0.5$ & $65.88$ & $66.05$ \\
$0.7$ & $66.48$ & $66.20$ \\
$0.9$ & $66.30$ & $66.21$ \\
\bottomrule
\end{tabular}
\end{minipage}\hfill
\begin{minipage}[t]{0.62\linewidth}
\centering
\setlength{\tabcolsep}{3pt}
\caption{\textbf{Batch size.} Same setting as \cref{tab:app-ema} with the
global batch size $B$ varied and the learning rate scaled as $0.3\times B/256$.
Raw $\erank$: backbone, in-training estimator, with its ceiling (i.i.d.\
Gaussian features, same estimator); normalized $\erank$ = raw / ceiling.}
\label{tab:app-batch}
\begin{tabular}{lcccc}
\toprule
$B$ (LR) & $512$ ($0.6$) & $1024$ ($1.2$) & $2048$ ($2.4$) & $4096$ ($4.8$) \\
\midrule
Top-1 (\%)   & $63.85$ & $66.11$ & $66.23$ & $64.50$ \\
Raw $\erank$ & $446$   & $823$   & $1284$  & $1441$ \\
Ceiling      & $612$   & $1155$  & $1770$  & $1931$ \\
Norm.\ $\erank$ & $0.73$ & $0.71$ & $0.73$ & $0.75$ \\
\bottomrule
\end{tabular}
\end{minipage}
\end{table}

\subsection{Wall-clock cost and conditioning of the $k\times k$ solve}
\label[appendixsection]{app:wallclock}

\Cref{tab:app-wallclock} reports the training time of the three ImageNet-1K
recipes of \cref{tab:in1k_config} on one node of $8\times$ NVIDIA H200 GPUs at
batch size $2048$. The per-epoch times include the regularized linear solve
of \cref{alg:peira}. This solve acts on a $k\times k$ system, with $k$ the SSL
latent dimension, and is a negligible fraction of each step. \method{} is the
fastest of the three methods per epoch; for VICReg, the $8192$-d expander
dominates the difference. The
system $\hat{N}_{U,V}+\lambda I$ is symmetric positive definite with smallest
eigenvalue at least $\lambda$, so its condition number is bounded by
$1+\|\hat{N}_{U,V}\|/\lambda$; we observed no inversion-related numerical
failure in any run, including the wider projectors of \cref{app:width}. For much larger $k$, the joint updates of
\cref{remark:bilevel} remove the inversion altogether.

\begin{table}[h]
\centering
\small
\caption{\textbf{Training cost on ImageNet-1K.} ResNet-50, batch size $2048$,
one node of $8\times$ NVIDIA H200, recipes of \cref{tab:in1k_config}.
Per-epoch time includes data loading, the forward and backward passes, and,
for \method{}, the $k\times k$ regularized solve. \method{} is $24\%$ faster
per epoch than VICReg and $23\%$ faster than SIGReg.}
\label{tab:app-wallclock}
\begin{tabular}{lccc}
\toprule
Method & Projector output $k$ & s / epoch & $100$-epoch wall-clock \\
\midrule
SIGReg    & $128$  & $809$ & $\sim\!22.5$\,h \\
\method{} & $256$  & $626$ & $\sim\!17.4$\,h \\
VICReg    & $8192$ & $822$ & $\sim\!22.8$\,h \\
\bottomrule
\end{tabular}
\end{table}

\subsection{Transfer to the LeJEPA-8 suite}
\label[appendixsection]{app:transfer}

We evaluate frozen ImageNet-1K ResNet-50 backbones on the eight
classification datasets used by \citet{Balestriero:2025}: CIFAR-10 and CIFAR-100~\citep{Krizhevsky:2009},
DTD~\citep{Cimpoi:2014}, Flowers102~\citep{Nilsback:2008}, Oxford-IIIT Pets~\citep{Parkhi:2012},
Food101~\citep{Bossard:2014}, FGVC-Aircraft~\citep{Maji:2013} and Stanford Cars~\citep{Krause:2013}.
\paragraph{Protocol.} Images are resized to $256$ and center-cropped to
$224^2$ with ImageNet normalization; the $2048$-d pooled backbone features are
extracted once without augmentation. A linear head is trained on the frozen
features with Adam (learning rate $10^{-3}$, cosine schedule, $100$ epochs);
the weight decay is selected in $\{0, 10^{-4}, 10^{-2}\}$ on a $10\%$
held-out split of the training set, and the head is refit on the full
training set. We report top-1 accuracy on CIFAR-10, CIFAR-100 and Food101 and
mean per-class recall on the other five datasets. Checkpoints: \method{}
$k{=}256$ and $k{=}1024$, SIGReg $k{=}128$, VICReg $k{=}8192$ (\cref{tab:in1k_config}).

\paragraph{Results.} \method{} ($k{=}256$) exceeds SIGReg by $2.6$\,pp on average and on
$7$ of $8$ datasets (all but DTD), and trails VICReg by $1.2$\,pp on average,
less than its $2.4$\,pp ImageNet-1K gap (\cref{tab:app-transfer}). The
$k{=}1024$ \method{} checkpoints, stronger on ImageNet-1K, also transfer
slightly better on average ($71.09\%$ vs.\ $70.89\%$) and are best on
CIFAR-10, Pets and Food101.

\begin{table}[h]
\centering
\small
\setlength{\tabcolsep}{4pt}
\caption{\textbf{Linear transfer from ImageNet-1K.} Frozen ResNet-50
backbones pretrained for $100$ epochs on ImageNet-1K (recipes of
\cref{tab:in1k_config}; VICReg and SIGReg: seed $1$; \method{}: mean over
seeds $\{1,1000,10000\}$); linear probe protocol of \cref{app:transfer}.
IN1K: offline ImageNet-1K top-1 of the same checkpoints. Metric: top-1 (\%) for
CIFAR-10/100 and Food101, mean per-class recall (\%) otherwise. Bottom row:
\method{} with a $1024$-d projector (recipe of \cref{tab:in1k_config}, as in \cref{tab:in1k}).
Best per column in bold.}
\label{tab:app-transfer}
\begin{tabular}{lcccccccccc}
\toprule
Method ($k$) & IN1K & C10 & C100 & DTD & Flowers & Pets & Food & Aircraft & Cars & Mean \\
\midrule
VICReg ($8192$)   & $68.80$ & $90.55$ & $\mathbf{73.98}$ & $68.24$ & $\mathbf{85.80}$ & $82.94$ & $69.92$ & $\mathbf{54.79}$ & $\mathbf{50.35}$ & $\mathbf{72.07}$ \\
SIGReg ($128$)    & $66.46$ & $90.16$ & $71.94$ & $\mathbf{70.05}$ & $82.56$ & $77.44$ & $66.87$ & $47.86$ & $39.25$ & $68.27$ \\
\method{} ($256$) & $66.45$ & $90.46$ & $72.89$ & $68.30$ & $85.76$ & $79.52$ & $68.47$ & $53.11$ & $48.64$ & $70.89$ \\
\midrule
\method{} ($1024$) & $\mathbf{69.17}$ & $\mathbf{90.63}$ & $73.54$ & $69.08$ & $84.62$ & $\mathbf{83.20}$ & $\mathbf{70.02}$ & $50.95$ & $46.70$ & $71.09$ \\
\bottomrule
\end{tabular}
\end{table}

\subsection{Synthetic two-view Gaussian data}
\label[appendixsection]{app:synthetic}

\paragraph{Strict-rank staircase.}
On real data $\lambda$ stays far below $\sqrt{c_1}$, so the mode filtering of
\cref{prop:peira-thresholded-subspace}, $\rank_{\mathrm{max}}=\#\{i:\sqrt{c_i}>\lambda\}$,
is not observable (\cref{sec:appendix:eff_rank}). We therefore train two linear
encoders $\mathbb{R}^{32}\to\mathbb{R}^{12}$ with \method{} for $250$ epochs on
$20{,}000$ pairs of two-view Gaussian data with six prescribed canonical
correlations, sweeping $\lambda$ with one value between each pair of consecutive
thresholds $\sqrt{c_i}$. The recovered rank steps down by one each time $\lambda$
crosses a $\sqrt{c_i}$ and equals $\rank_{\mathrm{max}}$ at every $\lambda$
(\cref{fig:app-synthetic}a). With clustered correlations
(\cref{fig:app-synthetic}b), it drops at each cluster and is resolved up to
$\pm1$ inside a cluster, where thresholds differ by only $0.01$--$0.02$.

\begin{figure}[h]
\centering
\begin{subfigure}[t]{0.495\linewidth}
\includegraphics[width=\linewidth,trim=0 0 446pt 70pt,clip]{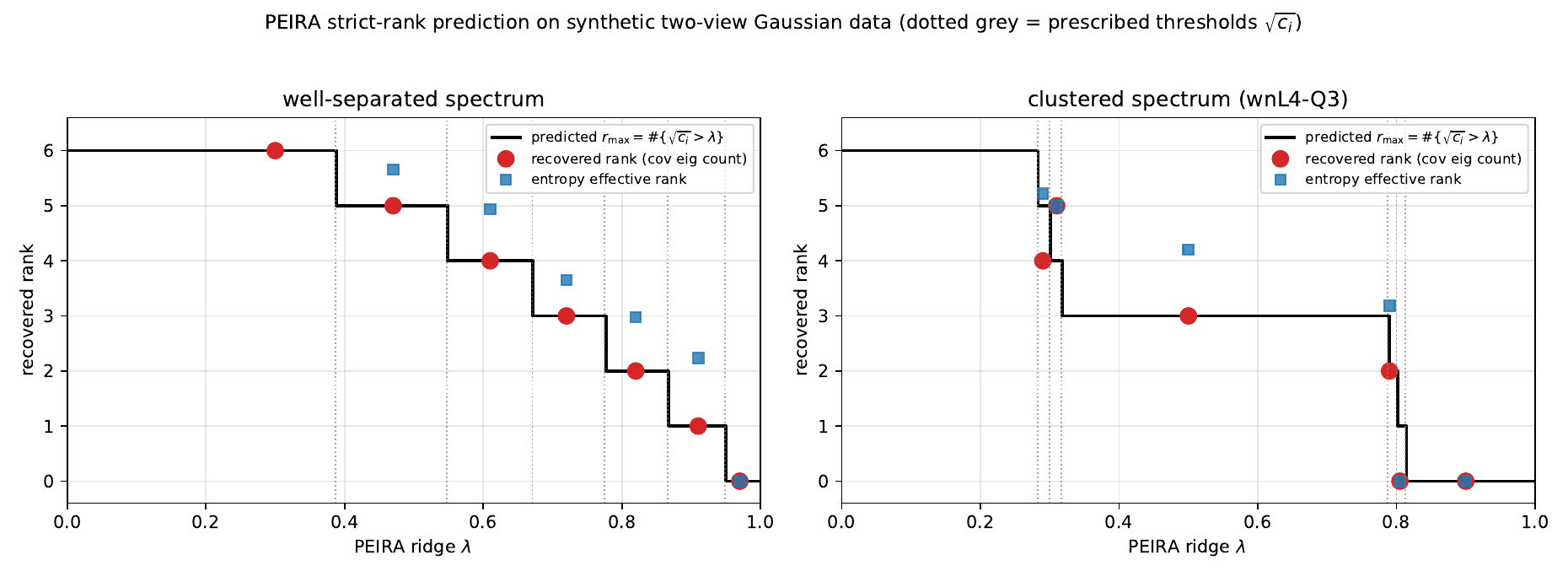}
\subcaption{}
\end{subfigure}\hfill
\begin{subfigure}[t]{0.495\linewidth}
\includegraphics[width=\linewidth,trim=446pt 0 0 70pt,clip]{figures/cr/peira_synthetic_strict_rank.pdf}
\subcaption{}
\end{subfigure}
\caption{\textbf{Strict-rank prediction on synthetic two-view Gaussian data.}
Linear encoders trained with \method{} on two-view Gaussian data with $6$
prescribed canonical correlations $c_i$, swept over $\lambda$. Black: predicted
$\rank_{\mathrm{max}}=\#\{i:\sqrt{c_i}>\lambda\}$; dotted grey: thresholds
$\sqrt{c_i}$; red: recovered rank (covariance eigenvalues above
$5\cdot10^{-3}$); blue: entropy effective rank. \emph{(a)} Well-separated
spectrum $c=(0.90,0.75,0.60,0.45,0.30,0.15)$. \emph{(b)} Clustered spectrum
$c=(0.66,0.64,0.62,0.10,0.09,0.08)$.}
\label{fig:app-synthetic}
\end{figure}

\paragraph{Recovery of CCA directions across initialization scales.}
We consider a linear Gaussian two-view model with known CCA directions
and spectra, using a fixed, orthogonalized batch to remove sampling noise. 
We train linear encoders for $400{,}000$ updates from Gaussian initializations with standard deviation
$\rho_0 \in \{10^{-2},10^{-3},10^{-4},10^{-5},10^{-6}\}$.
Both methods use exact predictors; self-distillation follows the
fast-predictor dynamics of \cref{sec:dynamics}.
We track ground-truth directional alignment and spectral ratios relative
to the analytical target, which retains four canonical modes. 
\method{} recovers all four target directions and suppresses inactive
modes at every scale (\cref{fig:app-cca}).
Self-distillation recovers four directions at $\rho_0=10^{-2}$ and
$10^{-3}$, three at $10^{-4}$, one at $10^{-5}$, and none at $10^{-6}$.
Thus, even with an exact predictor, self-distillation exhibits
initialization-dependent dimensional collapse, whereas \method{}
consistently recovers the target subspace within the training budget.

\begin{figure}[h]
\centering

\includegraphics[width=0.49\linewidth]{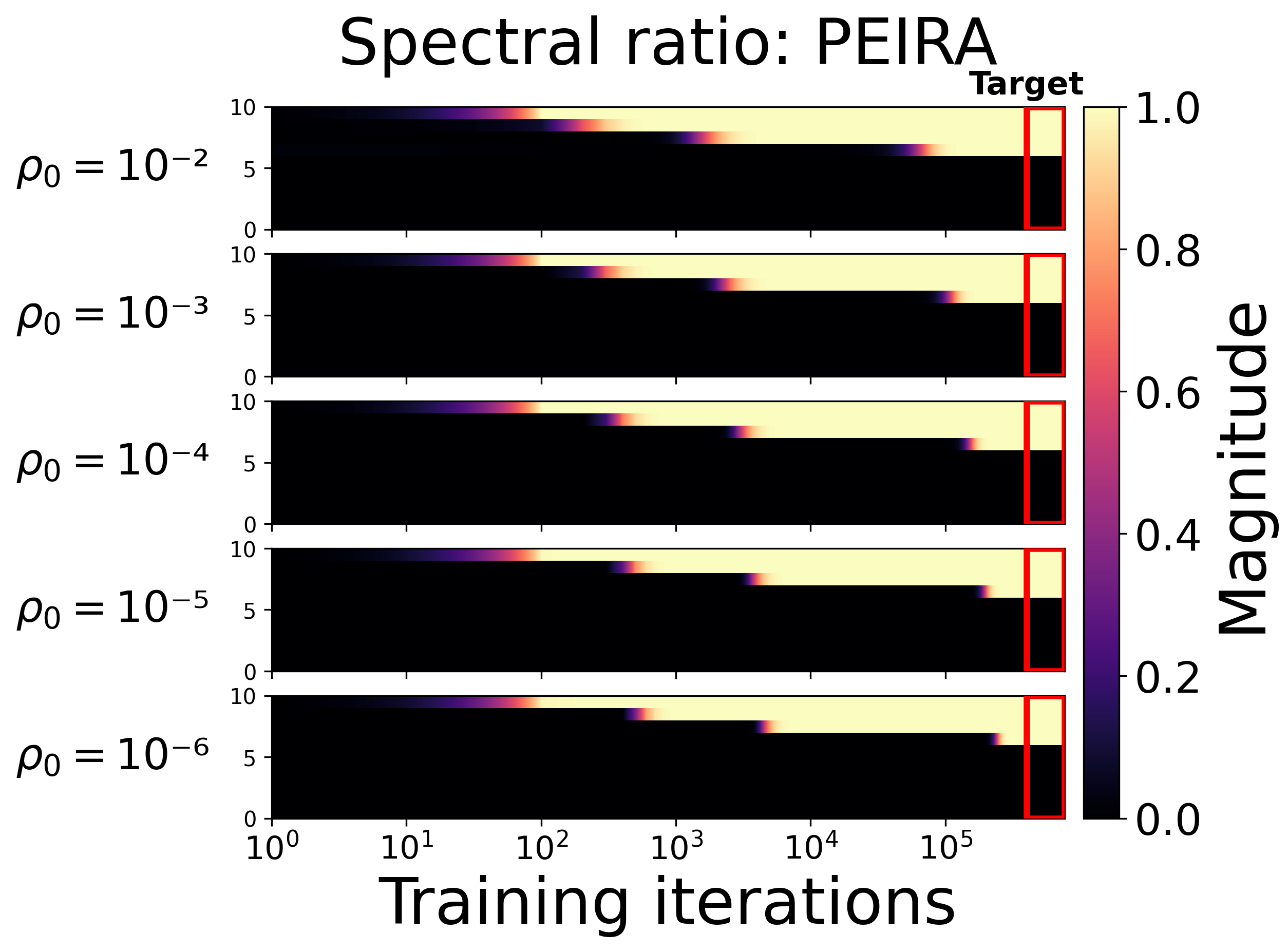}
\includegraphics[width=0.49\linewidth]{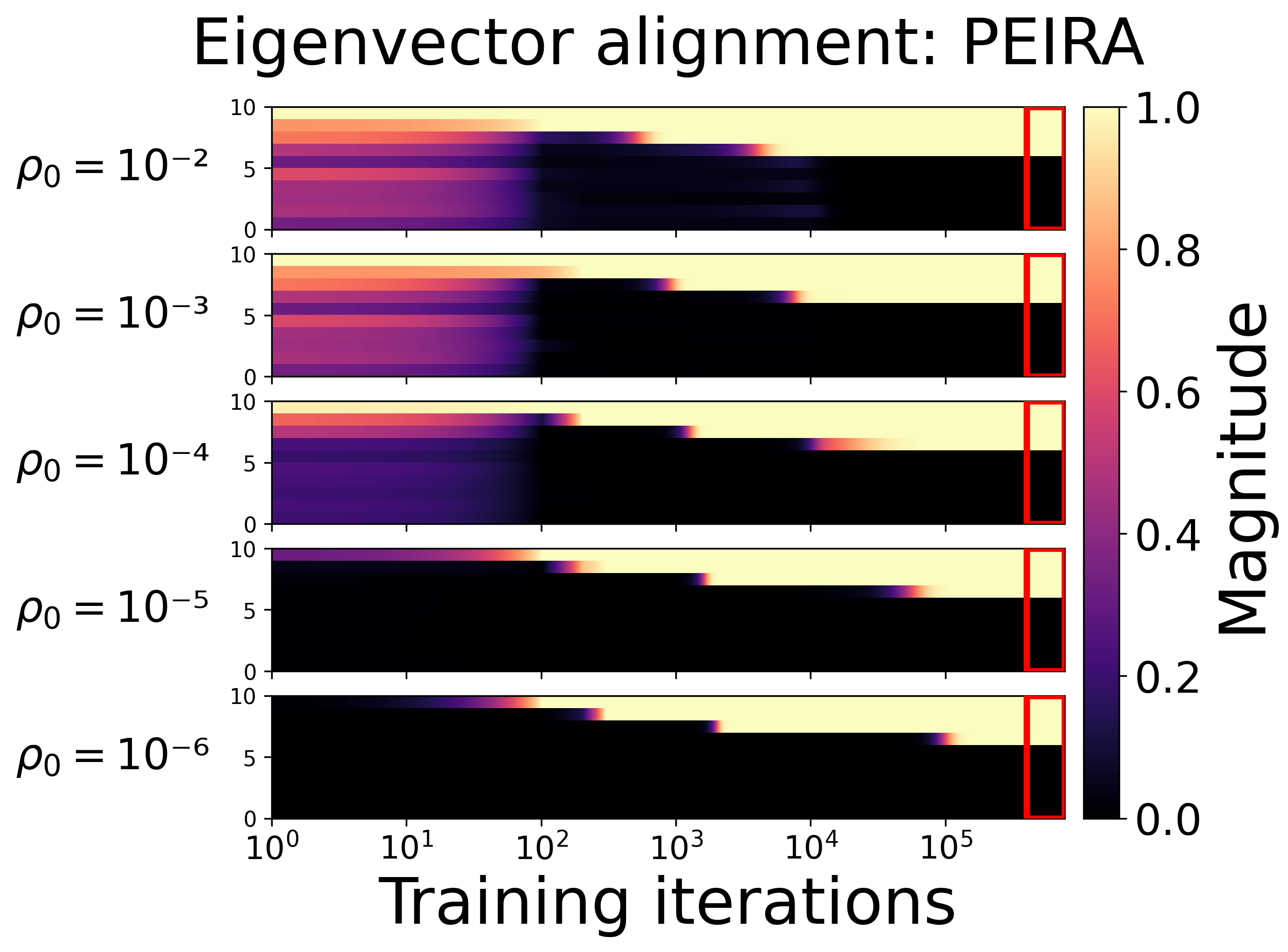}

\includegraphics[width=0.49\linewidth]{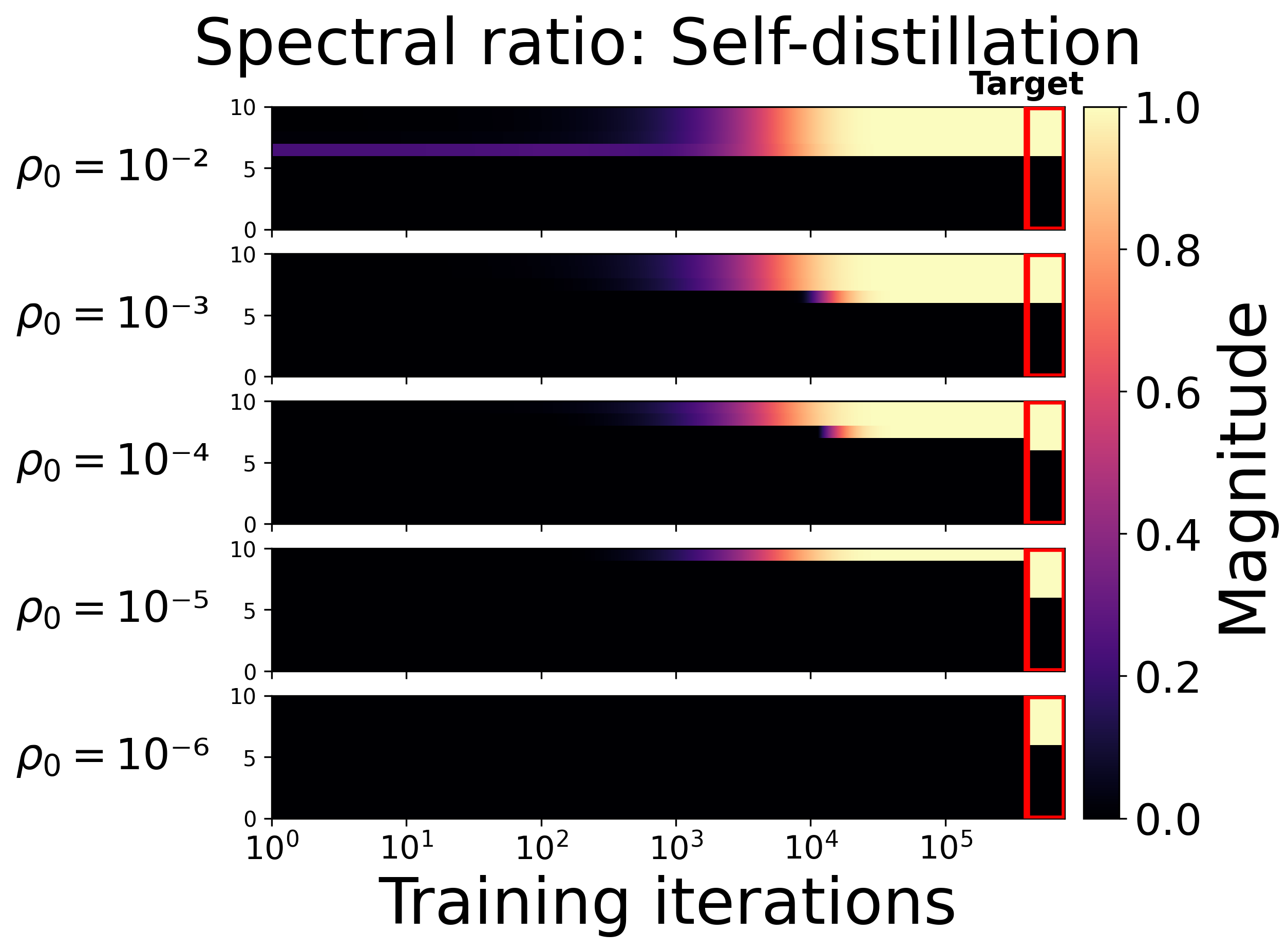}
\includegraphics[width=0.49\linewidth]{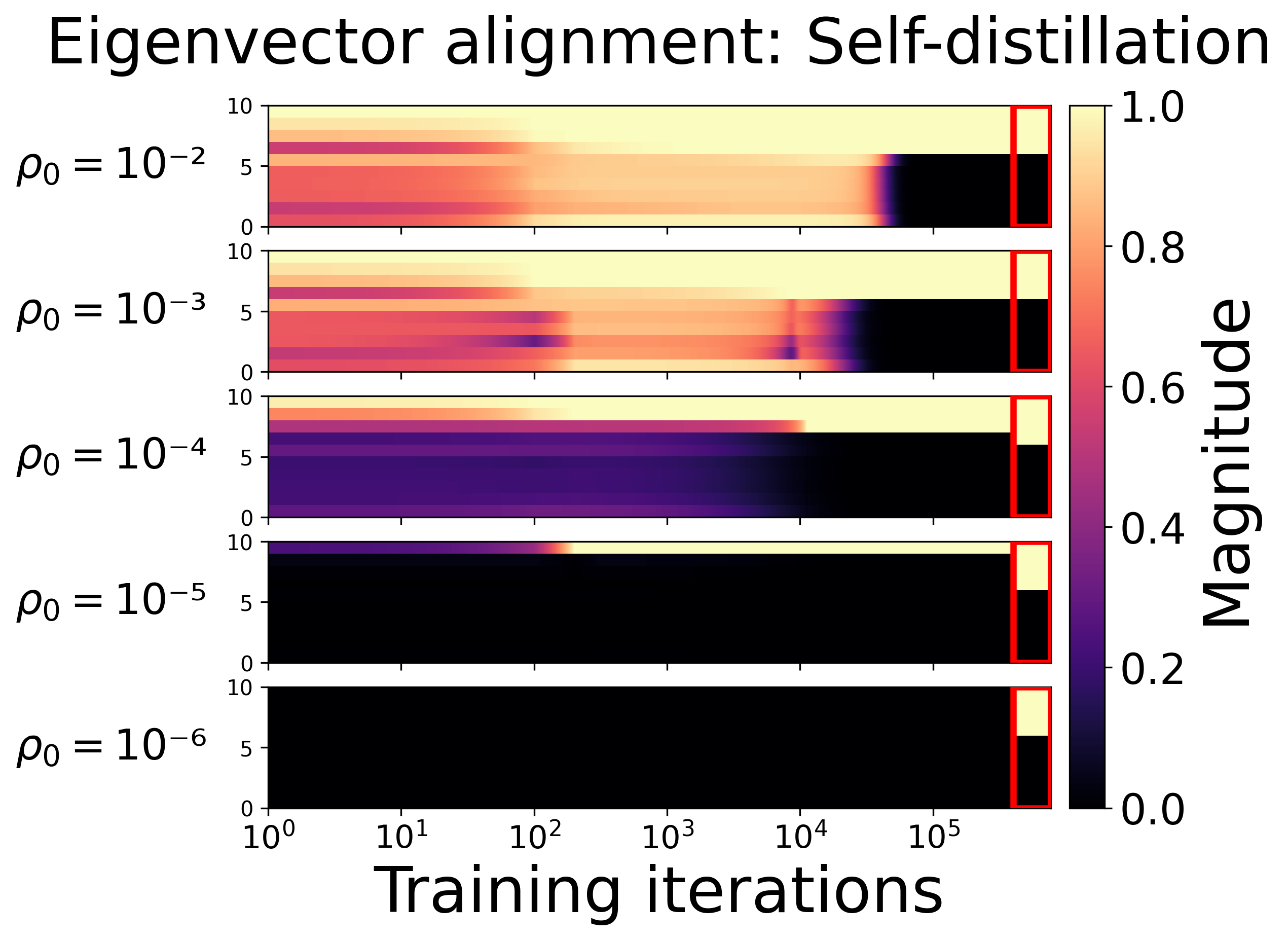}

\caption{\textbf{Learning dynamics on deterministic toy data.}
Spectral ratios and eigenvector alignment for \method{} and self-distillation
across initialization scales $\rho_0 \in
\{10^{-2},10^{-3},10^{-4},10^{-5},10^{-6}\}$, ordered from top to bottom.
Each heatmap shows the evolution of individual components over training iterations on a logarithmic horizontal scale; color indicates magnitude.
The red-outlined block shows the reference target, with four active
components. 
\method{} recovers the target spectral structure across the
initialization scales, whereas self-distillation exhibits increasing
dimensional collapse as initialization becomes smaller.}
\label{fig:app-cca}
\end{figure}

\subsection{CIFAR-10: robustness to initialization scale}
\label[appendixsection]{app:cifar-initvar}

We study the sensitivity of \method{} and self-distillation to the initial
weight scale on CIFAR-10. We vary the multiplier applied to convolutional
and linear weights over three orders of magnitude,
$\sigma_0 \in \{10^{-2},10^{-1},1,10\}$. 
Both methods use ResNet-18, a $1024$-dimensional projector, the same
augmentations, batch size $256$, and $1000$ training epochs.
The self-distillation baseline uses stop-gradient targets and a learned predictor updated by an implicit regression step after each encoder update. 
We retain method-specific optimization settings: \method{} uses $\lambda{=}0.7$ with the two covariance-estimation settings
described in the caption of \cref{fig:app-cifar-initvar},
an encoder learning rate of $0.04$, gradient clipping, and float16, whereas
self-distillation uses a learning rate of $0.001$, implicit step-size of $0.01$, no gradient clipping, and bfloat16 as we found this setting to yield improved results. Results are obtained for three seeds.

We evaluate classification accuracy using an online linear probe trained on detached backbone features.
To characterize dimensional collapse, we compute the entropy-based effective rank of the projector's uncentered second-moment matrix and the mean alignment across its $1024$ directions. 
These diagnostics are computed separately for each seed before averaging.

\Cref{fig:app-cifar-initvar} shows that \method{} learns useful representations across scales $\sigma_0\in \{0.01, 0.1, 1\}$. 
Its mean final accuracy ranges from $89.56\%$ to $91.65\%$, with effective rank $364$--$467$ and mean alignment $0.44$--$0.54$. 
Larger initialization delays the growth of alignment and effective rank, but these quantities recover during training.
Self-distillation is substantially more sensitive to initialization:
its mean accuracy is $77.69\%$ and $81.10\%$ at scales $0.1$ and $1$,
respectively, but falls to $49.59\%$ at $0.01$ and $46.00\%$ at $10$.
Across scales, its final effective rank remains approximately $1$--$61$
and mean alignment $0.009$--$0.070$, indicating pronounced dimensional
collapse. The smallest initialization also produces substantial
seed-to-seed variability in accuracy
($49.59 \pm 24.09\%$, mean $\pm$ sample standard deviation).
Under these training settings, \method{} therefore exhibits more
consistent accuracy and avoids the severe dimensional collapse observed
for self-distillation.

\begin{figure}[h]
\centering
\includegraphics[width=\linewidth]{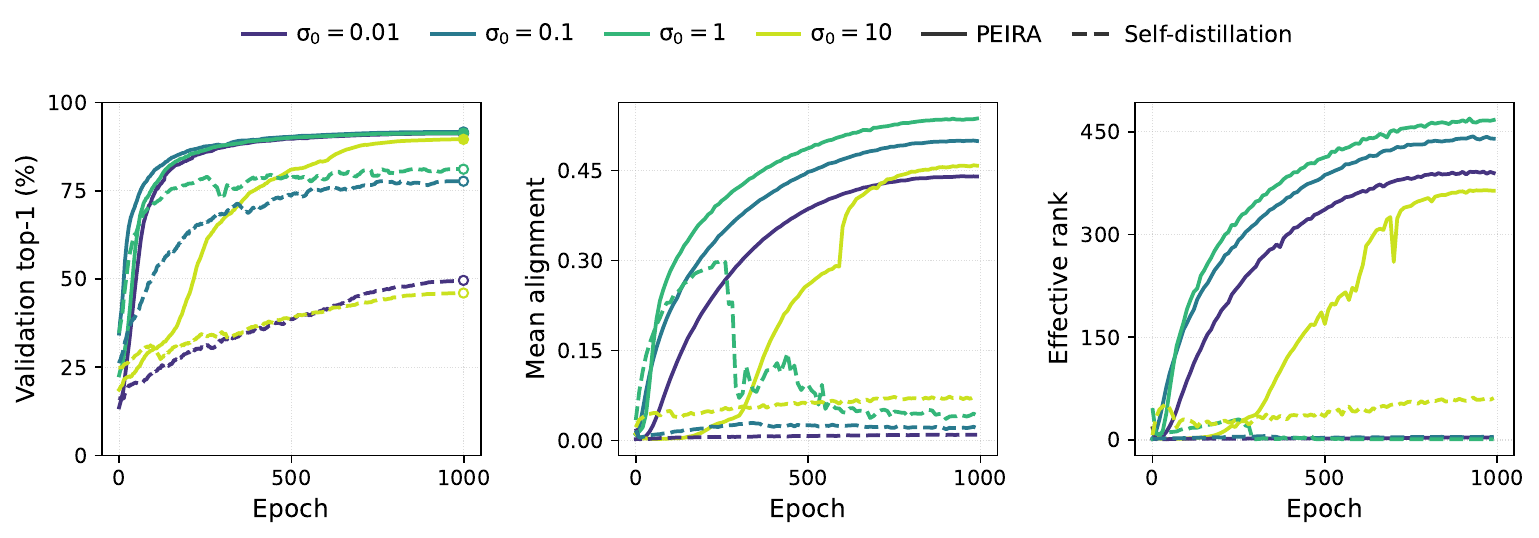}
\includegraphics[width=\linewidth]{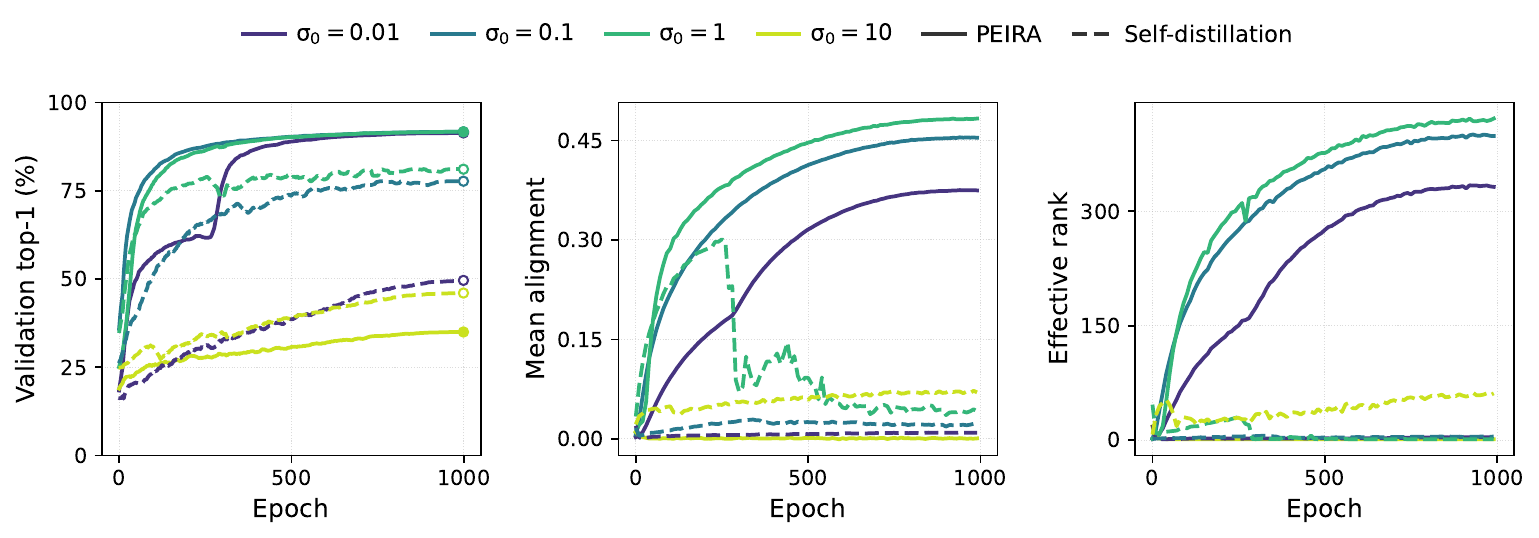}
\caption{\textbf{Initialization sensitivity and dimensional collapse on CIFAR-10.}
PEIRA (solid) and self-distillation (dashed) use ResNet-18 and are compared at
initialization scales $\sigma_0 \in \{10^{-2},10^{-1},1,10\}$ (colors).
\textit{Left:} online linear-probe validation accuracy.
\textit{Middle:} mean inter-view alignment across projector directions.
\textit{Right:} entropy-based effective rank of the projector second-moment matrix.
\textit{Top:} PEIRA uses an EMA of the within-view covariance and minibatch
estimates of the inter-view covariance, with the EMA coefficient scheduled
from $0.80$ to $0.50$.
\textit{Bottom:} PEIRA uses EMAs of both covariances, with the coefficient
scheduled from $0.99$ to $0.95$. The same self-distillation runs appear in
both rows. Curves are averaged over three seeds. Accuracy curves use a
centered 21-epoch moving average, and markers show the unsmoothed final means.
For $\sigma_0 \leq 1$, PEIRA achieves higher final accuracy, alignment, and
effective rank than self-distillation under both configurations.
At $\sigma_0=10$, PEIRA recovers after an initial drop in effective rank
in the top row, but remains near rank one in the bottom row.}
\label{fig:app-cifar-initvar}
\end{figure}

\end{document}